\documentclass{article}

\title{Real Time Reasoning in OWL2 for GDPR Compliance}
\author{Piero A.~Bonatti\\Universit\`a di Napoli Federico II \and Luca Ioffredo\\CeRICT \and  Iliana M.~Petrova\\CeRICT \and Luigi Sauro\\Universit\`a di Napoli Federico II \and Ida R.~Siahaan\\CeRICT }

\date{April 19, 2019}

\usepackage{url}
\usepackage{times}
\usepackage{latexsym}
\usepackage{xspace}
\usepackage{amsmath,amssymb,amsfonts}
\usepackage{stmaryrd}
\usepackage{epstopdf}
\usepackage{eurosym}

\newtheorem{theorem}{Theorem}[section]
\newtheorem{lemma}[theorem]{Lemma}
\newtheorem{corollary}[theorem]{Corollary}
\newtheorem{proposition}[theorem]{Proposition}

\newtheorem{remark}[theorem]{Remark}
\newtheorem{example}[theorem]{Example}
\newtheorem{definition}[theorem]{Definition}

\newenvironment{proof}
{\begin{trivlist} \item[] {\bf Proof.}}%
{\qed \end{trivlist}}
\newcommand{\qed}{\hspace*{\fill} \rule{1.1mm}{2.2mm}}

{\qed \end{trivlist}}

\newenvironment{theorem*}[2]%
{\begin{trivlist} \item[] {\bf #1~\protect{\ref{#2}.}}\it}{\end{trivlist}}

\newcommand{\NOTE}[1]{\marginpar{\footnotesize #1}}
\newcommand{\hide}[1]{}
\newcommand{\tup}[1]{\ensuremath{\langle #1 \rangle}\xspace}

\newcommand{\easycal}[1]{\ensuremath{\mathcal{#1}}\xspace}
\newcommand{\easysf}[1]{\ensuremath{\mathsf{#1}}\xspace}
\newcommand{\easyit}[1]{\ensuremath{\mathit{#1}}\xspace}

\newcommand{\easytt}[1]{\ensuremath{\mathtt{#1}}\xspace}
\newcommand{\mt}[1]{\ensuremath{\mathtt{#1}}\xspace}

\usepackage{enumerate}
\usepackage{pgfplots}
\pgfplotsset{compat=1.13}

\newcommand{\grant}{731601}
\newcommand{\SPECIAL}{SPECIAL\xspace}
\newcommand{\SPsite}{\url{https://www.specialprivacy.eu/}}

\usepackage[ruled,linesnumbered]{algorithm2e}

\newcommand{\EL}{\easycal{EL}}
\newcommand{\ELp}{\easycal{EL^+}}
\newcommand{\ELpp}{\easycal{EL^{++}}}

\newcommand{\DLL}{\mbox{\textit{DL-lite}}\xspace}
\newcommand{\DLLh}{\easyit{\mbox{\it DL-lite}_\mathit{horn}^\mathcal{H}}}

\newcommand{\SRIQ}{\easycal{SRIQ}}

\newcommand{\A}{\easycal{A}}
\newcommand{\I}{\easycal{I}}
\newcommand{\J}{\easycal{J}}
\newcommand{\K}{\easycal{K}}
\newcommand{\M}{\easycal{M}}
\newcommand{\N}{\easycal{N}}
\renewcommand{\L}{\easycal{L}}
\renewcommand{\O}{\easycal{O}}

\newcommand{\T}{\easycal{T}}

\newcommand{\Q}{\easycal{Q}}
\newcommand{\U}{\easycal{U}}

\newcommand{\NC}{\easysf{N_C}}
\newcommand{\NR}{\easysf{N_R}}
\newcommand{\NI}{\easysf{N_I}}
\newcommand{\NF}{\easysf{N_F}}
\newcommand{\ND}{\easysf{D}}
\newcommand{\Np}{\easysf{N_P}}
\newcommand{\NP}{\easysf{NP}}

\newcommand{\func}{\easysf{func}}
\newcommand{\range}{\easysf{range}}
\newcommand{\disj}{\easysf{disj}}

\newcommand{\PL}{\easycal{PL}}
\newcommand{\SSA}{\easysf{STS}}
\newcommand{\PLR}{\easysf{PLR}}
\newcommand{\SSO}{\easysf{STS^\O}}
\newcommand{\PLRO}{\easysf{PLR^\O}}

\newcommand{\KB}{\easycal{KB}}
\newcommand{\BP}{\easycal{BP}}

\newcommand{\purp}{\easysf{has\_purpose}}
\newcommand{\data}{\easysf{has\_data}}
\newcommand{\proc}{\easysf{has\_processing}}
\newcommand{\recip}{\easysf{has\_recipient}}
\newcommand{\stor}{\easysf{has\_storage}}
\newcommand{\loc}{\easysf{has\_location}}
\newcommand{\dur}{\easysf{has\_duration}}
\newcommand{\duty}{\easysf{has\_duty}}

\newcommand{\splt}[2]{\ensuremath{\mathit{split}_{#2}(#1)}\xspace}

\newcommand{\sig}[1]{\Sigma(#1)}
\newcommand{\pos}{\easyit{pos}}
\newcommand{\SROIQ}{\easycal{SROIQ}}

\newcommand{\PLSO}{\easycal{PLSO}}
\newcommand{\lit}{\easyit{lit}}
\newcommand{\comp}{\easyit{comp}}

\renewcommand{\ni}{\easyit{\# int}}

\begin{document}

\maketitle

\begin{abstract}
This paper shows how knowledge representation and reasoning techniques
can be used to support organizations in complying with the GDPR, that
is, the new European data protection regulation.  This work is carried
out in a European H2020 project called SPECIAL.  Data usage policies,
the consent of data subjects, and selected fragments of the GDPR are
encoded in a fragment of OWL2 called \PL (policy language);
compliance checking and policy validation are reduced to subsumption
checking and concept consistency checking.
This work proposes a satisfactory tradeoff between the expressiveness
requirements on \PL posed by the GDPR, and the scalability
requirements that arise from the use cases provided by SPECIAL's
industrial partners.  Real-time compliance checking is achieved by
means of a specialized reasoner, called \PLR, that leverages knowledge
compilation and structural subsumption techniques.  The performance of
a prototype implementation of \PLR is analyzed through systematic
experiments, and compared with the performance of other important
reasoners.
Moreover, we show how \PL and \PLR can be extended to support richer
ontologies, by means of import-by-query techniques.
\PL and its integration with OWL2's profiles constitute new tractable
fragments of OWL2. We prove also some negative results, concerning
the intractability of unrestricted reasoning in \PL, and the
limitations posed on ontology import.
\end{abstract}

\section{Introduction}
\label{sec:intro}

The new European General Data Protection
Regulation\footnote{\url{http://data.consilium.europa.eu/doc/document/ST-5419-2016-INIT/en/pdf}}
(GDPR), that has come into force on May 25, 2018, places stringent
restrictions on the processing of personally identifiable data.
The regulation applies  also to companies and organizations that are not located in Europe, whenever they track or provide services to data subjects that are in the European Union.\footnote{Cf.\ Article~3 of the GDPR.}
Infringements may severely affect the reputation of the violators, and
are subject to substantial administrative fines (up to 4\% of the
total worldwide annual turnover or 20 million Euro, whichever is
higher).
Therefore, the risks associated to infringements constitute a major disincentive to the abuse of personal data.
Given that  the collection
and the analysis of personal data are paramount sources of innovation
and revenue, companies are interested in maximizing personal data
usage within the limits posed by the GDPR.
Consequently, \emph{data controllers} (i.e.\ the personal and legal entities that process personal data) are looking for
methodological and technological means to comply with the
regulation's requirements efficiently and safely.

The European H2020 project SPECIAL\footnote{\SPsite} is aimed at
supporting controllers in complying with the GDPR. SPECIAL is tackling
several hard problems related to usability, transparency and compliance, see
\cite{DBLP:conf/safecomp/BonattiKPW17,DBLP:conf/semweb/BonattiBDFKPPW18,DBLP:conf/esws/KirraneFDMPBWDR18}
for an overview.  In this paper, we focus on SPECIAL's approach to the
representation of data usage activities and consent to data
processing, together with the associated reasoning tasks related to
the validation of data usage policies and compliance checking.

The management of the consent to data processing granted by data
subjects plays a central role in this picture. The GDPR is not
concerned with anonymous data, nor data that do not describe persons
(like astronomical data). The other data (hereafter called
\emph{personal data}) must be processed according to the legal bases
provided by the regulation. Some examples of such legal bases include
public interest, the vital interests of the data subject, contracts,
and the legitimate interests of the data controller, just to name a
few.\footnote{Cf.~Article~6 of the GDPR.}  These legal bases are
constrained by a number of provisos and caveats that restrict their
applicability.\footnote{Of particular relevance here are the data
  minimization principle introduced in Article~5, and the limitations
  to the legitimate interests of the controller rooted in
  Article~6.1(f).} So, in practice, the kinds of personal data
processing that are most useful for data-driven business are almost
exclusively allowed by another legal basis, namely, the explicit
consent of the data subjects.\footnote{Article~6.1(a)} Thus, it is
important to encode consent appropriately, so as to record it for
auditing, and give automated support to compliance checking.

Also the controller's usage of personal data must be appropriately
represented and stored, in order to fulfill the obligation to record
personal data processing activities,\footnote{Cf.\ Article~30 of the
  GDPR.} and in order to verify that such activities comply with the
available consent and with the GDPR.

SPECIAL tackles these needs by adopting a logic-based representation
of data usage policies, that constitutes a uniform language to encode
consent, the activities of controllers, and also selected parts of the
GDPR.
A logic-based approach is essential for achieving several important objectives, including the following:
\begin{itemize}
\item strong correctness and completeness guarantees on permission
  checking and compliance checking;
\item ensuring the mutual coherence of the different reasoning tasks
  related to policies, such as policy validation, permission checking,
  compliance checking, and explanations;
\item ensuring correct usage after data is transferred to other
  controllers (i.e.\ interoperability), through the unambiguous
  semantics of knowledge representation languages.
\end{itemize}

\noindent
Some of SPECIAL's use cases place challenging scalability requirements
on reasoning.  During the execution of the controllers' data
processing software, each operation involving personal data must be
checked for compliance with the consent granted by the data subjects.
The frequency of such compliance checks may be significantly high, so
SPECIAL needs to implement the corresponding reasoning tasks in such a
way that the time needed for each check does not exceed
a few hundreds of $\mu$-seconds.
We address this requirement by designing a specialized reasoner for
the policy language.

After recalling the notions about description logics and their properties, that will be needed in the paper, our contributions will be illustrated in the following order.
\begin{itemize}
\item Section~\ref{sec:enc} shows how to encode usage policies and the
  relevant parts of the GDPR with a fragment of $\SROIQ(D)$ (the
  logical foundation of OWL2-DL).  The details of the encoding will be
  related explicitly to GDPR's requirements. Afterwards, we formally
  define \PL, that is, the fragment of $\SROIQ(D)$ used to encode data
  usage policies.
\item Section~\ref{sec:reasoner} is devoted to the complexity analysis
  of reasoning in \PL. We consider concept satisfiability and
  subsumption checking, that are at the core of policy validation and
  compliance checking, respectively.  We will show that unrestricted
  \PL subsumption checking is co\NP-complete. However, under a restrictive
  hypothesis motivated by SPECIAL's use cases, subsumption checking is
  possible in polynomial time. Tractability is proved by means of a
  specialized two-stage reasoner called \PLR, based on a preliminary normalization
  phase followed by a structural subsumption algorithm.  A preliminary
  account of this section has been published in
  \cite{DBLP:conf/ijcai/Bonatti18}.

\item Section~\ref{sec:oracles} shows how to support richer ontology
  languages for the description of policy elements.  The vocabularies
  for policy elements are treated like imported ontologies by means of
  an \emph{import by query} (IBQ) approach, that can be implemented
  with a modular integration of the specialized reasoner for \PL with
  a reasoner for the imported ontology. We prove that this integration
  method is correct and complete, and justify the restrictive
  assumptions on the imported ontologies, by adapting and slightly
  extending previous results on IBQ limitations.  Moreover, we show
  that under hypotheses compatible with SPECIAL's application
  scenarios, the external ontology can be compiled into a \PL
  ontology, thereby reducing the IBQ approach to plain \PL reasoning.
  
\item \PL subsumption checking (which is the core of compliance
  checking) is experimentally evaluated in
  Section~\ref{sec:experiments}.  After describing the implementation
  of \PLR and its optimizations, \PLR's performance is compared with
  that of other important engines, such as Hermit
  \cite{DBLP:journals/jar/GlimmHMSW14} and ELK
  \cite{DBLP:journals/jar/KazakovKS14}.  For this purpose, we use two
  sets of experiments. The first set is derived from the pilots of
  SPECIAL that have reached a sufficient development level, namely, a
  recommendation system based on location and internet navigation
  information, developed by Proximus, and a financial risk analysis
  scenario developed by Thomson Reuters.  The second batch of
  experiments is fully synthetic, instead, and contains increasingly
  large policies and ontologies, in order to assess the scalability of
  \PLR.
\end{itemize}
Section~\ref{sec:conclusions} concludes the paper with a final
discussion of our results and interesting perspectives for future
work.  Related work is heterogeneous (declarative policy languages,
tractable description logics, IBQ methods) so we distribute its
discussion across the pertinent sections, rather than in a single
dedicated section.

\section{Preliminaries on Description Logics}
\label{sec:DL}

Here we report the basics on Description Logics (DL) needed for our work and refer the reader to \cite{DBLP:conf/dlog/2003handbook} for further details.
The DL languages of our interest are built from countably infinite sets of concept names (\NC), role names (\NR), individual names (\NI), concrete property names (\NF),  and concrete predicates (\Np).
A signature $\Sigma$ is a subset of $\NC\cup\NR\cup\NI\cup\NF$.\footnote{Concrete predicates are deliberately left out due to their special treatment.}

We will use metavariables $A,B$ for concept names, $C,D$ for possibly compound concepts, $R,S$ for role expressions, $a,b$ for individual names, and $f,g$ for concrete property names. The syntax of the concept and role expressions used in this paper is illustrated in Table~\ref{tab:syntax-semantics}.

An \emph{interpretation} \I of a signature $\Sigma^\I$ is a structure
$\I=(\Delta^\I, \cdot^\I)$ where $\Delta^\I$ is a nonempty set, and
the \emph{interpretation function} $\cdot^\I$, defined over
$\Sigma^\I$, is such that (i)~$A^\I \subseteq \Delta^\I$ if
$A\in\NC$; (ii)~$R^\I \subseteq \Delta^\I\times\Delta^\I$ if
$R\in\NR$; (iii)~$a^\I\in\Delta^\I$ if $a\in\NI$; (iv)~$f^\I
\subseteq \Delta^\I\times\Delta^\ND$ if $f\in\NF$, where
$\Delta^\ND$ denotes the domain of the predicates in \Np.\footnote{We
  are assuming -- for brevity -- that there is one concrete
  domain. However, this framework can be immediately extended to
  multiple domains.}
The semantics of an $n$-ary predicate $p\in\Np$ is a set of tuples $p^\ND\subseteq (\Delta^\ND)^n$.
In this paper we use $\Delta^\ND=\mathbb{N}$ and unary concrete predicates \mt{in_\mathit{\ell,u}}, where $\ell,u\in\mathbb{N}$, such that $\mt{in}_{\ell,u}^\ND = [\ell,u]$.  To enhance readability we will abbreviate $\mt{in}_{\ell,u}(f)$ to $\exists f.[\ell,u]$. So an individual $d\in\Delta^\I$ belongs to $(\exists f.[\ell,u])^\I$ if, for some integer $i\in[\ell,u]$, $(d,i)\in f^\I$.

The third column of Table~\ref{tab:syntax-semantics} shows how to
extend the valuation $\cdot^\I$ of an interpretation \I to compound DL
expressions and axioms.  GCI stands for ``general concept
inclusion''. An interpretation \I \emph{satisfies} an axiom $\alpha$
(equivalently, \I is a \emph{model} of $\alpha$) if \I satisfies the
corresponding semantic condition in Table~\ref{tab:syntax-semantics}.
When \I satisfies $\alpha$ we write $\I\models \alpha$.
We will sometimes use axioms of the form $C\equiv D$, that are
abbreviations for the pair of inclusions $C\sqsubseteq D$ and
$D\sqsubseteq C$.

\begin{table}[h]
  \begin{center}
    \leavevmode
    \footnotesize
    \begin{tabular}{p{4.9em}cp{26em}}
      \hline
      \hline
      Name &Syntax&Semantics\\
      \hline
      \hline
      \multicolumn{3}{l}{\easycal{SRIQ} concept and role expressions}
      \\
      \hline
      & &\\[-1em]
      inverse & $R^-$ & $\{ (y,x) \mid (x,y)\in R^\I \}$ \quad ($R\in\NR$)
      \\[-2pt]
      roles & &
      \\
      & &\\[-1em]
      top&$\top$& $\top^\I=\Delta^\I$ \\
      & &\\[-1em]
      bottom&$\bot$& $\bot^\I=\emptyset$ \\
      & &\\[-1em]
      intersection&$C\sqcap D$&$(C\sqcap D)^\I=C^\I\cap D^\I$\\
      & &\\[-1em]
      union&$C\sqcup D$&$(C\sqcup D)^\I=C^\I\cup D^\I$\\
      & &\\[-1em]
      complement&$\neg C$& $(\neg C)^\I=\Delta^\I\setminus C^\I $
      \\
      & &\\[-1em]
      existential
      &$\exists R. C$& 
      $\{ d \in \Delta^\I\mid \exists (d,e) \in R^\I : e\in C^\I \}$\\[-2pt]
      restriction & &
      \\
      universal
      &$\forall R. C$& 
      $\{ d \in \Delta^\I\mid \forall (d,e) \in R^\I : e\in C^\I \}$\\[-2pt]
      restriction & &
      \\
      number & $({\bowtie}\,n\ S.C)$ & $\big\{ x\in\Delta^\I \mid
      \#\{y\mid (x,y)\in S^\I \land y\in C^\I\} \bowtie n \big\}$\quad ($\bowtie = \leq, \geq$)
      \\[-2pt]
      restrictions & &
      \\
      self & $\exists S.\mathsf{Self}$ & $\{x\in\Delta^\I \mid (x,x) \in S^\I\}$
      \\[2pt]
      \hline
      \multicolumn{3}{l}{\easycal{SRIQ} terminological axioms}
      \\
      \hline
      \\[-1em]
      GCI & $C\sqsubseteq D$ & $C^\I\subseteq D^\I$
      \\[2pt]
      role & $\disj(S_1,S_2)$ & $S_1^\I\cap S_2^\I = \emptyset$
      \\[-2pt]
      disjointness & &
      \\
      complex & $R_1\circ\!...\!\circ R_n \sqsubseteq R$ &
      $R_1^\I\circ\ldots\circ R_n^\I \subseteq R^\I$
      \\[-2pt]
      \multicolumn{2}{l}{role inclusions} &
      \\
      \hline
      \multicolumn{3}{l}{\easycal{SRIQ} assertion axioms  \quad ($a,b\in\NI$)}
      \\
      \hline
      \\[-1em]
      conc.\ assrt. & $C(a)$ & $a^\I \in C^\I$
      \\
      role assrt. & $R(a,b)$ & $(a,b)^\I \in R^\I$
      \\[2pt]
      \hline
      \multicolumn{3}{l}{Other concept and role expressions}
      \\
      \hline
      & &\\[-1em]
      nominals
      &$\{a\}$& $\{ a\}^\I=\{a^\I\}$ \quad ($a\in\NI$)
      \\[2pt]
       universal 
       & $U$ & $U^\I = \Delta^\I \times \Delta^\I$ \\[-2pt]
       role & &
      \\
       concrete
      & $p(f_1,{..},f_n)$ &
       $\{x{\in}\Delta^\I \,|\, \exists \vec v {\in} (\Delta^\ND)^n.\,
       (x,v_i)\in  f_i^\I  \hspace*{.29em}(1\leq i\leq n)$
       and $\vec v \in p^\ND \}$\\[-2pt]
       constraints & &
       \\ \hline
       \multicolumn{3}{l}{Other terminological axioms}
       \\
       \hline
       & &\\[-1em]
      disjointness & $\disj(C,D)$ & $C^\I\cap D^\I = \emptyset$
      \\
      & &\\[-1em]
      functionality & $\func(R)$ & $R^\I$ is a partial function
      \\
      & &\\[-1em]
      range & $\range(R,C)$ & $R^\I\subseteq \Delta^\I\times C^\I$
      \\[2pt]
      \hline
    \end{tabular}
    \caption{Syntax and semantics of some DL constructs and axioms.}
    \label{tab:syntax-semantics}
  \end{center}
  \vspace*{-.8em}
\end{table}

A \emph{knowledge base} \K is a finite set of DL axioms.  Its \emph{terminological part} (or \emph{TBox}) is the set of terminological axioms\footnote{See Table~\ref{tab:syntax-semantics}.} in \K, while its \emph{ABox} is the set of its assertion axioms.


If $X$ is a DL expression or a knowledge base, then $\sig{X}$ denotes the signature consisting of all symbols occurring in $X$.  An interpretation \I of a signature $\Sigma^\I\supseteq\sig{\K}$ is a \emph{model} of \K (in symbols, $\I\models\K$) if \I satisfies all the axioms in \K. We say that \K \emph{entails} an axiom $\alpha$ (in symbols, $\K\models\alpha$) if all the models of \K satisfy $\alpha$.

A \emph{pointed interpretation} is a pair $(\I,d)$ where $d\in\Delta^\I$. We say $(\I,d)$ \emph{satisfies} a concept $C$  iff $d\in C^\I$.  In this case, we write $(\I,d)\models C$.

\subsection{The description logics used in this paper}

The logic \SRIQ supports the \SRIQ expressions and axioms illustrated
in Table~\ref{tab:syntax-semantics}. In a \SRIQ knowledge base, in
order to preserve decidability, the set of role axioms should be
\emph{regular} and the roles $S,S_1,S_2$ \emph{simple}, according to
the definitions stated in \cite{DBLP:conf/kr/HorrocksKS06}.
Horn-\SRIQ further restricts \SRIQ GCIs as specified in
\cite{DBLP:conf/kr/OrtizRS10}.  For simplicity, here we illustrate
only the normal form adopted in \cite{DBLP:conf/ijcai/OrtizRS11}, see
Table~\ref{horn-SRIQ}.

\begin{table}[h]
  \begin{center}
    \small
    \begin{tabular}{l}
      \hline
      $C_1 \sqcap C_2 \sqsubseteq D$
      \\
      $\exists R.C \sqsubseteq D$
      \\
      $C \sqsubseteq \forall R.D$
      \\
      $C \sqsubseteq \exists R.D$
      \\
      $C \sqsubseteq {} \leq 1\ S.D$
      \\
      $C \sqsubseteq {} \geq n\ S.D$
      \\
      \hline
      \\[-1em]
      $C,C_1,C_2,D$ either belong to $\NC \cup \{\bot,\top\}$,
      or are of the form $\exists S.\mathsf{Self}$
      \\
      $S$ is a \emph{simple} role \cite{DBLP:conf/kr/HorrocksKS06}
    \end{tabular}
  \end{center}
  \caption{The Horn restriction of \SRIQ GCIs (normal form)}
  \label{horn-SRIQ}
\end{table}

\noindent
Like all Horn DLs, Horn-\SRIQ is \emph{convex}, that is, $\K\models C_0\sqsubseteq C_1\sqcup C_2$ holds iff either $\K\models C_0\sqsubseteq C_1$ or  $\K\models C_0\sqsubseteq C_2$

The logic \EL is a fragment of Horn-\SRIQ that supports only atomic
roles, $\top$, $\sqcap$, and existential restrictions. Supported
axioms are GCIs and assertions. We will denote with $\EL^+$ the extension of \EL
with $\bot$ and $\range$ axioms. \ELpp denotes the extension of
$\EL^+$ with concrete domains.  Subsumption checking and consistency
checking are tractable in \EL and $\EL^+$.  The same holds for \ELpp
provided that concrete domains have a tractable entailment problem and
are \emph{convex}, in the sense that $\models p_1(\vec
f_1)\lor\ldots\lor p_n(\vec f_n)$ holds iff $\models p_i(\vec f_i)$
holds for some $i\in[1,n]$ \cite{DBLP:conf/ijcai/BaaderBL05}.

The logic \DLL is a fragment of Horn-\SRIQ that supports only inverse
roles, unqualified existential restrictions (i.e.\ concepts of the
form $\exists R.\top$), GCIs and assertions.  Moreover, complements
($\neg$) are allowed on the right-hand side of GCIs.  \DLLh extends
\DLL by supporting $\sqcap$ and role inclusions of the form
$R_1\sqsubseteq R_2$. Subsumption and consistency checking are
tractable in both logics.

The logic $\SROIQ(\ND)$ supports all the constructs and axioms illustrated in Table~\ref{tab:syntax-semantics}. It is the description logic underlying the standard OWL2-DL.

\subsection{The disjoint model union property}

A knowledge base \K such that $\sig{\K}\cap\NI=\emptyset$ enjoys the \emph{disjoint model union property} if for all disjoint models \I and \J of \K, their disjoint union $\I\uplus\J=\tup{\Delta^{\I}\uplus\Delta^{\J},\cdot^{\I\uplus \J}}$ -- where $P^{\I\uplus\J}=P^\I \uplus P^\J$ for all $P\in \NC\cup\NR\cup\NF$ --  satisfies \K, too (\cite{DBLP:conf/dlog/2003handbook}, Ch.~5). This definition is extended naturally to the union $\biguplus S$ of an arbitrary set $S$ of disjoint models.  The disjoint model union property plays an important role in our results. It is broken by the universal role and nominals.
The main problem with nominals (and the reason of the prerequisite $\sig{\K} \cap \NI=\emptyset$) is that if \I and \J are disjoint, then for all individual constants $a\in\NI$, $a^\I\neq a^\J$, so it is not immediately clear what $a^{\I\uplus\J}$ should be. The problem can be resolved for the constants occurring in ABoxes. Informally speaking, it suffices to pick the constants' interpretation from an arbitrary argument of the union.\footnote{The following formalization of this idea generalizes a proof technique used in \cite[Lemma~1]{DBLP:conf/ijcai/GrauMK09}.}

\begin{definition}[Generalized disjoint union]
  For all sets of mutually disjoint interpretations $S$ and all $\I\in
  S$, let $\biguplus^\I S$ be the interpretation \U such that:
  \begin{eqnarray*}
    \Delta^\U & = & \bigcup \{\Delta^\J \mid \J\in S\}
    \\
    P^\U & = & \bigcup \{P^\J \mid \J\in S\} \quad\mbox{for all } P\in \NC\cup\NR\cup\NF
    \\
    a^\U & = & a^\I  \hspace*{71pt}\mbox{for all } a\in \NI \,.
  \end{eqnarray*}
\end{definition}

\noindent
If the terminological part of a knowledge base \K has the (standard) disjoint model union property, then the generalized union of disjoint models of \K is still a model of \K:

\begin{proposition}
  \label{prop:x-union}
  Let $\K = \T \cup \A$, where \T is the terminological part of \K and \A is its ABox. If \T has the disjoint model union property then for all sets $S$ of mutually disjoint models of \K, and for all $\I\in S$, $\biguplus^\I S \models \K$.
\end{proposition}

\begin{proof}
Let $S$ and \I be as in the statement, and let $\U=\biguplus^\I S$. Note that $\sig{\T}\cap
\NI=\emptyset$, otherwise the disjoint union of \T's models would not
be defined and \T would not enjoy the disjoint model union property,
contradicting the hypothesis.
For all interpretations \J, let $\J{\setminus\NI}$ denote the
restriction of \J to the symbols in $\NC\cup\NR\cup\NF$
(i.e.\ excluding the individual constants in \NI).
Note that for all $\J\in S$, $\J{\setminus\NI}$ is a model of \T, because $\sig{\T}\cap
\NI=\emptyset$. Therefore, by hypothesis, $\biguplus \{\J{\setminus\NI} \mid \J\in S\}$ is a model of \T.
Clearly, $\biguplus \{\J{\setminus\NI} \mid \J\in S\} = (\biguplus^\I S){\setminus\NI}$; as a consequence, also $\biguplus^\I S$ is a model of \T.  We are only left to prove that \U is a model of \A.   Consider
an arbitrary assertion $\alpha\in\A$. Since the models in $S$ are disjoint, and the interpretation of constants in \U ranges over $\Delta^\I$, it holds that $a^\U\in C^\U$ iff $a^\I\in C^\I$,   $(a^\U,b^\U)\in R^\U$ iff $(a^\I,b^\I)\in R^\I$, and $f^\U(a^\U)=f^\I(a^\I)$ ($f\in\NF$).  Moreover, \I is a model of \A by hypothesis. It follows immediately that \U is a model of \A.
\end{proof}

\subsection{Modularity and locality}

A knowledge base $\K$ is \emph{semantically modular} with respect to a signature $\Sigma$ if each interpretation $\I=(\Delta^\I, \cdot^\I)$ over $\Sigma$ can be extended to a model $\J=(\Delta^\J, \cdot^\J)$ of $\K$ such that $\Delta^\J=\Delta^\I$ and $X^\J=X^\I$, for all symbols $X\in \Sigma$. Roughly speaking, this means that \K does not constrain the symbols of $\Sigma$ in any way.

A special case of semantic modularity exploited in  \cite{DBLP:conf/ijcai/GrauMK09} is \emph{locality}: A knowledge base $\K$ is \emph{local} with respect to a signature $\Sigma$ if the above \J can be obtained simply as specified in the next definition.

\begin{definition}[Locality]
  A knowledge base $\K$ is \emph{local} with respect to a signature
  $\Sigma$ if each interpretation $\I=(\Delta^\I, \cdot^\I)$ over
  $\Sigma$ can be extended to a model $\J=(\Delta^\J, \cdot^\J)$ of
  $\K$ by setting $X^\I=\emptyset$ for all concept and role names
  $X\in \Sigma(\K)\setminus\Sigma$.
\end{definition}
Locality will be needed in Section~\ref{sec:oracles}, for the
integration of \PL knowledge bases with imported
ontologies.  In particular, it is an essential ingredient of the
completeness proof for IBQ reasoning.



\section{Semantic Encoding of Data Usage Policies}
\label{sec:enc}

SPECIAL's policy language \PL\ -- that is a fragment of OWL2-DL -- has been
designed to describe data usage. Such descriptions can be exploited
to encode: (i) the consent to data processing given by data subjects,
(ii) how the controller's internal processes use data, and (iii)
selected parts of the GDPR that can be used to support the validation
of the controller's internal processes.  Moreover, \PL is used to
encode the entries of SPECIAL's \emph{transparency ledger}, that is a log of
data processing operations that can be queried by:
\begin{itemize}
\item data subjects, in order to monitor how their personal data are used by
  the controller and where they are transferred to;
\item data protection officers, in order to audit the behavior of the controller;
\item the controllers themselves, in order to monitor their own internal processes.
\end{itemize}

\noindent
The aspects of data usage that have legal relevance are clearly
indicated in several articles of the GDPR and in the available
guidelines. They are mentioned, for example, in the specification of
what is valid consent, what are the legal bases for processing, what
are the rights of data subjects, which aspects should be covered by
national regulations, and the obligation of controllers to keep a
record of the processing operations that involve personal data (see,
inter alia, articles~6.1, 6.3, 6.4, 7, 15.1, 23.2, 23.2, 30.1). See
also the section titled ``Records should contain'' in the guidelines
for SMEs published on
\url{http://ec.europa.eu/justice/smedataprotect/index_en.htm}. That
section describes how to fulfill the obligation to record
the data subjects' consent to processing (Article~7) and, in particular, it specifies which pieces of information should be recorded.
According to the above sources of requirements, the main properties of data usage that
need to be encoded and archived are the following:
\begin{itemize}
\item reasons for data processing (purpose);
\item which data categories are involved;
\item what kind of processing is applied to the data;
\item which third parties data are distributed to (recipients);
\item countries in which the data will be stored (location);
\item time constraints on data erasure (duration).
\end{itemize}

\noindent
The above properties characterize a \emph{usage policy}. \SPECIAL adopts a direct encoding of usage policies in description
logics, based on those features. The simplest possible policies have the form:
{
\begin{equation}
  \label{pol1}
  \renewcommand{\arraystretch}{1.3}
  \begin{array}{l}
    \exists \purp.P \sqcap \exists \data.D \sqcap \exists \proc.O
    \sqcap \exists \recip.R \sqcap {}\\
    ~~ \exists \stor(\exists \loc.L \sqcap \exists \dur.T) \,.
  \end{array}
\end{equation}%
}%
All of the above roles are functional.
Duration is represented as an interval of integers
$[t_1,t_2]$, representing a minimum and a maximum storage time (such
bounds may be required by law, by the data subject, or by the controller itself).
The classes $P$, $D$, $O$, etc.\ are defined in suitable
\emph{auxiliary vocabularies} (ontologies) that specify also the
relationships between different terms. The expressiveness requirements
on the vocabularies and their design are discussed later, in
Section~\ref{sec:oracles}. Until then, \emph{the reader may assume
  that the vocabularies are defined by means of inclusions
  $A\sqsubseteq B$ and disjointness constraints $\disj(A,B)$}, where
$A,B$ are concept \emph{names}. Such restrictions will be lifted
later.

If the data subject consents to a policy of the form (\ref{pol1}), then she authorizes all of its instances. For example if $D=\mathsf{DemographicData}$ then the data subject authorizes -- in particular -- the use of her address, age, income, etc.\ as specified by the other properties of the policy.

It frequently happens that the data controller intends to use different data categories in different ways, according to their usefulness and sensitivity, so consent requests comprise multiple \emph{simple usage policies} like (\ref{pol1}) (one for each usage type). The intended meaning is that consent is requested for all the instances of all those policies; accordingly, such a compound policy is formalized with the union of its components. The result is called \emph{full} (usage) \emph{policy} and has the form:
\begin{equation}
  \label{pol2}
  P_1 \sqcup \ldots \sqcup P_n
\end{equation}
where each $P_i$ is a simple usage policy of the form (\ref{pol1}).
Symmetrically, with a similar union, data subjects may consent to different usage modalities for different categories of data and different purposes.

\begin{example}\rm
  \label{ex:befit}
  A company -- call it BeFit -- sells a wearable fitness
  appliance and wants (i) to process biometric data (stored in the EU) for
  sending health-related advice to its customers, and (ii) share the customer's location
  data with their friends. Location data are kept for a minimum of one year but no longer
  than 5; biometric data are kept for an unspecified amount of time. In order to do all
  this legally, BeFit needs consent from its customers. The internal (formalized)
  description of such consent would look as follows:%
{
  \begin{equation}
    \label{ex1:0}
    \renewcommand{\arraystretch}{1.3}
    \begin{array}{l}
      ( \exists\purp.\mathsf{FitnessRecommendation} \sqcap {}\\
      ~~
      \exists\data.\mathsf{BiometricData} \sqcap {}\\
      ~ ~
      \exists\proc.\mathsf{Analytics} \sqcap {}\\
      ~ ~
      \exists\recip.\mathsf{BeFit} \sqcap {}\\
      ~ ~
      \exists\stor.\loc.\mathsf{EU} )\\
      \sqcup \\
      ( \exists\purp.\mathsf{SocialNetworking} \sqcap {}\\
      ~ ~
      \exists\data.\mathsf{LocationData} \sqcap {}\\
      ~ ~
      \exists\proc.\mathsf{Transfer} \sqcap {}\\
      ~ ~
      \exists\recip.\mathsf{DataSubjFriends} \sqcap {}\\
      ~ ~
      \exists\stor.( \exists \loc.\mathsf{EU} \sqcap \exists \dur.[y_1,y_5])\,.      
    \end{array}
  \end{equation}%
}%
Here $y_1$ and $y_5$ are the  integer representation of one
year and five years, respectively.  If ``\easysf{HeartRate}'' is a
subclass of ``\easysf{BiometricData}'' and ``\easysf{ComputeAvg}'' is
a subclass of ``\easysf{Analytics}'', then the above consent allows
BeFit to compute the average heart rate of the data subject in order
to send her fitness recommendations.
BeFit customers may restrict their consent, e.g.\ by picking a
specific recommendation modality, like ``recommendation via
SMS only''. Then the first line should be replaced with something like
$\exists\purp.(\mathsf{FitnessRecommendation}\sqcap\exists\mathsf{contact}.\mathsf{SMS})$.
Moreover, a customer of BeFit may consent to the first or the second argument of
the union, or both.  Then her consent would be encoded, respectively,
with the first argument, the second argument, or the entire concept
(\ref{ex1:0}).  Similarly, each single process in the controller's
lines of business may use only biometric data, only location data, or
both.  Accordingly, it may be associated to the first simple policy,
the second simple policy, or their union. In other words,
(\ref{ex1:0}) models the complete data usage activities related to the
wearable device, that may be split across different processes.
\qed
\end{example}

The usage policies that are actually applied by the data controller's
business processes are called \emph{business policies} and include a
description of data usage of the form (\ref{pol1}).  Additionally,
each business policy is labelled with its legal basis and describes
the associated obligations that must be fulfilled. For example, if the
data category includes personal data, and processing is allowed by
explicit consent, then the business policy should have the additional
conjuncts:
{
  \begin{equation}
  \label{pol3}
  \renewcommand{\arraystretch}{1.3}
  \begin{array}{l}
    \exists \mathsf{has\_legal\_basis. Art6\_1\_a\_Consent} \sqcap {} \\
    ~ ~
    \exists\duty.\mathsf{GetConsent} \sqcap \exists\duty.\mathsf{GiveAccess} \sqcap {}\\
    ~ ~
    \exists\duty.\mathsf{RectifyOnRequest} \sqcap {}\\
    ~ ~
    \exists\duty.\mathsf{DeleteOnRequest} 
  \end{array}
  \end{equation}%
}%

\noindent
that label the policy with the chosen legal basis, and model the obligations related to the data subjects' rights, cf.\ Chapter 3 of the GDPR. More precisely, the terms involving \duty assert that the process modelled by the business policy includes the operations needed to obtain the data subject's consent ($\exists\duty.\easysf{GetConsent}$) and those needed to receive and apply the data subjects' requests to access, rectify, and delete their personal data.

Thus, business policies are an abstract description of a business process, highlighting the aspects related to compliance with the GDPR and data subjects' consent.
Similarly to consent, a business policy may be a union
$\easyit{BP_1\sqcup\ldots\sqcup BP_n}$ of simple business policies \easyit{BP_i} of the form
$(\ref{pol1})\sqcap (\ref{pol3})$.

In order to check whether a business process complies with the consent
given by a data subject $S$, it suffices to check whether the
corresponding business policy \easyit{BP} is subsumed by the consent
policy of $S$, denoted by \easyit{CP_S} (in symbols,
\easyit{BP\sqsubseteq CP_S}). This subsumption is checked against a
knowledge base that encodes type restrictions related to policy
properties and the corresponding vocabularies, i.e.\ subclass
relationships, disjointness constraints, functionality restrictions,
domain and range restrictions, and the like.  Some examples of the actual axioms occurring in the knowlede base are:
\[
\renewcommand{\arraystretch}{1.3}
\begin{array}{l}
  \easysf{\func(\purp)} \\
  \easysf{\range(\data,AnyData)} \\
  \easysf{Demographic \sqsubseteq AnyData} \\
  \easysf{Update \sqsubseteq AnyProcessing} \\
  \easysf{Erase \sqsubseteq Update} \\
  \easysf{\disj(AnyData, AnyPurpose)}
\end{array}
\]
(recall that more general knowledge bases will be discussed later).

In order to verify that all the required obligations are fulfilled by a business process (as abstracted by the business policy),  selected parts of the GDPR are formalized with concepts like the following.  The first concept states that a business policy should either support the rights of the data subjects, or concern anonymous data, or it should fall under some of the exceptional cases mentioned by the regulation, such as particular law requirements.  The remaining requirement are not listed here (they are replaced with an ellipsis):
{
\begin{equation}
  \label{gdpr1}
  \renewcommand{\arraystretch}{1.3}
  \begin{array}{l}
    (\exists\duty.\mathsf{GetConsent} \sqcap
    \exists\duty.\mathsf{GiveAccess} \sqcap \ldots) \sqcup {} \\
    ~ ~  \exists\data.\mathsf{Anonymous} \sqcup {}\\
    ~ ~  \exists\purp.\mathsf{LawRequirement} \sqcup \ldots
  \end{array}
\end{equation}%
}%
The second example encodes the constraints on data transfers specified in Articles~44--49 of the GDPR:
{
\begin{equation}
  \label{gdpr2}
  \renewcommand{\arraystretch}{1.3}
  \begin{array}{l}
    \exists\stor.\loc.\mathsf{EU} \sqcup {} \\
    ~ ~  \exists\stor.\loc.\mathsf{EULike} \sqcup \ldots 
  \end{array}
\end{equation}%
}%
It states that data should remain within the EU, or countries that adopt similar data protection regulations. The ellipsis stands for further concepts that model the other conditions under which data can be transferred to other nations (e.g.\ under suitable binding corporate rules). Please note that the above concepts constitute only a largely incomplete illustration of the actual formalization of the GDPR, that is significantly longer due to the special provisions that apply to particular data categories and legal bases. The purpose of the above examples is conveying the flavor of the formalization. Its usage is sketched below.

A business policy \easyit{BP} can be checked for compliance with the formalized parts of the GDPR by checking whether the aforementioned knowledge base entails that \easyit{BP} is subsumed by the concepts that formalize the GDPR.

\begin{example}\rm
  \label{ex:GDPR-checking}
  The following business policy complies with the consent-related
  obligations formalized in (\ref{gdpr1}) since it is subsumed by it:
{
    \begin{equation}
    \label{ex1:1}
    \renewcommand{\arraystretch}{1.3}
    \begin{array}{l}
      ( \exists\purp.\mathsf{FitnessRecommendation} \sqcap {}\\
      ~ ~
      \exists\data.\mathsf{BiometricData} \sqcap {}\\
      ~ ~
      \exists\proc.\mathsf{Analytics} \sqcap {}\\
      ~ ~
      \exists\recip.\mathsf{BeFit} \sqcap {}\\
      ~ ~
      \exists\stor.\loc.\mathsf{EU}  \sqcap {}\\
      ~ ~
      \exists\mathsf{has\_legal\_basis . Art6\_1\_a\_Consent} ) \sqcap {} \\
      ~ ~
      \exists\duty. \mathsf{GetConsent} \sqcap \ldots \mbox{ \emph{all the remaining concepts in}  (\ref{pol3})} \ldots )\\
      \sqcup \\
      ( \exists\purp.\mathsf{Sell} \sqcap {}\\
      ~ ~
      \exists\data.\mathsf{Anonymous} \sqcap {}\\
      ~ ~
      \exists\proc.\mathsf{Transfer} \sqcap {}\\
      ~ ~
      \exists\recip.\mathsf{ThirdParty}  )\,.      
    \end{array}
    \end{equation}%
}%
In particular, the two disjuncts of (\ref{ex1:1}) are subsumed by the first two lines of  (\ref{gdpr1}), respectively.  Note that the second simple policy does not place any restrictions on location, so it allows data to flow to any country, including those that do not enjoy adequate data protection regulations.  However, this is compliant with the GDPR because data are anonymous.
  \qed
\end{example}

\noindent
The concepts in the range of existential restrictions may themselves
be a conjunction of atoms, interval constraints and existential
restrictions. We have already seen in policy (\ref{ex1:0}) that \stor
may contain a conjunction of existential restrictions over properties
\loc and \dur.  Another example, related to SPECIAL's pilots, concerns
the accuracy of locations, that can be modelled with concepts like:
\begin{equation*}
  \exists \data . (\easysf{Location  \sqcap \exists has\_accuracy. Medium} ) \,.
\end{equation*}

\hide{
Moreover, policies can be nested in case of data transfers. A
\emph{sticky policy} \cite{5959137} is a usage policy associated to
data.
The data and the associated sticky policies are transferred together;
the recipient should use the data as specified by its sticky
policy. In SPECIAL's notation this can be expressed with nested policies like:
\begin{equation*}
  \renewcommand{\arraystretch}{1.3}
  \begin{array}{ll}
    \lefteqn{ \exists \proc. \easysf{Transfer} \sqcap {} }
    \\
    \exists \data . (
    & \mbox{\it\textless data category\textgreater} \sqcap {}
    \\
    & \exists \easysf{has\_sticky\_policy}.(
    \\
    & \hspace*{2em} \exists \proc.  \mbox{\it\textless processing category\textgreater} \sqcap {}
    \\
    & \hspace*{2em} \exists \purp.  \mbox{\it\textless purpose category\textgreater} \sqcap {}
    \\
    & \hspace*{2em} \ldots
    \\
    \lefteqn{ ) \sqcap \exists \recip \ldots }
  \end{array}
\end{equation*}
}

\noindent
Based on the above discussion, we are now ready to specify \PL
(\emph{policy logic}), a fragment of OWL~2 that covers -- and slightly
generalizes -- the encoding of the usage policies and of the GDPR
outlined above.
\begin{definition}[Policy logic \PL]
  A \emph{\PL knowledge base} \K is a set of axioms of the following kinds:
  \begin{itemize}
  \item $\func(R)$ where $R$ is a role name or a concrete property;
  \item $\range(S,A)$ where $S$ is a role and $A$ a concept name;
  \item $A\sqsubseteq B$ where $A,B$ are concept names;
  \item $\disj(A,B)$ where $A,B$ are concept names.
  \end{itemize}
  \emph{Simple \PL concepts} are defined by the following grammar,
  where $A\in\NC$ $R\in\NR$, and $f\in\NF:$
  $$C::=A\mid \bot\mid \exists f.[l,u] \mid \exists R.C \mid C \sqcap C \,.$$
  A \emph{(full) \PL concept} is a union $D_1\sqcup\ldots\sqcup D_n$
  of simple \PL concepts ($n\geq 1$).  \PL's \emph{subsumption
    queries} are expressions $C\sqsubseteq D$ where $C,D$ are (full)
  \PL concepts.
\end{definition}

\subsection{Discussion of the encoding}

The formalization of policies as \emph{classes} of data usage
modalities addresses several needs.

First, on the controller's side, each instance of a process may
slightly differ from the others. For example, different instances of a
same process may operate on data that are stored in different servers,
possibly in different nations (this typically happens to large,
international companies). The concrete data items involved may change
slightly (e.g.\ age may be expressed directly or through the birth
date; the data subject may be identified via a social security number
(SSN), or an identity card number, or a passport number). By describing storage location, data,
and the other policy attributes as classes, controllers can concisely
describe an entire collection of similar process instances.  With
reference to the above examples, classes allow to express that data
are stored ``somewhere in the EU'' and ``in the controller's
servers''; both age and birthdate fall under the class of demographic
data; SSN and document numbers can be grouped under the class of
unique identifiers.

A second advantage of classes is that they support a rather free
choice of granularity. For example, the classes that model locations
can be formulated at the granularity of continents, federations,
countries, cities, zip-codes, down to buildings and rooms. Subsumption
naturally models the containment of regions into other regions.
A flexible choice of granularity helps in turning company documentation into formalized business policies, since it facilitates the import of the abstractions spontaneously used by domain experts.

The third, and perhaps most important advantage is that classes
\emph{facilitate the reuse of consent}. The GDPR sometimes allows to
process personal data for a purpose other than that for which the data
has been collected, provided that the new purpose is ``compatible''
with the initial purpose.\footnote{See for example articles~5.1 (b)
  and 6.4.}  Compatibility cannot be assessed automatically, in
general, because it is not formalized in the regulation, and involves
enough subtleties to need the assessment of a lawyer. However, by
expressing purposes as classes, one can at least have the data subject
consent upfront to a specified range of ``similar'' purposes.
Roughly speaking, the accepted class of purposes is like an agreement -- between data subjects and controllers -- on which purposes are ``compatible'' in the given context.
Also expressing the other policy properties as classes is beneficial. As data subjects consent to wider  classes of usage modalities,  the need for additional consent requests tends to decrease; this may yield benefits to both parties, because:
\begin{enumerate}
\item data subjects are disturbed less frequently with consent requests (improved usability, better user experience);
\item the costs associated to consent requests decrease. Consider that sometimes the difficulties related to reaching out to the data subjects, and the concern that too many requests may annoy users, make controllers decide \emph{not} to deliver a service that requires additional consent.
\end{enumerate}

From a theoretical viewpoint, the class-based policy formalization adopted by
SPECIAL is essentially akin to a well-established policy composition algebra
\cite{DBLP:journals/tissec/BonattiVS02}. The algebra treats policies as classes of authorizations (each policy $P$ is identified with the set of authorizations permitted by $P$). In turn, authorizations are tuples that encode the essential elements of permitted operations, such as the resources involved and the kind of processing applied to those resources.
Analogously, each \PL policy like (\ref{pol1}) denotes a set of reifications of tuples, whose elements capture the legally relevant
properties of data usage operations.

\subsection{Related policy languages}

Logic-based languages constitute natural policy languages, because
\emph{policies are knowledge}.  First, note that policies encode declarative constraints on a system's behavior, that depend on metadata about the actors and the objects involved (e.g.\ ownership, content categories),\hide{the operations,} and an environment (as some operations may be permitted only in certain places, or at specified times of the day, or in case of emergency). Semantic languages and formats have been expressly designed to encode metadata, so standard knowledge representation languages can represent in a uniform way both policy constraints and the metadata they depend on.

The second important observation is that -- like knowledge and unlike programs -- every single  policy  is meant to be used for multiple, semantically related tasks, such as the following:
\begin{itemize}
\item \emph{permission checking}: given an operation request, decide whether it is permitted;
\item \emph{compliance checking}: does a policy $P_1$ fulfill all the restrictions requested by policy $P_2$? (Policy comparison);
\item \emph{policy validation}: e.g.\ is the policy contradictory? Does it comply with a given regulation? Does a policy update strengthen or relax the previous policy?
\item \emph{policy explanation}: explain a policy and its decisions.
\end{itemize}

\noindent
The terse formal semantics of logical languages is essential in
validating the correctness of the policies themselves and the implementation of the above tasks, ensuring their mutual
coherence. Moreover, when data are
transferred under agreed policies, it is crucial that both parties
understand the policies in the same way. So unambiguous semantics is
essential for correct interoperability, too.

In the light of the above observations, it is clear that  knowledge representation languages are ideal policy representation languages.
Indeed, both rule languages and description logics have already been used as policy languages; a non-exhaustive list is \cite{DBLP:journals/csec/WooL93,DBLP:journals/tods/JajodiaSSS01,DBLP:conf/policy/UszokBJSHBBJKL03,rei,DBLP:journals/tkde/BonattiCOS10}. As noted in \cite{DBLP:conf/datalog/Bonatti10}, the advantage of rule languages is that they can express $n$-ary authorization conditions for arbitrary $n$, while encoding such conditions for $n>2$ is challenging in DL. The advantage of DL is that all the main policy-reasoning tasks are decidable (and tractable if policies can be expressed with OWL~2 profiles), while compliance checking is undecidable  in rule languages, or at least intractable, in the absence of recursion, because it can be reduced to datalog query containment. So a DL-based policy language is a natural choice in a project like SPECIAL, where policy comparison is the predominant task.

The aforementioned works on logic-based policy languages focus on access control and trust management, rather than data usage control. Consequently, those languages lack the terms for expressing privacy-related and usage-related concepts.  A more serious drawback is that the main reasoning task in those papers is permission checking; policy comparison (which is central to our work) is not considered.  Both Rei and Protune \cite{rei,DBLP:journals/tkde/BonattiCOS10} support logic program rules. We have already mentioned that if rules are recursive, then policy comparison is generally undecidable; it is \NP-hard if rules are not recursive.  This drawback makes such languages unsuitable to SPECIAL's purposes. Similarly, KAoS \cite{DBLP:conf/policy/UszokBJSHBBJKL03} is based on a DL that, in general, is not tractable, and supports role-value maps -- a construct that easily makes reasoning undecidable (see \cite{DBLP:conf/dlog/2003handbook}, Chap.~5). The papers on KAoS do not discuss how the policy language is restricted to avoid this issue.

The terms used as role fillers in SPECIAL's policies are imported from
well established formats for expressing privacy preferences and
digital rights, such as P3P (the Platform for Privacy
Preferences)\footnote{ \url{http://www.w3.org/TR/P3P11}} and ODRL (the
Open Digital Right
Language).\footnote{\url{https://www.w3.org/TR/odrl/}} More general
vocabularies will be discussed in Section~\ref{sec:oracles}.
It is interesting to note that P3P's privacy policies -- that are
encoded in XML -- are almost identical to simple \PL policies: the tag
\easytt{STATEMENT} contains tags \easytt{PURPOSE}, \easytt{RECIPIENT},
\easytt{RETENTION}, and \easytt{DATA\mbox{-}GROUP}, that correspond to
the analogous properties of SPECIAL's usage policies.  Only the
information on the location of data is missing. The tag
\easytt{STATEMENT} is included in a larger context that adds
information about the controller (tag \easytt{ENTITY}) and about the
space of web resources covered by the policy (through so-called \emph{policy reference
files}). All of these additional pieces of information can be directly
encoded with simple \PL concepts.
Similar considerations hold for ODRL. The tag \easytt{RIGHTS}
associates an \easytt{ASSET} (the analogue of \data) to a
\easytt{PERMISSION} that specifies a usage modality. ODRL provides
terms for describing direct use (e.g.\ play or execute), reuse
(e.g.\ annotate or aggregate), transfer (sell, lend, lease), and asset
management operations (such as backup, install and delete, just to
name a few). These terms provide a rich vocabulary of fillers for the
\proc property of SPECIAL's policies. Also in the case of ODRL, the
tree-like structure of XML documents can be naturally encoded with \PL
concepts.

\section{Reasoning with \PL}
\label{sec:reasoner}

Some of the use cases of SPECIAL place challenging scalability
requirements on compliance checking.  For example, if the applicable
legal basis for processing is consent, then storing personal data
without permission is always unlawful, even if storage is temporary
and for the sole purpose of running a batch process to discard the
information items that cannot be persistently stored.  This means that the intense
flow of data produced by the communication infrastructure of Deutsche
Telekom or Proximus (two of SPECIAL's industrial partners) must be
filtered on the fly by checking the compliance of each storage
operation with the consent given by the involved customer.  In
general, the frequency of compliance tests can be high enough to place
real time requirements on the compliance checker.

These scalability requirements have been addressed by
finding a tradeoff between expressiveness and efficiency.  The
language \PL\ -- that is rich enough to encode the policies of
interest -- is also rather simple. Actually, \PL would be a fragment
of the tractable description logic \ELpp
\cite{DBLP:conf/ijcai/BaaderBL05}, if it did not support functional roles and interval
constraints (that constitute a non-convex datatype, while \ELpp
supports only convex domains).  The latter feature keeps \PL outside the
family of Horn Description Logics, which include the tractable
profiles of OWL2. As a consequence, no off-the-shelf solutions are
available to reason efficiently on \PL concepts.
Actually, we are going to show that in \PL unrestricted subsumption
checking is co\NP-hard.

However, we can exploit the structure of usage policies to make
restrictive assumptions on \PL concepts. Under such assumptions, we
can prove that an approach articulated in  two stages -- where first business policies
are suitably normalized, then compliance with consent policies is checked with
a structural subsumption algorithm -- is correct, complete, and
tractable.  Its scalability will be experimentally assessed in
Section~\ref{sec:experiments}.

We start by laying out the formal description and the theoretical properties of normalization and structural subsumption. In particular, this section deals with the correctness and completeness of the two-stages method, and discusses the computational complexity of arbitrary subsumptions and of the restricted, tractable case.  
We first prove the intractability of unrestricted subsumption in \PL. The root of intractability lies -- as it should be expected -- in the non-convex datatype, i.e.\ interval constraints.

\begin{theorem}
  \label{thm:PL-is-NP-hard}
  Deciding whether $\K\models C\sqsubseteq D$, where \K is a \PL
  knowledge base and $C,D$ are \PL concepts, is
  co\NP-hard. This statement holds even if the knowledge base is
  empty and $C$ is simple.
\end{theorem}

\begin{proof}
    Hardness is proved by reducing 3SAT to the complement of
  subsumption. Let $S$ be a given set of clauses $c_i=L_{i1}\lor
  L_{i2}\lor L_{i3}$ ($1\leq i\leq n$) where each $L_{ij}$ is a
  literal.  We are going to use the propositional symbols
  $p_1,\ldots,p_m$ occurring in $S$ as property names in \PL concepts,
  and define a subsumption $C\sqsubseteq D$ that is valid iff $S$ is
  unsatisfiable. Let $C = \big(\, \exists p_1.[0,1] \sqcap \ldots \sqcap
  \exists p_m.[0,1] \,\big)$ and $D=\bigsqcup_{i=1}^n \big( \tilde L_{i1}
  \sqcap \tilde L_{i2} \sqcap \tilde L_{i3} \big)$, where each $\tilde
  L_{ij}$ encodes the complement of $L_{ij}$ as follows:
  \[
  \tilde L_{ij} = \left\{
  \renewcommand{\arraystretch}{1.2}
  \begin{array}{ll}
    \exists p_k.[0,0] & \mbox{if } L_{ij}=  p_k \,,
    \\{}
    \exists p_k.[1,1] & \mbox{if } L_{ij}= \neg p_k \,.
  \end{array}
  \right.
  \]
  The correspondence between the propositional interpretations $I$ of
  $S$ and the interpretations \J of $C\sqsubseteq D$ is the following.

  Given $I$ and an arbitrary element $d$, define $\J=\tup{\{d\},\cdot^\J}$ such that $(d,0)\in p_i^\J$ iff $I(p_i)=\mathit{false}$, and $(d,1)\in p_i^\J$ otherwise.  By construction, $(\J,d)\models C$, and  $I\models S$ iff $(\J,d)\not\models D$. Consequently, if $S$ is satisfiable, then $C\sqsubseteq D$ is not valid.

  Conversely, if $C\sqsubseteq D$ is not valid, then there exist \J and $d\in\Delta^\J$ such that $(\J,d) \models C\sqcap \neg D$. Define a propositional interpretation $I$ of $S$ by setting $I(p)=\mathit{true}$ iff $(d,1)\in p_i^\J$.  By construction (and since $d$ does not satisfy $D$ in \J), $I\models S$, which proves that if $C\sqsubseteq D$ is not valid, then $S$ is satisfiable.

  We conclude that the above reduction is correct. Moreover, it can be
  clearly computed in\hide{LOGSPACE} polynomial time. This proves that subsumption is
  co\NP-hard even if the knowledge base is empty and $C$ simple.
\end{proof}

\noindent
Later on we will complete the characterization of \PL subsumption by
proving that it is co\NP-complete (Theorem~\ref{thm:PL-is-in-NP}).

The above intractability result does not apply to SPECIAL's usage policies
because each simple usage policy contains at most one interval
constraint, namely, a specification of storage duration of the form
$\exists \stor.\exists \dur.[\ell,u]$. We are going to show that this
property (actually, a slight generalization thereof) makes reasoning
quite fast.  More specifically, it enables an efficient treatment of
interval constraints based on a suitable interval normalization
method. Such normalization produces subsumption queries that satisfy
the following property.

\begin{definition}[Interval safety]
   An inclusion $C\sqsubseteq D$ is \emph{interval safe} iff, for all
   constraints $\exists f.[\ell,u]$ occurring in $C$ and all $\exists
   f'.[\ell',u']$ occurring in $D$, either
   $[\ell,u]\subseteq[\ell',u']$, or $[\ell,u]\cap[\ell',u'] =
   \emptyset$.
\end{definition}

\noindent
Roughly speaking, interval safety removes the need of treating intervals like disjunctions; it makes them behave like plain atomic concepts.
Every inclusion can be turned into an equivalent, interval safe
inclusion, using the following method.

\begin{definition}[Interval normalization, \splt{C}{D}]
  For each constraint $\exists f.[\ell,u]$ in $C$, let $x_1<x_2<\cdots<x_r$ be the
  integers that occur as interval endpoints in $D$ and belong to
  $[\ell,u]$.  Let $x_0=\ell$ and $x_{r+1}=u$ and replace $\exists f.[\ell,u]$ with the equivalent concept
  {
    \begin{equation}
      \label{interval-splitting}
      \bigsqcup_{i=0}^r \big(\exists f.[x_i,x_i] \sqcup \exists f.[x_i+1,x_{i+1}-1] \big)
      \sqcup \exists f.[x_{r+1},x_{r+1}] \,.
    \end{equation}
  }%
  Then use distributivity of
  $\sqcap$ over $\sqcup$ and the equivalence $\exists R.(C_1\sqcup
  C_2)\equiv \exists R.C_1\sqcup \exists R.C_2$ to move all occurrences of $\sqcup$ to
  the top level. Denote the result of this interval normalization
  phase with \splt{C}{D}.
\end{definition}
\begin{example}\rm
  Let $C=\exists f.[1,9] \sqcap A$ and $D=\exists f.[5,12]$. Then
  $r=1$ and $x_0=1$, $x_1=5$, $x_2=9$ ($12$ falls outside $[1,9]$ and
  is ignored). According to (\ref{interval-splitting}), the concept
  $\exists f.[1,9]$ in $C$ is replaced by the following
  union: $$\exists f.[1,1] \sqcup \exists f.[2,4] \sqcup \exists
  f.[5,5] \sqcup \exists f.[6,8] \sqcup \exists f.[9,9] \,.$$ Then,
  after applying distributivity, we obtain the concept \splt{C}{D}
  (that is a full \PL concept): $$(\exists f.[1,1] \sqcap A) \sqcup
  (\exists f.[2,4]\sqcap A) \sqcup (\exists f.[5,5]\sqcap A) \sqcup
  (\exists f.[6,8]\sqcap A) \sqcup (\exists f.[9,9]\sqcap A)\,.$$ \qed
\end{example}

\noindent
The reader may easily verify that:
\begin{proposition}
  \label{prop:interval-norm}
  For all \PL subsumption queries $C\sqsubseteq D$, \splt{C}{D} is
  equivalent to $C$ and $\splt{C}{D}\sqsubseteq D$ is an interval-safe \PL
  subsumption query.
\end{proposition}

\noindent
In general, \splt{C}{D}may be exponentially larger than $C$, due to the
application of distributivity (e.g.\ this happens with the concepts
$C$ and $D$ in the proof of Theorem~\ref{thm:PL-is-NP-hard}). However,
as we have already pointed out, each simple
policy has at most one, functional concrete property
so no combinatorial explosion occurs during interval
normalization. Accordingly -- and more generally -- the following
proposition holds:

\begin{proposition}
  \label{prop:bounded-constraints}
  Let $C=C_1\sqcup\ldots\sqcup C_n$ be a \PL concept, and suppose that
  for all $i=1,\ldots,n$, the number of concrete properties occurring in
  $C_i$ is bounded by a constant $c$.  Then, for all concepts $D$, the
  size of \splt{C}{D} is $O(|C|\cdot|D|^c)$.\footnote{We denote the size of the encoding of an expression $E$ with $|E|$.}
\end{proposition}

\noindent
Note that $C$ as a whole may still contain an unbounded number of
interval constraints, as $n$ grows, because the bound $c$ applies only to
the individual disjuncts $C_i$.

The structural subsumption algorithm for \PL's subsumption queries accepts subsumptions whose left-hand side is further normalized with respect to the rewrite rules illustrated in Table~\ref{norm-rules}.  Such rules make contradictions explicit and merge functional properties. They clearly preserve equivalence, as stated in the next proposition:

\begin{table*}
  \centering
  \small
  \framebox{
    \begin{minipage}{.96\textwidth}
      \renewcommand{\arraystretch}{1.5}
      \begin{tabular}{rlp{14em}}
        1) & $\bot \sqcap D \leadsto \bot$
        \\
        2) & $\exists R.\bot \leadsto \bot$
        \\
        3) & $\exists f.[l,u] \leadsto \bot$ & if $l>u$
        \\
        4) & $(\exists R.D)\sqcap (\exists R.D') \sqcap D'' \leadsto {}$\\[-5pt]
        & \hspace*{9.5em} $\exists R.( D\sqcap D')\sqcap D''$
        & if $\func(R) \in \K$
        \\
        5) & $\exists f.[l_1,u_1] \sqcap \exists f. [l_2,u_2] \sqcap D \leadsto {}$\\[-5pt]
        & \hspace*{3em} $\exists f. [\max(l_1,l_2),\min(u_1,u_2)] \sqcap D$
        & if $\func(f) \in \K$
        \\
        6) & $\exists R.D\sqcap D' \leadsto \exists R.( D\sqcap A)\sqcap D'$
        & if $\range(R,A)  \in \K$, and neither $A$ nor $\bot$ are conjuncts of $D$
        \\
        7) & $A_1 \sqcap A_2\sqcap D \leadsto \bot$
        & if $A_1 \sqsubseteq^* A_1'$, $A_2 \sqsubseteq^* A_2'$, and $\disj(A_1',A_2')  \in \K$
      \end{tabular}
    \end{minipage}
  }
  \caption{Normalization rules w.r.t.\ \K. Intersections are treated as sets (the ordering of
    conjuncts and their repetitions are irrelevant). }
  \label{norm-rules}
  \vspace*{-1em}
\end{table*}

\begin{proposition}
  \label{prop:norm}
  If $C\leadsto C'$ then $\K\models C\equiv C'$.
\end{proposition}

\noindent
The proof is trivial and left to the reader.
It is easy to see that concepts can be normalized in polynomial time:

\begin{lemma}
  \label{lem:norm-cost}
  Each \PL concept $C$ can be normalized w.r.t.\ a given \PL knowledge
  base \K in time $O(|C|^2\cdot |\K|)$.
\end{lemma}

\begin{proof}
  We take this chance to illustrate an algorithm which is similar to
  the one actually used in the implementation of normalization.  First
  $C$ is parsed into a syntax tree $T$ (time $O(|C|)$) where each
  conjunction of $n$ concepts is modelled as a single node with $n$
  children. Then the tree is scanned in a depth-first fashion, looking
  for nodes labelled with an existential restriction in order to apply
  rule 4). For each such node $\nu$, if $R$ is the involved role and
  $\func(R)\in\K$, then the previous siblings of $\nu$ are searched
  looking for a node $\nu'$ with the same role $R$. If such a $\nu'$
  is found, then the child $C'$ of $\nu'$ is replaced with the intersection of $C'$ itself and the child of $\nu$,
  then $\nu$ is deleted. This operation (including the functionality
  test for $R$) takes time $O(|\K|+|C|)$ for each existential
  restriction. Thus, the exhaustive application of rule 4) needs time
  $O(|C|\cdot|\K|+|C|^2)$.  Rule 5) is dealt with similarly (but
  instead of merging children, the interval associated to $\nu$ is
  intersected with the interval associated to $\nu'$); the cost is the
  same.  None of the other rules adds any new existential
  restrictions, so rules 4) and 5) are not going to be applicable
  again in the rest of the algorithm.

  Next, rule 6) is applied by searching the tree $T$ for existential
  restrictions whose role $R$ occurs in an axiom $\range(R,A)\in
  \K$. For each of such nodes, $A$ is added to the children as a new
  conjunct (if necessary). The cost for each existential
  restriction is $O(|\K|+|C|)$ (where $|C|$ is the cost of verifying
  whether the existential restriction already contains  $A$ or $\bot$). So the exhaustive
  application of rule 6) is again $O(|C|\cdot|\K|+|C|^2)$. The
  remaining rules can remove a range $A$ only by substituting it with
  $\bot$, so rule 6) cannot be triggered again in the rest of the algorithm.

  Finally, the nodes of $T$ are visited in a depth-first fashion in
  order to apply rules 1), 2), 3), and 7). 

  Rule 7) is the most expensive. \K is regarded as a labelled
  classification graph, where each node is labelled with an atomic
  concept and with the disjointness axioms in which that concept
  occurs. The disjointness test between $A_1$ and $A_2$ in rule 7) can
  be implemented by a relatively standard linear-time reachability
  algorithm, that climbs the classification graph from $A_1$ and
  starts descending the classification whenever it finds a node
  labelled with $\disj(A_1',A_2')$, searching for $A_2$. In the worst
  case, this stage involves $O(|C|^2)$ searches (one for each
  pair $A_1,A_2$ in each conjunction), so its global cost is
  $O(|C|^2\cdot|\K|)$.

  Finally, note that rules 1--3 do not need to be iteratively
  applied. If $C$ contains an empty interval $[l,u]$ ($l>u$), or an
  occurrence of $\bot$, at any nesting level, then surely $C$ can be
  rewritten to $\bot$.  Therefore, it suffices to scan $C$ once,
  looking for empty intervals or $\bot$.

  Since the cost of rule~7 dominates the cost of the other rules,
  normalization can be computed in time $O(|C|^2\cdot|\K|)$.
\end{proof}

Normalized queries are passed over to a structural subsumption
algorithm, called \SSA (Algorithm~\ref{alg:ss}). It takes as inputs a \PL
knowledge base \K and an \emph{elementary} \PL subsumption
$C\sqsubseteq D$:

\begin{definition}[Elementary subsumptions]
  A \PL subsumption $C\sqsubseteq D$ is \emph{elementary} (w.r.t.\ a
  \PL knowledge base \K) if both $C$ and $D$ are simple, $C\sqsubseteq
  D$ is interval safe, and $C$ is normalized w.r.t.\ \K (i.e.\ none of the rules in
  Table~\ref{norm-rules} is applicable).
\end{definition}

\begin{algorithm}[h]
  \caption{$\SSA(\K,C\sqsubseteq D)$}
  \label{alg:ss}
  \small
  \KwIn{A \PL KB $\K$ and a \PL subsumption $C\sqsubseteq D$ that is  elementary w.r.t.\ \K}
  \KwOut{ \easytt{true}\, if $\K \models C\sqsubseteq D$, \quad\easytt{false}\, otherwise
    \\\vspace*{\medskipamount}
    \textbf{Note 1:} Below, we treat intersections like sets. For example, by $C=C'\sqcap C''$ we mean that either
    $C=C'$ or $C'$ is a conjunct of $C$ (possibly not the first one).
    \\\vspace*{\medskipamount}
    \textbf{Note 2:} $\sqsubseteq^*$ denotes the reflexive and
    transitive closure of $\{(A,B)\mid (A\sqsubseteq B)\in\K\}$.
  }
  \vspace*{\medskipamount}
  \Begin{
      \lIf{$C=\bot$}{ \Return{\easytt{true}}  }
      
      \lIf{$D=A$, $C=A'\sqcap C'$ and $A'\sqsubseteq^* A$}{ \Return{\easytt{true}}  }

      \lIf{$D=\exists f.[l,u]$  and $C=\exists f.[l',u']\sqcap C'$ and $l\leq l'$ and $u'\leq u$}
          { \Return{\easytt{true}}  }
          
      \lIf{$D=\exists R.D'$, $C=(\exists R.C')\sqcap C''$
           and $\SSA(\K,\ C'\sqsubseteq D')$}
          { \Return{\easytt{true}}  }
              
      \lIf{$D=D'\sqcap D''$, $\SSA(\K,C\sqsubseteq D')$,
           and $\SSA(\K,C\sqsubseteq D'')$}
          { \Return{\easytt{true}} }
                  
      \lElse{ \Return{\easytt{false}} }
  }
\end{algorithm}

\noindent
The full subsumption checking procedure (that applies to \emph{all}
\PL subsumptions) is called \emph{\PL Reasoner} (PLR for short). It is
summarized in Algorithm~\ref{alg:main}.

\begin{algorithm}[H]
  \caption{$\PLR(\K,C\sqsubseteq D)$}
  \label{alg:main}
  \small
  \KwIn{A \PL KB $\K$ and a \PL subsumption query $C\sqsubseteq D$}
  \KwOut{ \easytt{true}\, if $\K \models C\sqsubseteq D$, \quad\easytt{false}\, otherwise}
  \vspace*{\medskipamount}
  \Begin{
      {\bf let} $C'$ be the normalization of $C$ w.r.t.\ \K (with the
      rules in Table~\ref{norm-rules}) \;
      
      {\bf let} $C'' = \splt{C'}{D}$  \; 

      \tcp{assume that $C''=C_1\sqcup\ldots\sqcup C_m$ and $D=D_1\sqcup\ldots\sqcup D_n$}

      \tcp{check whether each $C_i$ is subsumed by some $D_j$}

      \For{$i=1,\dots,m$}{
        \For{$j=1,\dots,n$}{
          \lIf  {$\SSA(\K,C_i\sqsubseteq D_j)=\easytt{true}$}
                {skip to next $i$ in outer loop}
        }
        \Return{\easytt{false}}
      }
      \Return{\easytt{true}}
  }
\end{algorithm}

\PLR is correct and complete. We only state this result, whose proof
is sketched in \cite{DBLP:conf/ijcai/Bonatti18}, since we are going to
prove it in a more general form for an extended engine that supports
more expressive knowledge bases (Section~\ref{sec:oracles}).

\begin{theorem}
  \label{thm:PLR-correct-complete}
  For all \PL knowledge bases \K and all \PL subsumption queries $q$,
  \[
  \K \models q \mbox{\: iff \:} \PLR(\K,q) = \easytt{true} \,.
  \]
\end{theorem}

\noindent
With this result, we can prove that subsumption checking in \PL
becomes tractable if the number of interval constraints per simple
policy is bounded by a constant $c$ (recall that in SPECIAL's policies
$c=1$).  First we estimate the complexity of \PLR.

\begin{lemma}
  \label{lem:PLR-complexity}
  For all \PL knowledge bases \K and all \PL subsumption queries
  $C\sqsubseteq D$, $\PLR(\K,C\sqsubseteq D)$ can be computed in time
  $O(|C\sqsubseteq D|^{c+1}+ |C\sqsubseteq D|^2\cdot |\K|)$, where $c$
  is the maximum number of interval constraints occurring in a single
  simple concept of $C$.
\end{lemma}

\begin{proof}
  By Lemma~\ref{lem:norm-cost} and
  Proposition~\ref{prop:bounded-constraints}, respectively, the
  complexity of line~2 of \PLR is $O(|C|^2\cdot |\K|)$ and the
  complexity of line~3 is $O(|C|\cdot|D|^c) = O(|C\sqsubseteq
  D|^{c+1})$.  Now consider the complexity of the calls
  $\SSA(\K,C_i\sqsubseteq D_j)$ in line~6. Each of them, in the worst
  case, scans $C_i$ once for each subconcept of $D_j$, searching for a
  matching concept.  Matching may require to solve a reachability
  problem on the hierarchy $\sqsubseteq^*$, so the cost of each call
  is $O(|D_j|\cdot|C_i|\cdot|\K|)$. If we focus on the outer loop
  (lines 4--9) then clearly each subconcept of $D$ is matched against
  all disjuncts of $C$, in the worst case. Then the overall cost of
  the outer loop is $O(|D|\cdot|C|\cdot|\K|)$.  By relating these
  parameters to the size of the query, it follows that the cost of the
  outer loop is bounded by $O(|C\sqsubseteq D|^2\cdot |\K|)$.  This
  dominates the cost of line~2. So we conclude that the overall time
  needed by \PLR in the worst case is $O(|C\sqsubseteq D|^{c+1}+
  |C\sqsubseteq D|^2\cdot |\K|)$.
\end{proof}

\noindent
Tractability imediately follows from
Theorem~\ref{thm:PLR-correct-complete} and
Lemma~\ref{lem:PLR-complexity}:

\begin{theorem}
  \label{thm:PLR-tractability}
  Let $c$ be an integer, and $\Q_c$ be the set of all \PL subsumptions
  $C_1\sqcup\ldots\sqcup C_n\sqsubseteq D$ such that each $C_i$
  contains at most $c$ interval constraints $(i=1,\ldots,n)$.  Then
  deciding whether a query in $\Q_c$ is entailed by a \PL knowledge
  base \K is in \easysf{P}.
\end{theorem}

\noindent
We conclude this section by completing the characterization of the
complexity of unrestricted \PL subsumptions. The following result, together with Theorem~\ref{thm:PL-is-NP-hard}, proves that \PL subsumption is co\NP-complete.

\begin{theorem}
  \label{thm:PL-is-in-NP}
  Deciding whether $\K\models C\sqsubseteq D$, where \K is a \PL
  knowledge base and $C,D$ are (simple or full) \PL concepts, is
  in co\NP.
\end{theorem}

\begin{proof}
  We prove the theorem by showing that the complement of subsumption
  is in \NP. For this purpose, given a query $C\sqsubseteq D$, it
  suffices to choose nondeterministically one of the disjuncts $C_i$
  in the left hand side of the query, and replace each constraint
  $\exists f.[\ell,u]$ occurring in $C_i$ with a nondeterministically
  chosen disjunct from (\ref{interval-splitting}). Call $C_i'$ the
  resulting concept and note that it is one of the disjuncts in
  $\splt{C}{D}$.  Therefore, $\K\not\models C\sqsubseteq D$ iff
  $\K\not\models \splt{C}{D}\sqsubseteq D$ iff, for some
  nondeterministic choice of $C'_i$, $\K\not\models C_i'\sqsubseteq D$.  Note
  that $C_i'\sqsubseteq D$ is interval-safe by construction.  Then
  this subsumption test can be evaluated in deterministic polynomial
  time by first normalizing $C_i'$ w.r.t.\ \K and then applying \SSA,
  that is complete for elementary queries
  \cite[Theorem~2]{DBLP:conf/ijcai/Bonatti18}. It follows immediately
  that the complement of \PL subsumption can be decided in
  nondeterministic polynomial time, hence its membership in \NP.
\end{proof}

\section{Supporting General Vocabularies}
\label{sec:oracles}

SPECIAL has founded the ``\emph{Data Privacy Vocabularies and Controls
  Community Group}'' (DPVCG),\footnote{\url{www.w3.org/community/dpvcg/}} a
W3C group aimed at developing privacy-related vocabularies.  The
purpose of this initiative is developing ontologies for the main
properties of usage policies and related GDPR concepts, with the
contribution of a group of stakeholders that spans beyond SPECIAL's
consortium.  This group aims at developing upper ontologies, that can
be later extended to meet the needs of specific application domains.

We intend to put as few constraints as possible on the development of
such standardized vocabularies, since it is difficult to predict the
expressiveness needs that may arise in their modeling --
especially because standards usually change to include new
application domains and follow the evolution of the old ones.
\PL knowledge bases are too simple to address this requirement.
We already have evidence that it is useful to have roles whose domain
is a vocabulary term, such as the accuracy of locations\hide{ and the degree
of anonimity} (cf.~Section~\ref{sec:enc}); so, in perspective, we
should expect the ontologies that define privacy-related vocabularies
to include at least existential restrictions (that cannot be used in
\PL knowledge bases, but are supported -- say -- by the tractable
profiles of OWL2).  It is hard to tell which other constructs will
turn out to be useful.

For the above reasons, we are going to show how to integrate \PL and its
specialized reasoner with a wide range of ontologies, expressed with
description logics that can be significantly more expressive than \PL.

Our strategy consists in treating such ontologies -- hereafter called
\emph{external ontologies} -- as \emph{oracles}.  Roughly speaking,
whenever \SSA needs to check a subsumption between two terms defined
in the external ontologies, the subsumption query is submitted to the
oracle.  In the easiest case, the oracle can be implemented simply as
a visit to the classification graph of the vocabularies.  Of course
this method, called \emph{import by query} (IBQ), is not always
complete
\cite{DBLP:conf/ijcai/GrauMK09,DBLP:journals/jair/GrauM12}. In this
section, we provide sufficient conditions for
completeness.

More formally, let \K and \O be two given knowledge bases. The former
will be called the \emph{main KB}, and may use terms that are
axiomatized in \O, that plays the role of the external ontology. For
example, in SPECIAL's policy modeling scenario, \K defines policy
attributes -- by specifying their ranges and functionality properties
-- while \O defines the privacy-related vocabularies that provide the
fillers for policy attributes.  Therefore, in SPECIAL's framework, \K
is a \PL knowledge base, while \O could be formulated with a more
expressive DL.
The reasoning task of interest in such scenarios is deciding, for a
given subsumption query $q=(C\sqsubseteq D)$, whether $\K \cup \O
\models q$.  Both $C$ and $D$ are \PL concepts that usually
contain occurrences of concept names defined in \O.

SPECIAL's application scenarios make it possible to adopt a
simplifying assumption that makes oracle reasoning technically simpler
\cite{DBLP:conf/ijcai/GrauMK09,DBLP:journals/jair/GrauM12}, namely, we
assume that neither \K nor the query $q$ share any roles with \O.
This naturally happens in SPECIAL precisely because the roles used in
the main KB identify the sections that constitute a policy (e.g.\ data
categories, purpose, processing, storage, recipients), while the roles
defined in \O model the \emph{contents} of those sections,
e.g.\ anonymization parameters, relationships between recipients (like
ownership, employment relations), relationships between storage
locations (e.g.\ part-of relations), and the like.  This layered
structure does not require arbitrary alternations of roles coming from
the main KB and from the external ontologies. As a consequence, shared roles can be eliminated as follows:

\begin{remark}\rm
  \label{rem:shared-roles}
  The roles occurring in \O can be used in policies as syntactic
  sugar. Any concept of the form $\exists R.C$ such that $R$ occurs in
  \O, can be eliminated from a query by replacing it with a fresh atom
  $A$ defined with an axiom $A\equiv \exists R.C$ in the external
  ontology, under the mild assumption that \O's language supports such
  equivalences.  The result satisfies the requirement that $q$ and \O
  should share no roles. The actual restriction -- that is not removed
  by the above transformation -- is that in every concept $\exists
  R.(\ldots \exists S \ldots)$, if $R\in\sig{\O}$ then also $S$ must
  be in $\sig{\O}$.  This is not an issue, in SPECIAL's application
  scenarios: as it has just been pointed out, role alternations where
  $R\in\sig{\O}$ and $S\not\in\sig{\O}$ are not required.
\end{remark}

\subsection{On the Completeness of IBQ Reasoning}

The IBQ framework was introduced to reason with a partly hidden
ontology \O.  For our purposes, IBQ is interesting because instead of
reasoning on $\K\cup \O$ as a whole, each partition can be processed
with a different reasoner (so, in particular, policies can be compared
with a very efficient algorithm similar to \SSA).  The reasoner for
$\K$ may query the reasoner for $\O$ as an oracle, using a query
language $\L_\O$ consisting of all the subsumptions
\begin{equation}
\label{oq}
A_1 \sqcap \ldots \sqcap A_m \sqsubseteq A_{m+1} \sqcup \ldots
\sqcup A_n 
\end{equation}
such that $A_1, \ldots, A_n$ are concept names in $\sig{\O}\cap\NC$.  If $n=m$, then we stipulate that the right-hand side of the inclusion is $\bot$.
We will denote with $\pos(\O)$ all the queries to \O that have a positive answer, that is:
\begin{equation*}
\pos(\O) = \{q \in \L_\O  \mid \O\models q\}\,.
\end{equation*}

\begin{remark}\rm
  \label{rem:qo2ct}
  Each subsumption of the form (\ref{oq}) is equivalent to a concept
  (in)\hspace{1pt}consistency check of the form:
  \begin{equation}
    \label{ct}
    A_1 \sqcap \ldots \sqcap A_m \sqcap \neg A_{m+1} \sqcap \ldots
    \sqcap \neg A_n \sqsubseteq \bot \,.
  \end{equation}
  By \cite[Theorem~2]{DBLP:conf/ijcai/GrauMK09}, consistency checks constitute a
  fully general oracle query language, under the assumption that \K and the query $q$ share no
  roles with \O.
\end{remark}

The problem instances we are interested in are formally defined by the next definition.

\begin{definition}[\PL subsumption instances with oracles, \PLSO]
  \label{def:ibq}
  A \emph{\PL subsumption instance with oracle} is a triple
  \tup{\K,\O,q} where $\K$ is a \PL knowledge base (the \emph{main
    knowledge base}), $\O$ is a Horn-\SRIQ knowledge base (the
  \emph{oracle}), and $q$ is a \PL subsumption query such that
  $(\sig{\K}\cup\sig{q})\cap\sig{\O}\subseteq \NC$.  The set of all
  \PL subsumption instances with oracle will be denoted by \PLSO.
\end{definition}

\noindent
The restrictions on \K, \O and $q$ will be discussed in
Section~\ref{sec:limitations}. We anticipate only two
observations. First, the restriction on the signatures is aimed at
keeping the roles of \O separated from those of \K and $q$, as
discussed in the previous section. The second obervation is that the
important properties of \O are the absence of nominals and the
following convexity property:
\begin{definition}[Convexity w.r.t.\ $\L_\O$]
  \label{def:convex-LO}
  A knowledge base \O is \emph{convex w.r.t.\ $\L_\O$} if for all
  subsumptions $q$ in $\L_\O$ of the form $(\ref{oq})$, $q\in\pos(\O)$
  iff there exists $i \in [m+1,n]$ such that $(A_1 \sqcap
  \ldots \sqcap A_m \sqsubseteq A_i) \in \pos(\O)$. A description
  logic is convex w.r.t.\ $\L_\O$ if all of its knowledge bases
  are.
\end{definition}

\noindent
Accordingly, we require \O to be in Horn-\SRIQ because, to the best of
our knowledge, this is the most expressive nominal-free and convex
description logic considered so far in the literature.


The next lemma rephrases the original IBQ completeness result
\cite[Lemma~1]{DBLP:conf/ijcai/GrauMK09} in our notation.  Our statement
relaxes the requirements on \O by assuming only that it enjoys the
disjoint model union property (originally it had to be in
\easycal{SRIQ}).  The proof, however, remains essentially the same.

\begin{lemma}
  \label{lem:CT-completeness-altrui}
  Let \K and \O be knowledge bases and $\alpha$ a GCI, such that
  \begin{enumerate}
  \item \K and $\alpha$ are in $\SROIQ(\easycal{D})$ without $U$,
    where \easycal{D} is the concrete domain of integer intervals;
  \item The terminological part of \O enjoys the disjoint model union property;
  \item The terminological part \T of \K is local w.r.t. $\sig{\O}$;
  \item $(\sig{\K} \cup \sig{\alpha}) \cap \sig{\O}\subseteq \NC$.
  \end{enumerate}
  Then $\K\cup\O\models \alpha$ iff $\K\cup\pos(\O)\models \alpha$\,.
\end{lemma}

\begin{proof}
We have to prove that under the above hypotheses
$\K\cup\O\models\alpha$ iff $\K \cup \pos(\O) \models \alpha$. The
right-to-left direction is trivial since by definition $\O\models
\pos(\O)$. For the other direction, by contraposition, assume that
$\K\cup\pos(\O)\not\models \alpha$. We shall find a model \N of
$\K\cup\O$ such that $\N\not\models\alpha$.  Since $\alpha$ is of the form $C\sqsubseteq D$, this means that for some $\bar d\in\Delta^\N$, $\bar d \in
(C\sqcap\neg D)^\N$.  The construction is similar to that used in \cite[Lemma~1]{DBLP:conf/ijcai/GrauMK09}.

By assumption, $\K\cup\pos(\O)$ has a model \I such that
$\I\not\models\alpha$, that is, there exists $\bar d\in\Delta^\I$ such
that $\bar d \in (C\sqcap\neg D)^\I$. Now we extend the interpretation
\I over $\sig{\K}\cup \sig{\alpha}$ to a model \N of $\K\cup\O$.

We need some auxiliary notation: for each $d\in\Delta^\I$, let
$\lit(d,\I)$ denote the set of all the literals $L$ in the language of \O satisfied by $d$,
that is,
\[
   \lit(d,\I) = \{ L \mid (\I,d) \models L \mbox{ and either } L=A
   \mbox{ or } L=\neg A,\mbox{ where } A\in\NC\cap\sig{\O} \} \,.
\]
Since $\I \models \pos(\O)$, it follows that for all $d\in \Delta^\I$,
$\O \not\models \bigsqcap \lit(d,\I)\sqsubseteq \bot$
(cf.\ Remark~\ref{rem:qo2ct}). Then, for all $d\in \Delta^\I$, there
exists a pointed interpretation $(\J_d,d)$ of $\sig{\O}$ such that
$\J_d\models \O$ and $\lit(d,\J_d)=\lit(d,\I)$. We may assume without
loss of generality that $\Delta^{\J_d}\cap \Delta^\I=\{d\}$ and that
$\Delta^{\J_d}\cap \Delta^{\J_{d'}}=\emptyset$ if $d\neq d'$.

Let \J be any of the above $\J_d$ and $\U=\biguplus^\J \{\J_d \mid d\in\Delta^\I\}$.  By hypothesis~2 and Proposition~\ref{prop:x-union}, \U is a model of \O. Moreover, by hypothesis~3, \U can be extended to a model \M of \T, by setting $X^\M=\emptyset$ for all predicates $X\in(\sig{\K} \cup \sig{\alpha})\setminus\sig{\O}$.

Finally, let \N be the interpretation such that:
\begin{eqnarray*}
  \Delta^\N &=& \Delta^\M \quad\quad\mbox{ (note that $\Delta^\I\subseteq \Delta^\M$)}\\
  X^\N &=& \left\{
  \begin{array}{lp{18em}}
    X^\I & for all symbols $X\in (\sig{\K} \cup \sig{\alpha}) \setminus\sig{\O}$\\
    X^\M & for all symbols $X\in\sig{\O}$ \,.
  \end{array}
  \right.
\end{eqnarray*}

The next part of the proof proceeds exactly as in \cite[Lemma~1]{DBLP:conf/ijcai/GrauMK09}, in order to show that $\N\models\K\cup\O$. Note that by definition \M and \N have the same domain and agree on the symbols in $\sig{\O}$, therefore \N is a model of \O because \M is. So one is only left to prove that $\N \models \K$.  For this purpose,  first it is proved that
\begin{quotation}
  $(\star)$~ for all $C$ in the closure\footnote{Recall that the
    closure of a set of DL expressions \easycal{S} is the set of all
    (sub)\,concepts occurring in \easycal{S}.} of \K and $\alpha$,
  $C^\N=C^\I\cup (C^\M\setminus \Delta^\I)$.
\end{quotation}
The proof of ($\star$) makes use of hypotheses~1 and 4. Then, using
$(\star)$ and the fact that \M is a model of \T, it can be shown that
\M is a model of \K.
Almost all details of the proof of ($\star$) and $\M\models\K$ can be
found in \cite{DBLP:conf/ijcai/GrauMK09}.  Here we only have to add
the details for ($\star$) concerning interval constraints (that are
not considered in \cite{DBLP:conf/ijcai/GrauMK09}).
Let $C=\exists f.[l,u]$. By hypothesis~4, $f\in(\sig{\K} \cup \sig{\alpha}) \setminus\sig{\O}$.
Then, by definition of \N and \M, $f^\N=f^\I$ and
$f^\M=\emptyset$. Consequently, $C^\N=C^\I$ and $C^\M=\emptyset$, so
($\star$) obviously holds.

For our formulation of this theorem, we only have to add the observation that $(\star)$ implies also that $\bar d\in (C\sqcap \neg D)^\I \subseteq (C\sqcap \neg D)^\N$, therefore $\N\not\models\alpha$.
\end{proof}

\noindent
Using the above lemma, we prove a variant of IBQ completeness for
\PLSO. The locality requirement of
Lemma~\ref{lem:CT-completeness-altrui} is removed by shifting axioms
from \K to \O.

\begin{theorem}
  \label{thm:compl-w-shifting}
  For all problem instances $\pi = \tup{\K,\O,q}\in\PLSO$, let
  \[
  \K^-_\O = \{ \alpha \in\K \mid  \alpha = \range(R,A) \mbox{ or } \alpha = \func(R) \,\}
  \]
  and let $\O^+_\K =\O\cup(\K \setminus \K^-_\O)$. Then
  \[
  \K\cup\O\models q \mbox{\: iff \:} \K^-_\O \cup \pos(\O^+_\K) \models q \,.
  \]
\end{theorem}

\begin{proof}
  Since $\K\cup\O = \K_\O^- \cup \O_\K^+$, it suffices to show that
  \begin{equation*}
    \label{shifting:1}
    \K_\O^-\cup\O_\K^+\models q \mbox{\: iff \:} \K^-_\O \cup
    \pos(\O^+_\K) \models q\,.
  \end{equation*}
  This equivalence can be proved with
  Lemma~\ref{lem:CT-completeness-altrui}; it suffices to show that
  $\K_\O^-$, $\O_\K^+$ and $q$ satisfy the hypotheses of the lemma.
  First note that $\K_\O^-$ is a \PL knowledge base and $\O_\K^+$ is a
  Horn-\SRIQ knowledge base (because, by definition of \PLSO, \K is in
  \PL and \O in Horn-\SRIQ, and the axioms shifted from \L to
  $\O_\K^+$ can be expressed in Horn-\SRIQ, too).  Both \PL and
  Horn-\SRIQ are fragments of $\cal SROIQ(D)$ without $U$, therefore
  hypothesis~1 is satisfied by $\K_\O^-$ and $\O_\K^+$.  Moreover,
  both \PL and Horn-\SRIQ enjoy
  the disjoint model union property, therefore hypothesis~2 is
  satisfied. Next recall that
  $(\sig{\K}\cup\sig{q})\cap\sig{\O}\subseteq \NC$ holds, by definition of
  \PLSO. Since the axioms $\alpha\in\K\setminus\K_\O^-$ (transferred
  from \K to $\O_\K^+$) contain no roles (they are of the form
  $A\sqsubseteq B$ or $\disj(A,B)$), it follows that
  \begin{equation*}
    (\sig{\K_\O^-}\cup\sig{q})\cap\sig{\O_\K^+}\subseteq \NC \,,
  \end{equation*}
  that is, hypothesis~4 holds.  A second consequence of this inclusion
  is that $\K_\O^-$ contains only axioms of the form $\range(R,A)$ and
  $\func(A)$ such that $R\not\in\sig{\O_\K^+}$.  They are trivially
  satisfied by any interpretation \I such that
  $R^\I=\emptyset$. Therefore $\K_\O^-$ is local
  w.r.t.\ $\sig{\O_\K^+}$ and hypothesis~3 is satisfied.
\end{proof}

\subsection{Extending \PL's Reasoner with IBQ Capabilities}

The integration of \PLR reasoner with external oracles
relies on the axiom shifting applied in
Theorem~\ref{thm:compl-w-shifting}.  Accordingly, in the following,
let $\K_\O^-$ and $\O_\K^+$ be defined as in
Theorem~\ref{thm:compl-w-shifting}.

The next step after axiom shifting consists in replacing the relation
$\sqsubseteq^*$ used by the normalization rules and \SSA with suitable
queries to the oracle. This change concerns the normalization rules
(Table~\ref{norm-rules}) and \SSA.  The new set of rules is
illustrated in Table~\ref{norm-rules-O}.
\begin{table}[h]
	\centering
	\small
	\framebox{
                \renewcommand{\arraystretch}{1.3}
		\begin{minipage}{.96\textwidth}			
			\begin{tabular}{rlp{12em}}
				1) & $\bot \sqcap D \leadsto \bot$
				\\
				2) & $\exists R.\bot \leadsto \bot$
				\\
				3) & $\exists f.[l,u] \leadsto \bot$ & if $l>u$
				\\
				4) & $(\exists R.D)\sqcap (\exists R.D') \sqcap D'' \leadsto \exists R.( D\sqcap D')\sqcap D''$ 
				& if $\func(R) \in \K_\O^-$
				\\
				5) & $\lefteqn{\exists f.[l_1,u_1] \sqcap \exists f.[l_2,u_2] \sqcap D \leadsto \exists f.[\max(l_1,l_2),\min(u_1,u_2)]\sqcap D }$  
				\\
				6) & $\exists R.D\sqcap D' \leadsto \exists R.( D\sqcap A)\sqcap D'$
				& if $\range(R,A)  \in \K_\O^-$ and~~~\vspace*{-1pt}\break $A,\bot$ are not conjuncts of $D$
				\\
				7) & $A_1 \sqcap\ldots\sqcap A_n \sqcap D \leadsto \bot$
				& if $\O_\K^+\models A_1 \sqcap\ldots\sqcap A_n \sqsubseteq \bot$
			\end{tabular}
		\end{minipage}
	}
	\caption{Normalization rules for \SSO. Conjunctions are
          treated as sets (i.e.\ the ordering of conjuncts is
          irrelevant, and duplicates are removed).}
	\label{norm-rules-O}
\end{table}

\noindent
Hereafter, $\leadsto$ denotes the rewriting relation according to Table~\ref{norm-rules-O}.
Clearly, the new rules preserve the meaning of concepts, in the following sense:

\begin{proposition}
  \label{prop:norm-O}
  If $C\leadsto C'$ then $\K_\O^-\cup\pos(\O_\K^+) \models C\equiv C'$.
\end{proposition}

\noindent
If none of the new rules is applicable to a concept $C$, then we say
that $C$ is \emph{normalized w.r.t.\ \K and \O}. The notion of elementary
inclusion is modified accordingly, by requiring normalization
w.r.t.\ both \K and \O.

\begin{definition}
  A \PL subsumption $C\sqsubseteq D$ is \emph{elementary w.r.t.\ \K
    and \O} if both $C$ and $D$ are simple, $C\sqsubseteq D$ is
  interval safe, and $C$ is normalized w.r.t.\ \K and \O (using the
  rules in Table~\ref{norm-rules-O}).
\end{definition}

\noindent
Then \SSA is integrated with the ``oracle'' \O by replacing its line~3
as in the following algorithm \SSO.
In the following, we call a subconcept ``\emph{top level}'' if it does
not occur in the scope of any existential restriction.
\medskip

\begin{algorithm}[H]
  \caption{$\SSO(C\sqsubseteq D)$}
  \label{alg:sso}
  \small
  \KwIn{An ontology \O and a \PL subsumption $C\sqsubseteq D$ \hide{that is  elementary w.r.t.\ \K and \O}}
  \KwOut{ \easytt{true}\, if $\O \models C\sqsubseteq D$, \quad\easytt{false}\,
    otherwise, \quad under suitable restrictions }
  \vspace*{\medskipamount}
  \Begin{
      \lIf{$C=\bot$}{ \Return{\easytt{true}}  }
      
      \lIf{$D=A$ and $(A_1\sqcap\ldots\sqcap A_n \sqsubseteq A) \in \pos(\O)$, 
        where $A_1,\ldots,A_n$ are the top-level concept names in $C$}{
        \Return{\easytt{true}} }
      
      \lIf{$D=\exists f.[l,u]$  and $C=\exists f.[l',u']\sqcap C'$ and $l\leq l'$ and $u'\leq u$}
          { \Return{\easytt{true}}  }
          
      \lIf{$D=\exists R.D'$, $C=(\exists R.C')\sqcap C''$ and $\SSO(C'\sqsubseteq D')$}
          { \Return{\easytt{true}}  }
              
      \lIf{$D=D'\sqcap D''$, $\SSO(C\sqsubseteq D')$, and $\SSO(C\sqsubseteq D'')$}
                  { \Return{\easytt{true}} }
                  
      \lElse{ \Return{\easytt{false}} }
  }
\end{algorithm}
\medskip

Finally, the reasoner for general \PL subsumptions with oracles can be
defined as follows:
\medskip

\begin{algorithm}[H]
  \caption{$\PLRO(\K,C\sqsubseteq D)$}
  \label{alg:main-O}
  \small
  \KwIn{\K and $C\sqsubseteq D$ such that $\pi=\tup{\K,\O,C\sqsubseteq D}\in\PLSO$}
  \KwOut{ \easytt{true}\, if $\K \cup \O \models C\sqsubseteq D$, \quad\easytt{false}\, otherwise}
  \vspace*{\medskipamount}
  \Begin{
      {\bf construct}  $\K_\O^-$ and $\O_\K^+$ as defined in Theorem~\ref{thm:compl-w-shifting} \;

      {\bf let} $C'$ be the normalization of $C$ w.r.t.\ \K and \O (with the
      rules in Table~\ref{norm-rules-O}) \;
      
      {\bf let} $C'' = \splt{C'}{D}$  \; 

      \tcp{assume that $C''=C_1\sqcup\ldots\sqcup C_m$ and $D=D_1\sqcup\ldots\sqcup D_n$}

      \tcp{check whether each $C_i$ is subsumed by some $D_j$}

      \For{$i=1,\dots,m$}{
        \For{$j=1,\dots,n$}{
          \lIf  {$\SSA^{\O_\K^+}(C_i\sqsubseteq D_j)=\easytt{true}$}
                {skip to next $i$ in outer loop}
        }
        \Return{\easytt{false}}
      }
      \Return{\easytt{true}}
  }
\end{algorithm}
\medskip

\noindent
The rest of this section is devoted to proving the soundness and
completeness of \PLRO.
We will need a set of canonical counterexamples to invalid subsumptions.

\begin{definition}
  \label{def:C-model}
  Let $C\neq\bot$ be a simple \PL concept normalized w.r.t.\ \K and \O.
  A \emph{canonical model} of $C$ (w.r.t.\ \K and \O) is a pointed interpretation
  $(\I,d)$ defined as follows, by recursion on the number of existential restrictions. 
  Hereafter we call a subconcept of $C$ ``top level'' if it does not occur within the scope of\,
  $\exists$\,.
  \begin{enumerate}
    \renewcommand{\theenumi}{\alph{enumi}}
  \item If  $C=\big(\bigsqcap_{i=1}^nA_i\big) \sqcap \big(  \bigsqcap_{j=1}^t \exists f_j.[l_j,u_j]\big) $
    (i.e.~$C$ has no existential restrictions),
    then  let
    $\I=\tup{\{d\},\cdot^\I}$ where
    \begin{itemize}
    \item $A^\I=\{d\}$ if  $\big(\bigsqcap_{i=1}^nA_i \sqsubseteq A\big) \in \pos(\O_\K^+)$\,;
    \item $f^\I = \{(d,u_j) \mid j=1,\ldots t \}$\,;
    \item all the other predicates are empty.
    \end{itemize}
    
  \item If the top-level existential restrictions of $C$ are
    $\exists R_i.D_i$ ($i=1,\ldots,m$), then for each
    $i=1,\ldots,m$, let $(\I_i,d_i)$ be a canonical model of $D_i$.
    Assume w.l.o.g.\ that all such models are mutually disjoint and
    do not contain $d$. Define an auxiliary interpretation \J as
    follows:
    \begin{itemize}
    \item $\Delta^\J=\{d,d_1,\ldots,d_m\}$;
    \item $A^\I=\{d\}$ if $\big(\bigsqcap_{i=1}^nA_i \sqsubseteq
      A\big) \in \pos(\O_\K^+)$\,, where $A_1,\ldots,A_n$ are the
      top-level concept names in $C$; all other concept names are
      empty;
    \item  $f^\J = \{(d,u) \mid \exists f.[l,u] \mbox{ is a top-level constraint of $C$}\,\}$\,;
    \item  $R_i^\J = \{(d,d_i) \mid i=1,\ldots,m \}$\,.
    \end{itemize}
    Finally let \I be the union of \J and all $\I_i$, that is
    \begin{eqnarray*}
      \Delta^\I &=& \Delta^\J \cup \mbox{$\bigcup_i \Delta^{\I_i}$}
      \\
      A^\I &=& A^\J \cup \mbox{$\bigcup_i A^{\I_i}$} \quad (A\in\NC)
      \\
      R^\I &=& R^\J \cup \mbox{$\bigcup_i R^{\I_i}$} \quad (R\in\NR\cup\NF) \,.
    \end{eqnarray*}
    The canonical model is $(\I,d)$.
  \end{enumerate}
\end{definition}
Note that each $C$ has a unique canonical model up to isomorphism.
The canonical model satisfies $\K^-_\O$, $\O_\K^+$, and $C$:

\begin{lemma}
  \label{lem:mod-of-norm-C}
  If $C$ is a simple \PL concept normalized w.r.t.\ \K and \O, and
  $C\neq\bot$, then each canonical model $(\I,d)$ of $C$ enjoys the
  following properties:
  \begin{enumerate}
    \renewcommand{\theenumi}{\alph{enumi}}
  \item $\I \models \K_\O^-\cup \pos(\O_\K^+)$;
  \item $(\I,d) \models C$.
  \end{enumerate}
\end{lemma}

\begin{proof}
  By induction on the maximum nesting level $\ell$ of $C$'s existential restrictions.
  
  If $\ell=0$ (i.e.\ there are no existential restrictions) then obviously
  $(\I,d) \models C$ by construction
  (cf.\ Def.~\ref{def:C-model}.a). The entailment $\I\models \K_\O^-$
  holds because $\K_\O^-$ contains only range and functionality
  axioms, that are trivially satisfied since all roles are empty in
  \I.  In order to prove the base case we are only left to show that
  $\I\models \pos(\O_\K^+)$.  Suppose not, i.e.\ there exists an
  inclusion $B_1\sqcap\ldots\sqcap B_m\sqsubseteq A$ in
  $\pos(\O_\K^+)$ such that $d\in(B_1\sqcap\ldots\sqcap B_m)^\I$ but
  $d\not\in A^\I$ (where $d$ is the only member of $\Delta^\I$).  By
  construction of \I, $d\in B_j^\I$ only if $\pos(\O_\K^+)$ contains
  $\bigsqcap_{i=1}^nA_i \sqsubseteq B_j$ for all $j=1,\ldots,m$, where
  the $A_i$ are the top-level concept names in $C$. These inclusions,
  together with $B_1\sqcap\ldots\sqcap B_m\sqsubseteq A$, imply by
  simple inferences that $\bigsqcap_{i=1}^nA_i \sqsubseteq A$ must be
  in $\pos(\O_\K^+)$, too. But then $A^\I$ should contain $\{d\}$ by
  definition (a contradiction). This completes the proof of the base
  case.
    
  Now suppose that $\ell>0$. By induction hypothesis (I.H), we have
  that all the submodels $(\I_i,d_i)$ used in Def.~\ref{def:C-model}.b
  satisfy $D_i$.  Then it is immediate to see that $(\I,d)\models C$
  by construction.  We are only left to prove that \I satisfies all
  axioms $\alpha$ in $\K_\O^-\cup \pos(\O_\K^+)$.
  
  If $\alpha=\func(R)$, then rewrite rules 4) and 5) make sure
  that $C$ contains at most one existential restriction for
  $R$, so \J satisfies $\alpha$. Since all {$\I_i$}
  satisfy $\alpha$ by I.H., \I satisfies $\alpha$, too.
  
  If $\alpha=\range(R,A)$, then rule 6) makes sure that
  for each top-level concept of the form $\exists R.D_i$ in $C$, $D_i\equiv D'_i\sqcap
  A$. Then, by I.H., $(I_i,d_i)\models A$ and, consequently, $\alpha$ is satisfied by \I.
  
  Finally, if $\alpha$ is an inclusion in $\pos(\O_\K^+)$, then $d$
  satisfies it by the same argument used in the base case, while the
  other individuals in $\Delta^\I$ satisfy $\alpha$ by I.H.
\end{proof}

\noindent
Another key property of the canonical models of $C$ is that they
characterize \emph{all} the valid elementary subsumptions whose
left-hand side is $C$:

\begin{lemma}
  \label{lem:counterex}
  If $C\sqsubseteq D$ is elementary w.r.t.\ \K and \O, $C\neq\bot$, and
  $(\I,d)$ is a canonical model of $C$, then $$\K_\O^-\cup \pos(\O_\K^+) \models
  C\sqsubseteq D \mbox{\: iff \:} (\I,d)\models D \,.$$
\end{lemma}

\begin{proof}
  (Only If part) Assume that $\K_\O^-\cup \pos(\O_\K^+) \models C\sqsubseteq D$. By
  Lemma~\ref{lem:mod-of-norm-C}.a, we have $\I\models \K_\O^-\cup \pos(\O_\K^+)$, so by assumption $C^\I\subseteq
  D^\I$. Moreover, by Lemma~\ref{lem:mod-of-norm-C}.b,
  $d\in C^\I\subseteq D^\I$. Therefore $(\I,d)\models D$.
  
  (If part) Assume that $(\I,d)\models D$. We are going to prove
  that $\K_\O^-\cup \pos(\O_\K^+) \models C\sqsubseteq D$ by structural induction on $D$.
  
  If $D=A$ (a concept name), then $d\in A^\I$ by assumption. Then, by
  construction of \I, there must be an inclusion $\big(
  \bigsqcap_{i=1}^nA_i \sqsubseteq A\big)\in\pos(\O_\K^+)$, where
  $A_1,\ldots,A_n$ are the top-level concept names of $C$.  This
  implies that both $\models C\sqsubseteq \bigsqcap_{i=1}^nA_i$ and
  $\pos(\O_\K^+) \models \bigsqcap_{i=1}^nA_i \sqsubseteq A$ hold,
  hence $\K_\O^-\cup \pos(\O_\K^+) \models C\sqsubseteq D$.
  
  If $D=D_1\sqcap D_2$, then $(\I,d) \models D_i$ ($i=1,2$),
  therefore, by induction hypothesis, $\K_\O^-\cup \pos(\O_\K^+)
  \models C\sqsubseteq D_i$ ($i=1,2$), hence $\K_\O^-\cup
  \pos(\O_\K^+) \models C\sqsubseteq D$.
  
  If $D=\exists R.D_1$, then for some $d_i\in\Delta^\I$, $(d,d_i)\in
  R^\I$ and $(\I_i,d_i)\models D_1$, where $(\I_i,d_i)$ (by
  construction of \I) is the canonical model of a concept $C_1$
  occurring in a top-level restriction $\exists R.C_1$ of $C$.  It
  follows that $\models C\sqsubseteq \exists R.C_1$ and (by induction
  hypothesis) $\K_\O^-\cup \pos(\O_\K^+) \models C_1\sqsubseteq D_1$,
  hence $\K_\O^-\cup \pos(\O_\K^+) \models C\sqsubseteq D$.
  
  If $D=\exists f.[\ell,u]$, then for some $u'\in[\ell,u]$, $(d,u')\in
  f^\I$.  By construction of \I, $C$ must contain a top-level
  constraint $\exists f.[\ell',u']$, so by interval safety (that is implied by the assumption that
  $C\sqsubseteq D$ is elementary), $[\ell',u']\subseteq[\ell,u]$.
  Then $\models C\sqsubseteq D$.
\end{proof}

\noindent
Moreover, by means of canonical models, one can prove that interval safety makes
the non-convex logic \PL behave like a convex logic.

\begin{lemma}
  \label{lem:convexity-PL}
  For all interval-safe \PL subsumption queries $\sigma =
  \big(C_1\sqcup\ldots\sqcup C_m \sqsubseteq D_1 \sqcup\ldots\sqcup
  D_n \big)$ such that each $C_i$ is normalized w.r.t.\ \K and \O, the
  entailment $\K_\O^- \cup \pos(\O_\K^+) \models \sigma$ holds iff for
  all $i\in[1,m]$ there exists $j\in[1,n]$ such that $\K_\O^- \cup \pos(\O_\K^+) \models
  C_i\sqsubseteq D_j$.
\end{lemma}

\begin{proof}
  Let \KB abbreviate $\K_\O^- \cup \pos(\O_\K^+)$.
  By simple logical inferences, these two facts hold: (i)~$\KB \models
  \sigma$ iff $\KB\models C_i\sqsubseteq \bigsqcup_{j=1}^n D_j$ holds
  for all $i\in[1,m]$, (ii)~if $\KB\models C_i\sqsubseteq D_j$ holds for some $j\in[1,n]$, then
  $\KB\models C_i\sqsubseteq \bigsqcup_{j=1}^n D_j$. So we are only
  left to show the converse of (ii): assuming that for all $j\in[1,n]$, $\KB\not\models
  C_i\sqsubseteq D_j$ holds, we shall prove that
  $\KB\not\models C_i\sqsubseteq \bigsqcup_{j=1}^n D_j$.
  
  By assumption and Lemma~\ref{lem:counterex}, the canonical model
  $(\I,d)$ of $C_i$ is such that $(\I,d)\models \neg D_j$ for all
  $j\in[1,n]$. Therefore $(\I,d)\models \neg \bigsqcup_{j=1}^n
  D_j$. Moreover, $(\I,d)$ satisfies both \KB and $C_i$ by
  Lemma~\ref{lem:mod-of-norm-C}. Then \I and $d$ witness that
  $\KB\not\models C_i\sqsubseteq \bigsqcup_{j=1}^n D_j$.
\end{proof}

\noindent
Now that the semantic properties are laid out, we focus on the
algorithms.  Roughly speaking, the next lemma says that
$\SSA^{\O_\K^+}$ decides whether the canonical model $(\I,d)$ of
$C$ satisfies $D$.

\begin{lemma}
  \label{lem:SSA-evaluates-D}
  If $C\sqsubseteq D$ is elementary w.r.t.\ \K and \O, $C\neq\bot$, and
  $(\I,d)$ is the canonical model of $C$, then $$\SSA^{\O_\K^+}(C
  \sqsubseteq D)=\mathtt{true}\mbox{\: iff \:} (\I,d)\models D \,.$$
\end{lemma}

\begin{proof}
  By structural induction on $D$. If $D=A$ (a concept name), then by definition
  $\SSA^{\O_\K^+}(C \sqsubseteq D)=\mathtt{true}$ iff there exists an inclusion $\bigsqcap_{i=1}^n A_i \sqsubseteq A$ in $\pos(\O_\K^+)$ such that the $A_i$'s are the top-level concept names in $C$ (cf.\ line~3 of Algorithm~\ref{alg:sso}). By def.\ of \I,
  this holds iff $d\in A^\I$, that is, $(\I,d)\models D$. This
  proves the base case.
  
  If $D=D_1\sqcap D_2$, then the lemma follows easily from the
  induction hypothesis (cf.~line 6 of Algorithm~\ref{alg:sso}).
  
  If $D=\exists R.D_1$, then $\SSA^{\O_\K^+}(C \sqsubseteq
  D)=\mathtt{true}$ iff: (i)~$C$ has a top-level subconcept $\exists
  R.C_1$, and (ii)~$\SSA^{\O_\K^+}(C_1 \sqsubseteq D_1)=\mathtt{true}$
  (cf.~line 5).  Moreover, by definition of \I, $(\I,d)\models D$ holds iff fact (i)
  holds and: (ii')~$(\I_i,d_i)\models D_1$, where $(\I_i,d_i)$ is a
  canonical model of $C_1$.  By induction hypothesis, (ii) is equivalent to (ii'),
  so the lemma immediately follows.
  
  If $D=\exists f.[\ell,u]$, then $\SSA^{\O_\K^+}(C \sqsubseteq
  D)=\mathtt{true}$ iff the following property holds:
  \begin{equation}
    \label{eq:lem-SSA:1}
    \mbox{$C$ has a top-level subconcept $\exists f.[\ell',u']$ such that
      $[\ell',u']\subseteq [\ell,u]$}
  \end{equation}
  (cf.~line~4).  We are only left to prove that (\ref{eq:lem-SSA:1})
  is equivalent to $(\I,d)\models D$.
  
  Property (\ref{eq:lem-SSA:1}) implies (by construction of \I) that
  $(d,u')\in f^\I$ and $u'\in[\ell,u]$, that is, $(\I,d)\models D$.
  
  Conversely, if $(\I,d)\models D$, then there exists
  $u'\in\Delta^\I$ such that $(d,u')\in f^\I$ and
  $u'\in[\ell,u]$. Then, by construction of \I, $C$ must have a
  top-level subconcept $\exists f.[\ell',u']$. By interval safety
  (that is implied by the hypothesis that $C\sqsubseteq D$ is
  elementary), the fact that $[\ell',u']$ and $[\ell,u]$ have $u'$
  in common implies $[\ell',u'] \subseteq [\ell,u]$. Therefore,
  (\ref{eq:lem-SSA:1}) holds.  This completes the proof.
\end{proof}

\noindent
We are now ready to prove that \PLRO is correct and complete.

\begin{theorem}
  \label{thm:PLRO-correct}
  Let \tup{\K,\O,C \sqsubseteq D} be any instance of \PLSO.  Then
  $$\PLRO(\K,C \sqsubseteq D)=\mathtt{true} \mbox{\: iff \:} \K \cup \O \models
  C\sqsubseteq D \,.$$
\end{theorem}

\begin{proof}
  $D$ is of the form $D_1\sqcup\ldots\sqcup D_n$.
  Let $C_1 \sqcup \ldots \sqcup C_m$ be the concept $C''$ computed by lines 2 and 3 of \PLRO.
  We start by proving the following claim, for all $i=1,\ldots,m$ and  $j=1,\ldots,n$:
  \begin{equation}
    \label{PLRO:1}
    \SSA^{\O_\K^+}(C_i\sqsubseteq D_j) = \easytt{true} \mbox{\:
      iff \:} \K_\O^-\cup\pos(\O_\K^+) \models C_i\sqsubseteq D_j \,.
  \end{equation}
  There are two possibilities. If $C_i=\bot$, then clearly $\K_\O^-\cup\pos(\O_\K^+) \models
  C_i\sqsubseteq D_j$ and $\SSA^{\O_\K^+}(C_i \sqsubseteq
  D_j)=\mathtt{true}$ (see line 2 of Algorithm~\ref{alg:sso}), so
  (\ref{PLRO:1}) holds in this case.  If $C\neq\bot$, then note that
  $C_i\sqsubseteq D_j$ is elementary w.r.t.\ \K and \O by construction
  of $C''$ (which is obtained by splitting the intervals of the
  normalization of $C$ w.r.t.\ \K and \O).  Then (\ref{PLRO:1}) follows immediately from
  lemmas~\ref{lem:counterex} and \ref{lem:SSA-evaluates-D}.

  By (\ref{PLRO:1}) and convexity (Lemma~\ref{lem:convexity-PL}), we
  have that lines 5--11 of Algorithm~\ref{alg:main-O} return
  \easytt{true} iff $\K_\O^- \cup \pos(\O_\K^+) \models C''\sqsubseteq
  D$.  Moreover, $C''$ can be equivalently replaced by $C$ in this
  entailment, by Proposition~\ref{prop:norm-O} and
  Proposition~\ref{prop:interval-norm}. The resulting entailment is
  equivalent to $\K \cup \O \models C\sqsubseteq D$ by
  Theorem~\ref{thm:compl-w-shifting}.  It follows that
  Algorithm~\ref{alg:main-O} returns \easytt{true} iff $\K \cup \O
  \models C\sqsubseteq D$.
\end{proof}

\noindent
\PLRO runs in polynomial time, modulo the cost of oracle queries.

\begin{lemma}
  \label{thm:PLRO-complexity}
  $\PLRO(\K,C\sqsubseteq D)$ runs in time $O(|C\sqsubseteq
  D|^{c+1}+|C\sqsubseteq D|^2\cdot|\K|)$ using an oracle for
  $\pos(\O_\K^+)$, where $c$ is the maximum number of interval
  constraints occurring in a single simple concept of $C$.
\end{lemma}

\begin{proof}
  Each query to the oracle triggered by the application of
  normalization rule~7 or by line~3 of \SSO counts as one step of
  computation, according to the definition of time complexity for
  oracle machines. Then, by the same arguments used in the proof of
  Lemma~\ref{lem:PLR-complexity}, the computation of the normalization
  steps in lines~2 and 3 of \PLRO takes time $O(|C|^2\cdot
  |\K|+|C|\cdot|D|^c)$, while the loops spanning over lines 5--9 take
  time $O(|D|\cdot|C|\cdot|\K|)$.  The lemma follows by expressing the
  size of $C$ and $D$ in terms of $|C\sqsubseteq D|$
  (cf.\ Lemma~\ref{lem:PLR-complexity}).
\end{proof}

\noindent
As a consequence of the above lemma, the classes of subsumption
instances where $c$ is bounded can be decided in polynomial time,
modulo the cost of oracle queries.

\begin{definition}
  For all non-negative integers $c$, let $\PLSO_c$ be the set of \PLSO
  instances \tup{\K,\O,C\sqsubseteq D} such that the maximum number of
  interval constraints occurring in a single simple concept of $C$ is
  bounded by $c$.
\end{definition}

\begin{theorem}
  \label{thm:PLSO-complexity}
  For all $c$, $\PLSO_c$ is in $\mathsf{P}^{\pos(\O_\K^+)}$.
\end{theorem}

\noindent
Computing the consequences of \O, in general, is intractable, although
\O is restricted to Horn-\SRIQ knowledge bases. For other Horn DLs,
however -- like the profiles of OWL2 and their generalizations \ELp
and \DLLh -- subsumption checking is tractable.  By
Theorem~\ref{thm:PLSO-complexity}, the tractability of convex oracles
extends to reasoning in \PL with such oracles. More precisely, it
suffices to assume that membership in $\pos(\O_\K^+)$ can be decided
in polynomial time, since in that case
$\mathsf{P}^{\pos(\O_\K^+)}=\mathsf{P}$.  This is what happens when \O
is in \ELp and \DLLh, since the axioms shifted from \K to \O
(i.e.\ $\O_\K^+\setminus \O$) can be expressed both in \EL and in
\DLL, therefore $\O_\K^+$ is in the same logic as \O. This is formalized as follows:

\begin{definition}
  For all integers $c\geq 0$, let $\PLSO_c^{\easycal{DL}}$ be the set
  of instances of $\PLSO_c$ whose oracle is in \easycal{DL}.
\end{definition}

\begin{corollary}
  \label{cor:profile-O}
  For all $c\geq 0$, $\PLSO_c^\ELp$ and
  $\PLSO_c^{\mathit{DL{-}lite}^\easycal{H}_\mathit{horn}}$ are in
    \easysf{P}.\footnote{The same holds if \O is in \ELpp, since all the relevant properties of \ELp in this context, such as tractability and convexity, hold for \ELpp, too. }
\end{corollary}

\noindent
It can also be proved that the normalization rules in
Table~\ref{norm-rules-O} may be used as a \emph{policy validation}
method, to detect unsatisfiable policies. 

\begin{theorem}
  \label{PL-satisfiability}
  Let \tup{\K,\O,q} be a \PLSO instance and $C$ be a \PL concept such
  that $\sig{C}\cap\sig{\O}\subseteq \NC$.
  \begin{enumerate}
  \item A \PL concept $C=C_1\sqcup\ldots\sqcup C_n$ is unsatisfiable
    w.r.t.\ $\K\cup\O$ iff $C_i\leadsto^* \bot$ for all
    $i\in[1,n]$.\footnote{As usual, $\leadsto^*$ denotes the reflexive
      and transitive closure of $\leadsto$\,.}
  \item Under the above hypotheses, \PL concept satisfiability testing
    w.r.t.\ $\K\cup\O$ is in $\mathsf{P}^{\pos(\O_\K^+)}$ (hence in
    \easysf{P} if \O belongs to a tractable logic).
  \end{enumerate}
\end{theorem}

\begin{proof}
  By Prop.~\ref{prop:norm-O} and Lemma~\ref{lem:mod-of-norm-C}, $C$ is
  satisfiable w.r.t.\ $\K_\O^-\cup\pos(\O_\K^+)$ iff $C_i\leadsto^*
  \bot$ does \emph{not} hold for some $i\in[1,n]$. Moreover, by
  Theorem~\ref{thm:compl-w-shifting},
  \begin{equation*}
    \K_\O^-\cup\pos(\O_\K^+) \models  C\sqsubseteq \bot \mbox{\: iff
      \:} \K\cup\O \models C\sqsubseteq \bot \,.
  \end{equation*}  
  Point~1 immediately follows.  Next, note that normalization can be
  computed in polynomial time using an oracle for $\pos(\O_\K^+)$.
  This can be shown with a straightforward adaptation of the proof of
  Lemma~\ref{lem:norm-cost} that takes into account the oracle queries
  in rule~7 (the details are left to the reader). Then Point~2 follows
  from the complexity of normalization and Point~1.
\end{proof}

\subsection{Related tractability and intractability results}

Our tractability result for combinations of \PL knowledge bases and
oracles in \ELp and \DLLh extends the known tractable fragments of
OWL2. The novelty of \PL with oracles in \ELp lies in the extension of
\EL with functional roles and non-convex concrete domains;
unrestricted combinations of such constructs are generally
intractable, when the knowledge base -- as in our subsumption
instances -- is nonempty and contains unrestricted GCIs.


In particular, in the extension of \EL with functional roles, subsumption checking
is EXPTIME-complete, in general \cite{DBLP:conf/ijcai/BaaderBL05}.  A tractability
result for empty TBoxes is reported in
\cite[Fig.~4]{DBLP:conf/ecai/HaaseL08}; however, in the same paper, it
is proved that even with acyclic TBoxes, subsumption is coNP-complete.

The tractability of an extension of \EL with non-convex concrete
domains has been proved in \cite{DBLP:conf/ecai/HaaseL08}, under the
assumption that the TBox is a set of \emph{definitions} of the form
$A\equiv C$, where each $A$ is a concept name and appears in the
left-hand side of at most one definition.

The overall tractability threshold for the \DLL family can
be found in \cite{DBLP:journals/jair/ArtaleCKZ09}.  The results most
closely related to our work are the following.

The data complexity of query answering raises at the first level of
the polynomial hierarchy if \DLLh is extended with functional
roles. Knowledge base satisfiability becomes EXPTIME-complete
(combined complexity).

Under three syntactic restrictions \cite[{\bf
    A$_1$--A$_3$}]{DBLP:journals/jair/ArtaleCKZ09} and the unique name
assumption, all the aforementioned reasoning tasks remain
tractable.

The most expressive knowledge representation language enjoying a complete structural
subsumption algorithm -- to the best of our knowledge -- is CLASSIC
\cite{DBLP:journals/jair/BorgidaP94}, that supports neither concept
unions ($\sqcup$) nor qualified existential restrictions ($\exists
R.C$).  If unions were added, then subsumption checking would
immediately become co\NP-hard (unless concrete domains were
restricted) for the same reasons why unrestricted subsumption checking
is co\NP-hard in \PL (cf.~Theorem~\ref{thm:PL-is-NP-hard}).  On
the other hand, CLASSIC additionally supports qualified universal
restrictions (that strictly generalize \PL's range restrictions),
number restrictions, and role-value maps, therefore it is not
comparable to \PL.
The complexity of the extensions of \PL with CLASSIC's constructs is
an interesting topic for further research.

\subsection{Compiling oracles into \PL knowledge bases}
\label{sec:compilation}

Note that $\pos(\O_\K^+)$ might be \emph{compiled}, i.e.\ computed
once and for all, so as to reduce oracle queries to retrieval. After
such knowledge compilation, \PLRO could run in polynomial time,
\emph{no matter how complex \O's logic is}, provided that the subset
of $\pos(\O_\K^+)$ queried by \PLRO (i.e.\ the part of $\pos(\O_\K^+)$
that should be pre-computed) is polynomial, too.

This is not always the case.  The conjunctions of classes $\bigsqcap_i
A_i$ that may possibly occur in the left-hand side of subsumption
queries are exponentially many in the signature's size, and each of
them may potentially occur in a query to the oracle. So, in order to
limit the space of possible oracle queries and reduce the partial
materialization of $\pos(\O_\K^+)$ to a manageable size, we have to
limit the number of concepts that may occur in the left-hand side of
subsumption queries.

Fortunately, in SPECIAL's use cases, the subsumption queries
$C\sqsubseteq D$ that implement compliance checks have always a
business policy on the left-hand side, and the set of business
policies of a controller is rather stable and not large. So the prerequisite for
applying oracle compilation is satisfied.  We are further going to
show that the oracle can be compiled into a plain, oracle-free \PL
knowledge base, therefore the IBQ framework can be implemented with
the same efficiency as pure \PL reasoning.

We start the formalization of the above ideas by defining the
restricted class of problem instances determined by the given set of
business policies \BP.

\begin{definition}
  For all sets of \PL concepts \BP, let $\PLSO(\BP)$ be the set of all
  $\tup{\K,\O,C\sqsubseteq D} \in \PLSO$ such that $C\in\BP$.
\end{definition}

\noindent
The first step of the oracle compilation consists in transforming
business policies so as to collapse each conjunction of concept names
into a single concept name. We say that the result of this
transformation is in \emph{single-atom form}, which is recursively
defined as follows:

\begin{definition}
  A simple \PL concept $C$ is in \emph{single-atom form} if either
  \begin{enumerate}
  \item $C$ is of the form $(\bigsqcap_{i=1}^m \exists f_i.[l_i,u_i])
    \sqcap (\bigsqcap_{i=1}^k \exists R_i.C_i)$, where $m,k\geq 0$, and
    each $C_i$ is in single-atom form, or
  \item $C$ is of the form $ A \sqcap (\bigsqcap_{i=1}^m \exists f_i.[l_i,u_i])
    \sqcap (\bigsqcap_{i=1}^k \exists R_i.C_i)$ where $m,k\geq 0$, and
    each $C_i$ is in single-atom form.
  \end{enumerate}
  A full \PL concept $C_1\sqcup \ldots \sqcup C_n$ is in single atom
  form if $C_1,\ldots,C_n$ are all in single atom form.
\end{definition}

\noindent
The given business policies can be transformed in single atom form in linear time:

\begin{proposition}
  \label{prop:single-atom-cost}
  For all finite sets of concepts \BP there exist a set of concepts
  $\BP^*$ in single atom form, and a knowledge base $\O^*$ that
  belongs to both \EL and $\DLL_\mathit{horn}$, such that for all
  $\tup{\K,\O,C\sqsubseteq D} \in \PLSO(\BP)$ there exists an equivalent
  problem instance $\tup{\K,\O\cup\O^*,C^*\sqsubseteq D} \in \PLSO(\BP^*)$, that is:
  \[
  \K\cup \O \models C\sqsubseteq D \mbox{\: iff \:} \K \cup \O \cup
  \O^* \models C^*\sqsubseteq D \,;
  \]
  Moreover, $\BP^*$ and $\O^*$ can be computed in time $O(|\BP|)$.
\end{proposition}

\begin{proof}
  For all $C\in\BP$, we obtain the corresponding concept $C^*$ by
  replacing each intersection of multiple concept names in $C$ with a
  single fresh concept name, whose definition is included in $\O^*$.
  More precisely, if $C=C_1\sqcup\ldots\sqcup C_n$ then for all
  $j=1,\ldots,n$, replace each $$C_j = (\bigsqcap_{i=}^n
  A_i)\sqcap (\bigsqcap_{i=1}^m \exists f_i.[l_i,u_i]) \sqcap
  (\bigsqcap_{i=1}^k \exists R_i.D_i)$$ such that $n>1$ with $$C_j^* =
  B \sqcap (\bigsqcap_{i=1}^m \exists f_i.[l_i,u_i]) \sqcap
  (\bigsqcap_{i=1}^k \exists R_i.D_i^*)\,,$$ where $B$ is a fresh
  concept name and each $D_i^*$ is obtained by recursively applying
  the same transformation to $D_i$.

  The knowledge base $\O^*$ is the set of all the definitions $B
  \equiv (\bigsqcap_{i=}^n A_i)$ such that $B$ is one of the fresh
  concepts introduced by the above transformations and
  $\bigsqcap_{i=}^n A_i$ is the intersection replaced by $B$.

  Finally, let $\BP^*$ be the set of concepts
  $C^*=C_1^*\sqcup\ldots\sqcup C_n^*$ obtained with the above
  procedure.  Clearly, by construction, $\K\cup\O\cup\O^* \models
  C\equiv C^*$, for all $C\in\BP$. Moreover, $\K\cup\O\cup\O^*$ is a
  conservative extension of $\K\cup\O$. Therefore
  \begin{eqnarray*}
    \K\cup \O \models C\sqsubseteq D & \mbox{iff} &
    \K\cup \O\cup\O^* \models C\sqsubseteq D
    \\
     & \mbox{iff} &  \K\cup \O\cup\O^* \models C^*\sqsubseteq D \,.
  \end{eqnarray*}
  Concerning complexity, $\BP^*$ and $\O^*$ can be computed with a single
  scan of \BP; the generation of the fresh concepts $B$, the
  replacement of $\bigsqcap_{i=}^n A_i$ and the generation of the
  definition for $B$ take linear time in $|C_j|$. Therefore $\BP^*$
  and $\O^*$ can be computed in time $O(|\BP|)$.
\end{proof}

\noindent
By the above proposition, we can assume without loss of generality
that \BP is in single atom form. Note that the ontologies \K and \O,
in a typical application scenario, do not change frequently. So we can
fix them and assume that the concepts in \BP are already normalized
w.r.t.\ \K and \O.  The set of problem instances with fixed \K and \O
is defined as follows:
{\small
\[
\PLSO(\K,\O,\BP) = \{ \tup{\K',\O',C\sqsubseteq D} \in \PLSO \mid
\K'=\K,\ \O'=\O, \mbox{ and } C\in\BP \}\,.
\]%
}%

\noindent
The compilation of \K and \O into a single \PL knowledge base is defined as follows:
\[
\comp(\K,\O) = \K_\O^- \cup \{A\sqsubseteq B \mid (A\sqsubseteq B)\in\pos(\O_\K^+) \} \,.
\]
The correctness of oracle compilation is proved by the next theorem.

\begin{theorem}
  \label{thm:precomp}
  Let \K and \O be two knowledge bases in \PL and Horn-\SRIQ,
  respectively, and let \BP be a set of \PL concepts in single atom
  form and normalized w.r.t.\ \K and \O. Then, for all
  $\tup{\K,\O,C\sqsubseteq D}\in\PLSO(\K,\O,\BP)$,
  \[
  \PLRO(\K,C\sqsubseteq D) = \PLR(\comp(\K,\O), C\sqsubseteq D) \,.
  \]
\end{theorem}

\begin{proof}
  Since $C$ is already normalized w.r.t.\ \K and \O by hypothesis,
  line~3 of \PLRO computes the identity function (i.e.\ $C'=C$).  It
  is easy to see that line~2 of \PLR does the same.  First, note that
  the two versions of rules~4 and 6 (in Table~\ref{norm-rules} and
  Table~\ref{norm-rules-O}) apply to the same set of functionality and
  range axioms, since $\func(R)\in\comp(\K,\O) \Leftrightarrow
  \func(R)\in \K_\O^-$ and $\range(R,A)\in\comp(\K,\O) \Leftrightarrow
  \range(R,A)\in \K_\O^-$ (by definition of \comp). So there are no
  additional axioms in $\comp(\K,\O)$ that may trigger rules~4 or 6 in
  \PLR.  Second, since $C$ is in single atom form by hypothesis,
  rule~7 of Table~\ref{norm-rules} never applies.  The other
  normalization rules are the same for \PLRO and \PLR. We conclude
  that lines~2 and 3 of \PLRO and \PLR produce the same concept
  $C''=\splt{C}{D}$.

  Consequently, the loops in lines 5--9 of \PLRO and lines 4--8 of
  \PLR return the same result, too. To see this, it suffices to show
  that
  \begin{equation}
    \label{comp:1}
    \SSA^{\O_\K^+}(C_i\sqsubseteq D_j) = \SSA(\comp(\K,\O),C_i\sqsubseteq D_j) \,.
  \end{equation}
  The only difference between $\SSA^{\O_\K^+}$ and \SSA is in
  their line~3.  The membership tests executed by
  $\SSA^{\O_\K^+}$ in line~3 are all of the form
  $(A_1\sqsubseteq A) \in \pos(\O_\K^+)$, because $C_j$ is in single
  atom form (this follows from the hypothesis that $C$ is in single
  atom form). For the same reason, \SSA in line~3 checks whether $A_1
  \sqsubseteq^* A$.  The two tests are equivalent by definition of
  \comp, therefore (\ref{comp:1}) holds and the theorem is proved.
\end{proof}

\begin{remark}\rm
  \label{rem:efficiency-of-compilation}
  Note that the size of $\comp(\K,\O)$ is at most quadratic in the
  size of $\K\cup\O$, and that \PLR runs in polynomial time if the number of
  interval constraints per simple policy is bounded.  Therefore, under
  this assumption -- and after $\comp(\K,\O)$ has been computed --
  subsumption queries can be answered in polynomial time. If \O uses
  expressive constructs from Horn-\SRIQ, then their computational cost
  is confined to the compilation phase only, that is essentially a
  standard classification of $\O_\K^+$. \qed
\end{remark}

A caveat on the size of $\comp(\K,\O)$ is in order, here. If the given set
of policies \BP is not in single atom form, then \O must be replaced
by $\O\cup\O^*$, as shown in Proposition~\ref{prop:single-atom-cost},
where the size of $\O^*$ is $O(|\BP|)$. Therefore the size of
$\comp(\K,\O\cup\O^*)$ may grow quadratically with $|\BP|$.  This
relationship shows the influence of \BP's size on the complexity of
the oracle compilation approach.
So, unfortunately, oracle compilation is not always possible.  For
example, in the application of \PL to data markets illustrated in the
conclusions, we currently see no general criterion to restrict the
space of possible queries as required by the compilation method.

\begin{remark}\rm
  \label{rem:indirect-completeness}
  Using the compilation approach, the soundness and completeness of
  \PLR follow easily from the soundness and completeness of \PLRO,
  according to which $\K\models q$ holds if and only if
  $\PLR^\emptyset(\K,q)=\mathtt{true}$. So it suffices to show that
  $\PLR(\K,q) = \PLR^\emptyset(\K,q)$.  Note that
  $\comp(\K,\emptyset)$ is simply the closure of \K with respect to
  inclusions (that is, $\comp(\K,\emptyset)$ preserves the relation
  $\sqsubseteq^*$ associated to \K).  This fact and
  Theorem~\ref{thm:precomp}, respectively, imply that $$\PLR(\K,q) =
  \PLR(\comp(\K,\emptyset),q)= \PLR^\emptyset(\K,q).$$
\end{remark}

\noindent
Similarly, the equality $\PLR(\K,q) = \PLR^\emptyset(\K,q)$ and the correspondence between the closure $\sqsubseteq^*$ of the inclusions in \K and those in $\comp(\K,\emptyset)$,
immediately imply the following corollary of
Theorem~\ref{PL-satisfiability}:

\begin{corollary}
  \label{PL-satisfiability-old}
    Let \K be a \PL knowledge base.
  \begin{enumerate}
  \item A \PL concept $C=C_1\sqcup\ldots\sqcup C_n$ is unsatisfiable
    w.r.t.\ \K iff $C_i\leadsto \bot$ for all $i\in[1,n]$.
  \item \PL concept satisfiability w.r.t.\ \K can be checked in polynomial time.
  \end{enumerate}
\end{corollary}

\subsection{On the limitations posed on \PLSO}
\label{sec:limitations}

In this section we briefly motivate the restrictions posed on \PL
subsumption problems with oracles (\PLSO).  We start with the
requirements on the oracle. Recall that \O should be convex w.r.t.\ $\L_\O$ and should
not use nominals.
Convexity w.r.t.\ $\L_\O$ is essential for tractability, as shown
by the next result.

\begin{theorem}
  \label{thm:no-convex-no-tractable}
  If \O is not convex w.r.t.\ $\L_\O$ and enjoys the disjoint model
  union property, then there exists a \PL knowledge base \K such that
  deciding whether $\K\cup\O\models C\sqsubseteq D$ holds, given an
  interval-safe \PL subsumption query $C\sqsubseteq D$, is co-\NP
  hard.
\end{theorem}

\begin{proof}
  We are proving
  co\NP-hardness by reducing 3SAT to the complement of subsumption.
  By hypothesis, $\pos(\O)$ contains an inclusion
  \begin{equation}
    \label{ncnt:0}
    A_1 \sqcap \ldots \sqcap A_n  \sqsubseteq  B_1 \sqcup \ldots \sqcup
    B_m 
  \end{equation}
  such that none of the inclusions $A_1 \sqcap \ldots \sqcap A_n
  \sqsubseteq B_i$ belongs to $\pos(\O)$, for $i=1,\ldots,m$.  Without
  loss of generality, we can further assume that $A_1 \sqcap \ldots
  \sqcap A_n \sqsubseteq B_2 \sqcup \ldots \sqcup B_m$ is not in
  $\pos(\O)$ (if not, then discard some $B_i$ from (\ref{ncnt:0})
  until the right-hand side is a minimal union entailed by $A_1 \sqcap
  \ldots \sqcap A_n$). Now let \K be the following set of inclusions,
  where $A'$ and $B'$ are fresh concept names:
  \begin{eqnarray*}
    A' &\sqsubseteq& A_i \quad (i=1,\ldots,n) \\
    B_j &\sqsubseteq& B' \quad (j=2,\ldots,m) \,.
  \end{eqnarray*}
  Note that $\K\cup\O \models A' \sqsubseteq B_1 \sqcup B'$, by
  construction of \K and (\ref{ncnt:0}). We are going to represent the
  truth values \easyit{true} and \easyit{false} with $B_1$ and $B'$,
  respectively.

  Let $S$ be any instance of 3SAT, and let $p_1,\ldots,p_k$ be the
  propositional symbols occurring in $S$.  We assume without loss of
  generality that $p_1,\ldots,p_k$ do not occur in \K nor in \O. Each
  positive literal $p_i$ is encoded by $e(p_i)=\exists p_i.B_1$, while
  negative literals $\neg p_i$ are encoded by $e(\neg p_i) = \exists
  p_i.B'$. Then the negation of $S$ is encoded by
  \[
  D = \bigsqcup \{ e(\bar L_1) \sqcap e(\bar L_2) \sqcap e(\bar L_3)
  \mid L_1\lor L_2\lor L_3 \in S\} \,.
  \]
  (where each $\bar L_i$ is the literal complementary to $L_i$).
  We claim that the entailment
  \begin{equation}
    \label{nominals:1}
    \K \cup \O \not\models \big (\bigsqcap_i \exists p_i.A' \big) \sqsubseteq D
  \end{equation}
  holds iff $S$ is satisfiable (note that the above subsumption query is
  interval-free, hence trivially interval safe).  To prove the ``only
  if'' part, assume that (\ref{nominals:1}) holds, that is, there
  exists a pointed interpretation $(\I,d)$ such that $\I\models
  \K\cup\O$, $d\in\big (\bigsqcap_i \exists p_i.A' \big)^\I$ and
  $d\not\in D^\I$.  Since $\K\cup\O \models\big (\bigsqcap_i \exists
  p_i.A' \big) \sqsubseteq (\exists p_i.B_1) \sqcup (\exists p_i.B')$
  holds for each symbol $p_i$, there exists $d_i\in\Delta^\I$ such
  that $(d,d_i)\in p_i^\I$ and either $d_i\in B_1^\I$ or $d_i\in
  (B')^\I$.  Construct a truth assignment $\sigma$ for $S$ by setting
  \begin{equation*}
    \sigma(p_i) = \left\{
    \begin{array}{lp{14.5em}}
      \mathit{true} & if $d_i \in B_1^\I$ \,, \\
      \mathit{false} & if $d_i \not\in B_1^\I$ (therefore $d_i \in (B')^\I$) \,.
    \end{array}
    \right.
  \end{equation*}
  Since $d\not\in D^\I$, each clause $L_1\lor L_2\lor L_3$ of $S$
  contains a literal $p_i$ or $\neg p_i$ such that, respectively,
  $d_i\in B_1^\I$ or $d_i\in (B')^\I$, so $\sigma$ satisfies the
  literal, by definition. It follows immediately that $\sigma$
  satisfies $S$.

  Conversely, suppose that $S$ is satisfied by a truth assignment
  $\sigma$. We are going to construct a pointed interpretation
  $(\I,\bar d)$ that witnesses (\ref{nominals:1}).  Recall that neither
  $A_1 \sqcap \ldots \sqcap A_n \sqsubseteq B_1$ nor $A_1 \sqcap
  \ldots \sqcap A_n \sqsubseteq B_2 \sqcup \ldots \sqcup B_m$ belong
  to $\pos(\O)$. Then \O has two disjoint models $\M_1$ and $\M_2$ such that
  for some $d_1\in\Delta^{\M_1}$ and $d_2\in\Delta^{\M_2}$,
  \begin{eqnarray*}
    d_i & \in & (A_1 \sqcap \ldots \sqcap A_n)^{\M_i} \quad (i=1,2)
    \\
    d_1 & \not\in & B_1^{\M_1}
    \\
    d_2 & \not\in & (B_2 \sqcup \ldots \sqcup B_m)^{\M_2} .    
  \end{eqnarray*}
  The union $\U=\M_1 \uplus \M_2$ is still a model of $\O$ by
  hypothesis, and it can be extended to a model \J of $\K\cup\O$ by
  setting:
  \begin{eqnarray*}
    \Delta^\J &=& \Delta^\U
    \\
    (A')^\J &=& (A_1 \sqcap \ldots \sqcap A_n)^\J
    \\
    (B')^\J &=& (B_2 \sqcup \ldots \sqcup B_m)^\J \,.
  \end{eqnarray*}
  Finally, we extend \J to the witness \I as follows.  First let
  $\Delta^\I=\Delta^\J$ and choose any $\bar d\in\Delta^\J$. For all
  symbols $p_i$ define:
  \begin{eqnarray*}
    p_i^\I &=& \{(\bar d,d_1) \} \quad \mbox{ if } \sigma(p_i)=\mathit{false}\,,
    \\
    p_i^\I &=& \{(\bar d,d_2) \} \quad \mbox{ otherwise}\,.
  \end{eqnarray*}
  Note that $\bar d$ belongs to $(\bigsqcap_i \exists p_i.A' \big)$ by
  construction, so we are only left to prove that $\bar d \not\in
  D^\I$.  By assumption, each clause in $S$ contains a literal $L$
  satisfied by $\sigma$.  If $L=\neg p_i$, then $p_i^\I =\{(\bar
  d,d_1)\}$, therefore $\bar d \not\in (\exists p_i.B_1)^\I = e(\bar
  L)^\I$.  Similarly, if $L=p_i$, then $p_i^\I =\{(\bar d,d_2)\}$,
  therefore $\bar d \not\in (\exists p_i.B')^\I = e(\bar L)^\I$.  It
  follows immediately that $\bar d\not\in D^\I$.
\end{proof}

\noindent
Note that the above theorem shows that reasoning can be intractable even if \K and \O are \emph{fixed}.

The requirement that nominals must not occur in oracles is needed
for completeness. Our algorithm \PLRO -- and the other IBQ methods
where oracle queries are consistency tests of the form (\ref{ct}), or the equivalent
expressions of the form (\ref{oq}) -- are generalized by the
following definition, that accounts for the shifting of axioms from \K
to \O.

\begin{definition}
  \label{def:shifting-IBQ}
  Let \easycal{PI} be a set of problem instances of the form
  \tup{\K,\O,q}, where \K and \O are knowledge bases and $q$ is an
  inclusion.  A \emph{shifting IBQ mechanism} for \easycal{PI} is a
  pair of functions $(s,r)$ such that for all $\tup{\K,\O,q} \in \cal PI$:
  \begin{enumerate}
  \item $s(\K) \subseteq  \K$\,,

  \item $r( s(\K), \pos(\O \cup (\K\setminus s(\K))), q) =
    \easytt{true} \mbox{\: iff \:} \K\cup\O \models q$\,.
  \end{enumerate}
\end{definition}

\noindent
Informally speaking, $s$ determines which axioms are shifted from \K
to \O, and $r$ is the IBQ reasoner that decides entailment using the
modified knowledge bases. Shifting IBQ mechanisms do not exist if \O
may use nominals.

\begin{theorem}
  \label{thm:IBQ-incompleteness}
  Let \easycal{DL} be a description logics that supports nominals and
  disjointness axioms.  Let \easycal{PI} be any set of problem
  instances that contains all \tup{\K,\O,q} such that $\K=\emptyset$,
  \O is a \easycal{DL} knowledge base, and $q$ is an \EL
  inclusion.\footnote{We use \EL inclusions to strengthen our
    result, since they are a special case of \PL subsumption queries.}
  There exists no shifting IBQ mechanism for \easycal{PI}.
\end{theorem}

\begin{proof}
  Let $\K=\emptyset$ and $q=\exists R.(A\sqcap B) \sqcap \exists
  R.(A\sqcap \bar B) \sqsubseteq A'$.  Let
  \begin{eqnarray*}
    \O_1 &=& \{\disj(B,\bar B)\} \,,
    \\
    \O_2 &=& \{\disj(B,\bar B),\ A \sqsubseteq \{a\}\} \,.
  \end{eqnarray*}
  Note that both \tup{\K,\O_1,q} and \tup{\K,\O_2,q} belong to
  \easycal{PI}.
  
  It can be easily verified that $\pos(\O_1)=\pos(O_2)$; in
  particular, the two sets contain all the inclusions of the form
  $A_1\sqcap\ldots\sqcap A_m \sqsubseteq B_1\sqcup\ldots\sqcup B_n$
  such that:
  \begin{itemize}
  \item either the inclusion is a tautology (i.e.\ some concept name
    occurs both in the left-hand side and in the right-hand side),
  \item or both $B$ and $\bar B$ occur in the left-hand side.
  \end{itemize}
  However, $\K\cup\O_1 \not\models q$, while $\K\cup\O_2\models
  q$.  The latter fact holds because due to the nominal $\{a\}$, both
  $\exists R.(A\sqcap B)$ and $\exists R.(A\sqcap \bar B)$ should have
  the same role filler, that cannot satisfy the disjoint concepts $B$
  and $\bar B$ at the same time.  It follows that $q$ is trivially
  satisfied because its left-hand side is equivalent to $\bot$.

  Now suppose that a shifting IBQ mechanism $(s,r)$ for \easycal{PI}
  exists; we shall derive a contradiction.  By condition~2 of
  Definition~\ref{def:shifting-IBQ},
  \begin{eqnarray}
    \label{IBQ-incomplete:1}
    r( s(\K), \pos(\O_1 \cup (\K\setminus s(\K))), q) &=&
    \easytt{false}
    \\
    \label{IBQ-incomplete:2}
    r( s(\K), \pos(\O_2 \cup (\K\setminus s(\K))), q) &=&
    \easytt{true}\,.
  \end{eqnarray}
  However, $\K= s(\K)=\emptyset$ and consequently:
  \begin{eqnarray*}
    r( s(\K), \pos(\O_1 \cup (\K\setminus s(\K))), q) &=& r( \emptyset, \pos(\O_1), q)
    \\
    & = & r( \emptyset, \pos(\O_2), q) 
    \\
    & = & r( s(\K), \pos(\O_2 \cup (\K\setminus s(\K))), q)
  \end{eqnarray*}
  which contradicts (\ref{IBQ-incomplete:1}) and (\ref{IBQ-incomplete:2}).
\end{proof}

\begin{remark}\rm
  \label{rem:convexity}
  The above result complements the analogous negative result
  \cite[Theorem~4]{DBLP:journals/jair/GrauM12} that applies to
  knowledge bases \K with infinity axioms (while \PL
  knowledge bases have the finite model property).  On the other
  hand, \cite[Theorem~4]{DBLP:journals/jair/GrauM12} covers also more
  expressive oracle query languages.
\end{remark}

The proof of the above negative result is based on the limited
expressiveness of the oracle query language $\L_\O$. A similar
consideration applies to the requirement that $\sig{\O}$ may share
only concept names with $\sig{\K}$ and $\sig{q}$.  Without this
assumption, \PLRO is not complete. More generally:

\begin{theorem}
  \label{thm:IBQ-incompleteness-2}
  Let \easycal{PI} be a set of problem instances that contains all \tup{\K,\O,q}
  such that $\K=\emptyset$, \O is an \EL knowledge base and $q$ is an \EL
  inclusion (possibly sharing roles \hide{and concrete properties} with
  \O). There exists no shifting IBQ mechanism for \easycal{PI}.
\end{theorem}

\begin{proof}
  Let $\K=\emptyset$, $q=(\exists R.A \sqsubseteq \exists S.A)$,
  $\O_1=\emptyset$ and $\O_2=\{q\}$. Note that
  \begin{itemize}
  \item $\pos(\O_1)=\pos(\O_2)$ (both contain all and only the
    tautological inclusions of the form (\ref{oq}));
  \item $\K\cup\O_1\not\models q$\,;
  \item $\K\cup\O_2\models q$\,.
  \end{itemize}
  Then the assumption that a shifting IBQ mechanism for \easycal{PI}
  exists leads to a contradiction, by the same argument used in
  Theorem~\ref{thm:IBQ-incompleteness}.  \hide{A similar proof holds if $q$
  and the oracle share concrete properties; it suffices to replace
  $\exists R.A$ and $\exists S.A$ in the above argument with $\exists
  f.[l,u]$ and $\exists g.[l,u]$, respectively.}
\end{proof}

\noindent
In the light of the above negative results, a natural question is
whether an oracle query language more expressive than $\L_\O$ would
remove the need for the restrictions on nominals and roles.  Note that
IBQ mechanisms for shared roles have already been introduced in
\cite{DBLP:journals/jair/GrauM12}.  For a fragment of \EL, there
exists an IBQ algorithm that terminates in polynomial time.  Nominals
are not allowed, but shared roles are, under suitable conditions.

In order to support more expressive oracle queries, \PLRO and \SSO
should be extensively changed, though.  The proofs of the above negative
results reveal that the simple treatment of existential restrictions
\hide{and interval constraints}in \SSO should be replaced with a more
complex computation, involving oracle queries, and it is currently not
clear how significantly such changes would affect the scalability of
reasoning and the possibility of compiling oracles into \PL knowledge
bases.
Given that scalability is one of SPECIAL's primary
requirements, and that there is no evidence that shared roles \hide{and
concrete properties} are needed by SPECIAL's application scenarios
(cf.\ Remark~\ref{rem:shared-roles}), we leave this question as an
interesting topic for further research.


\section{Experimental Assessment}
\label{sec:experiments}

In this section we describe a Java implementation of \PLR and compare
its performance with that of other popular engines.
We focus on \PLR (as opposed to the more complex \PLRO) because
SPECIAL's application scenarios are compatible with the oracle
compilation into a \PL knowledge base illustrated in Section~\ref{sec:compilation}.
The implementation and experimental evaluation of \PLRO, that may be
interesting in other applications of \PL, lie beyond the scope of
this paper.

SPECIAL's engine is tested on two randomly generated sets of inputs.
The first set is based on the knowledge base and policies
developed for Proximus and Thomson Reuters. Consent policies are
generated by modifying the business policies, mimicking a selection of
privacy options from a list provided by the controller.  This first
set of test cases is meant to assess the performance of the engines in
the application scenarios that we expect to arise more frequently in
practice. The second set of experiments, that makes use of larger
knowledge bases and policies, is meant to predict the behavior of the
engines in more complex scenarios, should they arise in the future.

The implementation of \PLR and its optimizations are described in the
next subsection. Then Section~\ref{sec:test-gen} illustrates the test
cases used for the evaluation. Finally, Section~\ref{sec:performance}
reports the results of the experiments.

\subsection{Prototype Implementation and Optimization}
\label{sec:prototype}

\PLR is implemented in Java and it is distributed as a .jar file. The
reasoner's class is named \emph{PLReasoner}, and supports the standard
OWL APIs, version 5.1.7. The package includes a complete
implementation of \PLR, including the structural subsumption algorithm
\SSA, and the preliminary normalization phases, based on the 7 rewrite
rules and on the interval splitting method for interval safety.

The interval splitting method has been refined in order to reduce the
explosion of business policies.
The reason for refinements can be easily seen: if a business policy
contains interval $[1,10]$ and a consent policy contains $[5,10]$,
then the method illustrated in (\ref{interval-splitting}) splits
$[1,10]$ into the (unnecessarily large) set of intervals
\[
[1,1],\ [2,4],\ [5,5],\ [6,9],\ [10,10]\,,
\]
that cause a single simple policy to be replaced with 5 policies.
Note that for interval safety the splitting $[1,4],\ [5,10]$ would be
enough. While (\ref{interval-splitting}) is convenient in the
theoretical analysis -- because it has a simpler definition and it
does not increase asymptotic complexity -- a more articulated
algorithm is advisable in practice.  Here we only sketch the
underlying idea: each interval end point is classified based on
whether it occurs only as a lower bound, only as an upper bound, or
both. A singleton interval is generated only for the third category of
endpoints, while the others are treated more efficiently.  In
particular, in the above example, 1 and 5 occur only as lower bounds;
this allows to generate non-singleton sub-intervals that have 1 and 5
as their lower bound.  Moreover, 10 occurs only as an upper bound;
this allows to create a non-singleton sub-interval where 10 is the
upper bound. Accordingly, the refined splitting algorithm generates
only the two intervals $[1,4]$ and $[5,10]$.

Several other optimizations have been implemented and assessed. The
corresponding versions of \PLR are described below:

\subsubsection*{PLR c}
The normalization steps (lines~2 and 3 of \PLR) are one of the most
expensive parts of the reasoner. In order to reduce their cost, two
caches are introduced. The first cache stores the business policies that have
already been normalized w.r.t.\ \K (line~2 of \PLR).  In this way, the
seven rewrite rules are applied to each business policy only once;
when the policy is used again, line~2 simply retrieves the normalized
concept from the cache.
This optimization is expected to be effective in SPECIAL's application
scenarios because only business policies need to be normalized, and
their number is limited.  So the probability of
re-using an already normalized policy is high, and the cache is
not going to grow indefinitly; on the contrary its size is expected to be
moderate.

Similarly, a second cache indexed by the two policies $C$ and
$D$ stores the concepts $\splt{C}{D}$ already computed (thereby
speeding up line~3 of \PLR, that is, the interval splitting step
needed for interval safety).

\subsubsection*{PLR 2n, PLR c 2n}
\PLR\!2n normalizes both $C$ \emph{and} $D$ with the seven
rewrite rules, before computing $\splt{C}{D}$.  Since the rewrite
rules may merge and delete the intervals of $D$, this optimization potentially
reduces the number of splitting points and, consequently, the size of
$\splt{C}{D}$.
We denote with \mbox{\PLR\!c\,2n} the version of \PLR that exploits both the
caches of \PLR\,c and applies double normalization, as \PLR\,pre.

\subsubsection*{PLR pre, PLR pre 2n}
Sometimes the two normalization phases can be pre-computed.  When the
set of business policies and the set of intervals that may occur in
consent policies are known in advance, the seven rules and interval splitting can be
applied once and for all before compliance checking
starts.
For example, intervals are available in advance when the minimum or
maximum storage time are determined by law, or when the duration
options available to data subjects when consent is requested are
specified by the data controller.
This version of the engine is designed for such scenarios. The given
set of business policies is fully normalized before compliance
checking starts, and stored in the caches supported by \PLR c.
During compliance checking, lines~2 and 3 only retrieve concepts from
the caches.  In this way the cost of a compliance check is almost
exclusively the cost of \SSA.  This version of \PLR will be evaluated
by measuring compliance checking time only; preliminary normalizations
are not included.

\subsection{Test case generation}
\label{sec:test-gen}

The first set of test cases is derived from the business policies
developed for the pilots of Proximus and Thomson Reuters; these
policies will be denoted with $P_\mathsf{PXS}$ and $P_\mathsf{TR}$
respectively.

In each compliance check $P_B \sqsubseteq P_C$, $P_B$ is a union of
simple business policies randomly selected from those occurring in the
pilots' policy ($P_\mathsf{PXS}$ or $P_\mathsf{TR}$).  Since $P_B$
describes the activity of a business process of the data controller,
the random choice of $P_B$ essentially corresponds to a random
distribution of the controller's data processing activities
(abstracted by the simple policies) across its business processes.

The consent policy $P_C$ is the union of a set of simple policies
$P_C^i$ ($i=1,\ldots,n$) randomly selected from the pilots' policy, and randomly perturbed by
replacing some vocabulary terms with a different term.  The random
selection mimicks the opt-in/opt-out choices of
data subjects with respect to the various data processing activities modelled
by the simple policies.  Similarly, the random replacement of terms
simulates the opt-in/opt-out choices of the data subject w.r.t.\ each
component of the selected simple policies. More precisely, if the
modified term occurring in $P_C^i$ is a superclass (resp.\ a subclass)
of the corresponding term in the original business policy, then the
data subject opted for a broader (resp.\ more restrictive) permission
relative to the involved policy property (e.g.~data categories,
purpose, and so on).

\begin{table}[!h]
  \begin{center}
    \small
    \begin{tabular}{lrr}
      \hline
      & \multicolumn{1}{c}{\normalsize Proximus (PXS)} & \multicolumn{1}{c}{\normalsize Thomson Reuters (TR)}
      \\
      \hline
      \hline
      \it Ontology & &
      \\
      \hline
      inclusions & 186 & 186
      \\
      \disj & 11 & 11
      \\
      \range & 10 & 10
      \\
      \func & 8 & 8
      \\
      classification hierarchy height & 4 & 4
      \\
      \hline
      \it Business policies & &
      \\
      \hline
      \# generated policies & 120 & 100
      \\
      avg.\ simple pol.\ per full pol. & 2.71 & 2.39
      \\
      std.\ dev. & 1.72 & 1.86
      \\
      \hline
      \it Consent policies & &
      \\
      \hline
      \# generated policies & 12,000 & 10,000
      \\
      avg.\ simple pol.\ per full pol. & 3.77 & 3.42
      \\
      std.\ dev. & 2.02 & 2.03
      \\
      \hline
      \it Test cases & &
      \\
      \hline
      \# generated queries & 12,000 & 10,000
      \\
      \hline
    \end{tabular}
  \end{center}
  \caption{Size of the test cases inspired by the pilots}
  \label{pilot-test-cases}
\end{table}

\begin{table}[h]
  \begin{center}
    \small
    \begin{minipage}[b][140pt][t]{170pt}
      \begin{tabular}{lrrr}
        \hline
        Ontology size & \normalsize O1 & \normalsize O2 & \normalsize O3
        \\
        \hline
        \hline
        classes & 100 & 1,000 & 10,000
        \\
        roles & 10 & 50 & 100
        \\
        concrete properties & 10 & 25 & 50
        \\
        \func & 10 & 37 & 75
        \\
        \range & 5 & 25 & 50
        \\[2pt]
        \hline
        \\[-6pt]
        avg.\ \disj & 3 & 31 & 298
        \\
        avg.\ inclusions & 211 & 2224 & 23418
        \\
        avg.\ classification & 8 & 10 & 14 \\[-2pt]
         hierarchy height & & &
        \\[14pt]
        \hline
      \end{tabular}
    \end{minipage}
    \quad
    \begin{minipage}[b][140pt][t]{145pt}
      \begin{tabular}{lrr}
        \hline
        Concept size & P1 & P2
        \\
        \hline
        \hline
        max \#simple pol. & 10 & 100 \\[-2pt]
        per full pol.
        \\
        max \#top-level inters. & 10 & 20\\[-2pt]
        per simple subconcept
        \\
        max depth (nesting) & 4 & 9
        \\
        \hline
        avg.\ \#simple pol. & 6.8 & 50.1 \\[-2pt]
        per full pol. & &
        \\
        avg.\ depth & 2.4 &  5
        \\
        \hline
        Simple policy size
        \\
        \hline
        avg.\ \#intersections & 10.6 & 25.8
        \\
        avg.\ \#intervals & 3.7 & 9
        \\[2pt]
        \hline
      \end{tabular}
    \end{minipage}
  \end{center}
  \caption{Size of fully synthetic test cases}
  \label{synthetic-test-cases}
\end{table}

In this batch of experiments, the knowledge base is always SPECIAL's
ontology, that defines policy roles and the temporary vocabularies
for data categories, purpose categories, etc.
The size and number of this batch of experiments is reported in
Table~\ref{pilot-test-cases}.  The number of randomly generated
business policies is higher in one case because $P_\mathsf{PXS}$ has
more simple policies than $P_\mathsf{TR}$: the ratio is 20 generated
policies per simple policy.  Queries have been obtained by generating
100 consent policies for each business
policy. Table~\ref{pilot-test-cases} reports also the average number
of simple policies per generated policy and its standard deviation.
The size of each policy is limited by SPECIAL's usage policy format:
at most one interval constraint per simple policy, and nesting depth
2.

In the second set of experiments, both the ontologies and \PL
subsumptions are completely synthetic, and have increasing size in
order to set up a stress test for verifying the scalability of
SPECIAL's reasoner. Fifteen ontologies have been generated: five for
each of the three sets of parameters O1--O3 reported in
Table~\ref{synthetic-test-cases}.  The same table reports the
parameters used to generate the \PL concepts occurring in the queries,
according to two size specifications: P1 and P2.  

Note that approximately half of the roles and concrete properties are
functional, and half of the roles have a range axiom.  Ontologies have
been generated by randomly distributing classes over approximately
$\log(\mathit{\#classes})$ layers.  Then the specified number of
disjointness axioms have been generated, by picking classes on the same
layer.  Finally, about $2\cdot\#\mathit{classes}$ inclusions have been
created, mostly across adjacent layers, in such a way that no class
became inconsistent.  The ratio between the number of inclusions and
the number of classes is similar to the ratio that can be observed
most frequently in real ontologies,
cf.\ \cite{DBLP:journals/jair/MotikSH09,DBLP:conf/dlog/MatentzogluBP13,DBLP:journals/jar/KazakovKS14}.

We have generated 100 concepts of size P1 and 1000 of size P2, picking
interval endpoints from $[0,365]$ (one year, in days).  Each set has
been split into business and consent policies (resp.\ 30\% and 70\% of
the generated policies), that have been paired randomly to generate
test queries.
The number of queries of size P1 generated for each ontology is 50.
Let \ni be the maximum number of interval constraints
per simple policy after normalization w.r.t.\ the 7 rules\footnote{The reason for measuring \ni after normalization is explained later.} (for a given business policy).
The number of queries of size P2 generated for each ontology and each
business policy with $\ni\leq 5$ is 10. The maximum number of queries
for each ontology and each $\ni>5$ has been limited to 40, in order to keep
the length of the experiments within a reasonable range.  In this
case, we maximized the number of different business policies occurring
in the selected queries.

For each ontology \K, the business policies have been selected from
the available \K-consistent policies.  Furthermore, whenever possible,
queries have been selected in such a way that the number of positive and negative
answers are the same.
Table~\ref{synthetic-test-cases} illustrates the average size of
the generated policies for each parameter setting.  We have not limited the number of
interval constraints, in order to analyze the behavior of PLReasoner
as the number of intervals per simple policy grows (if it is not
bounded then \PL subsumption query answering is co\NP-hard).  The
maximum nesting level occurring in the generated policies is
approximately $\lceil\log_2(\mathit{max\ disjuncts})\rceil$.  

\subsection{Performance analysis}
\label{sec:performance}

The experiments have been run on a server with an 8-cores processor
Intel Xeon Silver 4110, 11M cache, 198\,GB RAM, running Ubuntu 18.04
and JVM 1.8.0\_181, configured with 32GB heap memory (of which less
than 700\,MB have been actually used in all experiments).  We have
\emph{not} exploited parallelism in the engine's implementation.

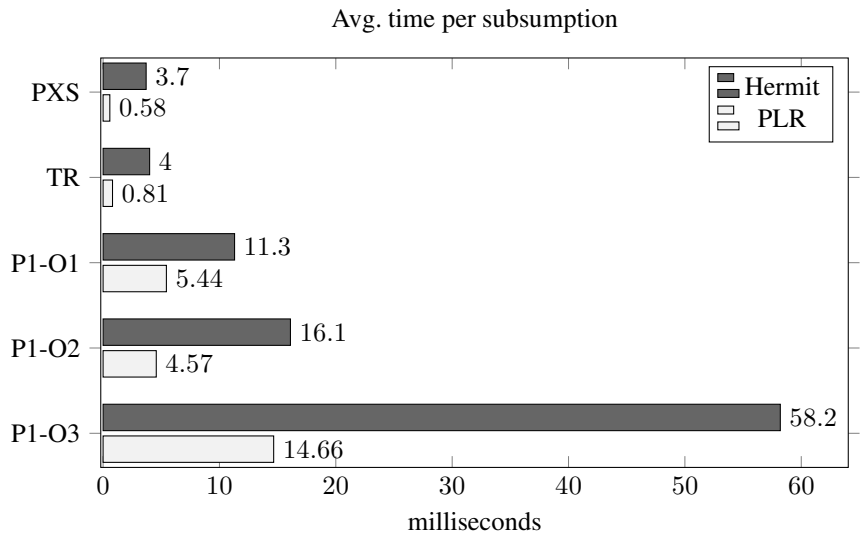
\begin{figure}[h]
  \begin{center}
    \begin{tikzpicture}
      \begin{axis}[
          title=Avg.\ time per subsumption,
          xbar,
          width=.95\textwidth,
          height=200pt,
          xmin=-0.2,
          xlabel={milliseconds},
          symbolic y coords={P1-O3,P1-O2,P1-O1,TR,PXS},
          reverse legend,
          ytick=data,
          nodes near coords, nodes near coords align={horizontal},
        ]
        \addplot [fill=black!05] coordinates {(0.58,PXS) (0.81,TR) (5.44,P1-O1) (4.57,P1-O2) (14.66,P1-O3)};
        \addplot [fill=black!60] coordinates {(3.7,PXS) (4,TR) (11.3,P1-O1) (16.1,P1-O2) (58.2,P1-O3)};
        \legend{PLR,Hermit}
      \end{axis}
    \end{tikzpicture}
  \end{center}
  \vspace*{-10pt}
  \caption{Comparisons on small/medium policies}
  \label{fig:hist}
\end{figure}

We start by illustrating the results for the test cases with small and
medium policies.  Figure~\ref{fig:hist} shows that \PLR is faster than
Hermit, over these test sets, even if no optimization is applied.   The size of the ontology
affects the performance of Hermit more than \PLR's (cf.\ the results
for O1, O2, and O3).

The good performance of \PLR over PXS and TR had to be expected, given
that the policies involved in these test sets are SPECIAL's usage
policies, that by definition contain at most one interval constraint
per simple policy, of the form $\exists \dur . [\ell,u]$. Let \ni
denote the maximum number of intervals per simple policy after
applying the rewrite rules, and recall that the size of $\splt{C}{D}$
may grow exponentially with \ni.  We have not limited \ni, while
generating the synthetic policies in P1 and P2, to see how the number
of intervals affects the performance of \PLR (recall that if \ni is
unbounded, then \PL subsumption is co\NP-complete).  We measured the
value of \ni after applying the rewrite rules, because
they can collapse and delete intervals, thereby reducing the
complexity of the subsequent interval splitting phase and the size of
$\splt{C}{D}$. After the application of the seven rules, the maximum
\ni over the business policies occurring in P1's queries is 9.
Figure~\ref{fig:hist} shows that the potential combinatorial explosion
of $\splt{C}{D}$ does not frequently occur with these policies. The
probability of splitting a single interval into many sub-intervals is
evidently not high. On the conttrary, a combinatorial explosion is
clearly observable in the test sets with large policies (P2);
Figure~\ref{fig:hist2} illustrates the results for the smallest
synthetic ontologies (O1).

\begin{figure}[h]
  \begin{center}
    \begin{tikzpicture}
      \begin{semilogyaxis}[
          title=P2-O1: Avg.\ time per subsumption query,
          ybar,
          width=.95\textwidth,
          height=200pt,
          ylabel={milliseconds},
          xlabel={\# intervals per simple policy},
          legend pos =north west,
          xtick=data,
        ]
        \addplot [fill=black!60] coordinates {(1,152.8)(2,159.8)(3,211.0)(4,165.9)(5,184.7)};
        \addplot [fill=black!05] coordinates {(1,25.4) (2,61.5) (3,291.1) (4,2732.0) (5,13470.9)};
        \legend{Hermit,PLR}
      \end{semilogyaxis}
    \end{tikzpicture}
  \end{center}
  \vspace*{-10pt}
  \caption{Impact of interval number per simple policy -- large policies}
  \label{fig:hist2}
\end{figure}
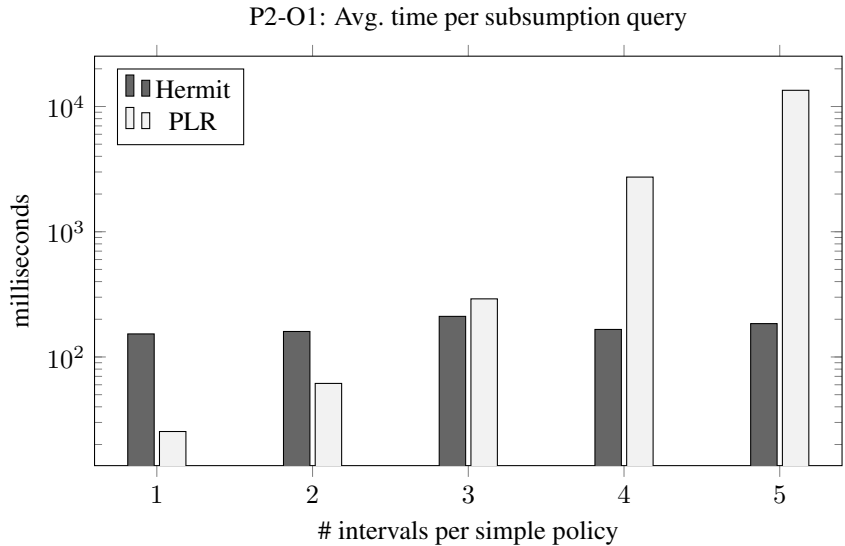

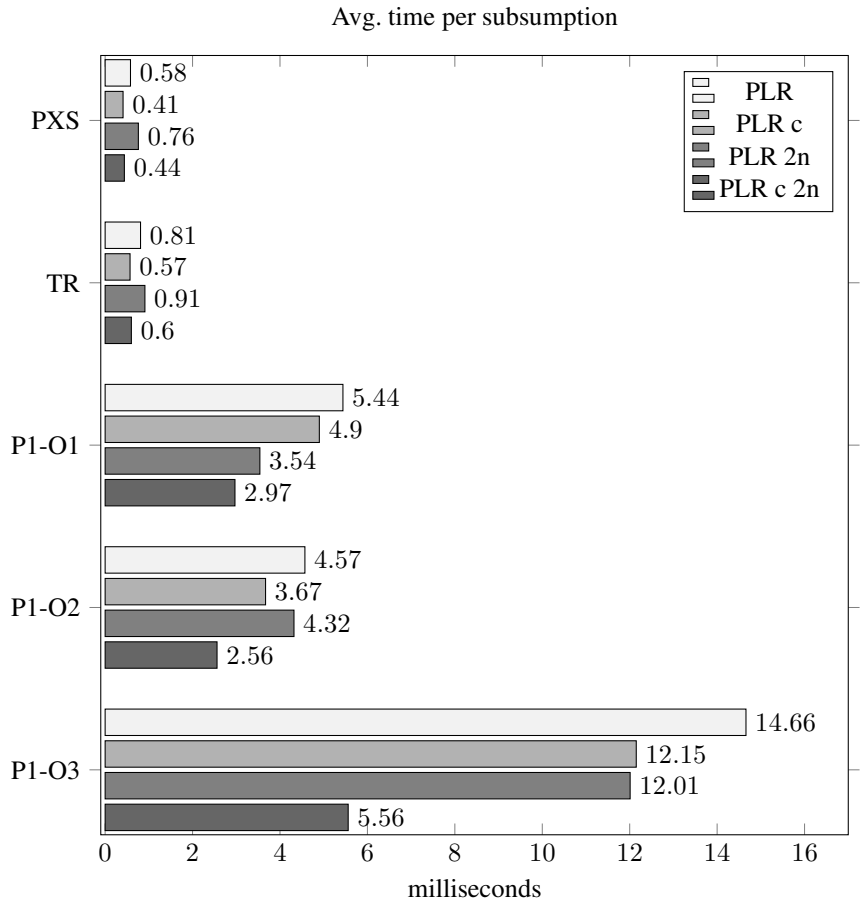
\begin{figure}[h]
  \begin{center}
    \begin{tikzpicture}
      \begin{axis}[
          title=Avg.\ time per subsumption,
          xbar,
          y dir=reverse,
          width=.95\textwidth,
          height=340pt,
          xmin=-0.1,
          xmax=17,
          xlabel={milliseconds},
          symbolic y coords={PXS,TR,P1-O1,P1-O2,P1-O3},
          reverse legend,
          ytick=data,
          nodes near coords, nodes near coords align={horizontal},
        ]
        \addplot [fill=black!60] coordinates {(0.44,PXS) (0.60,TR) (2.97,P1-O1) (2.56,P1-O2) (5.56,P1-O3)};
        \addplot [fill=black!50] coordinates {(0.76,PXS) (0.91,TR) (3.54,P1-O1) (4.32,P1-O2) (12.01,P1-O3)};
        \addplot [fill=black!30] coordinates {(0.41,PXS) (0.57,TR) (4.90,P1-O1) (3.67,P1-O2) (12.15,P1-O3)};
        \addplot [fill=black!05] coordinates {(0.58,PXS) (0.81,TR) (5.44,P1-O1) (4.57,P1-O2) (14.66,P1-O3)};
        \legend{PLR c 2n,PLR 2n,PLR c,PLR}
      \end{axis}
    \end{tikzpicture}
  \end{center}
  \vspace*{-10pt}
  \caption{Effectiveness of optimizations on small/medium policies}
  \label{fig:hist4}
\end{figure}

Then we analyzed the effects of the optimizations described in
Section~\ref{sec:prototype}.  Their effectiveness over small and
medium policies is illustrated by Figure~\ref{fig:hist4}.  The
normalization of consent policies (2n) brings no benefits with small
policies (actually, it slightly decreases the engine's performance, compare \PLR 2n
with \PLR, and \PLR c 2n with \PLR c).  Its benefits start to be
visible with medium policies.  The cache of normalized policies (\PLR
c) is the best option on small policies. On medium policies, the
combination of the caches with the normalization of consent policies
(\PLR c 2n) is the most effective optimization.

\begin{figure}
  \begin{center}
    \begin{tikzpicture}
      \begin{axis}[
          title=P2-O1: Avg.\ time per subsumption,
          xbar,
          y dir=reverse,
          width=.95\textwidth,
          height=350pt,
          xmin=-2,
          xmax=350,
          xlabel={milliseconds},
          ylabel={\# intervals per simple policy},
          reverse legend,
          legend pos =south east,
          ytick=data,
          nodes near coords, nodes near coords align={horizontal},
        ]
        \addplot [fill=black!60] coordinates {(24.2,1)(29.5,2)(44.3,3)(43.5,4)(54.5,5)};
        \addplot [fill=black!50] coordinates {(25.9,1)(31.5,2)(55.9,3)(53.0,4)(63.5,5)};
        \addplot [fill=black!30] coordinates {(24.6,1)(60.7,2)(302.6,3)(+inf,4)(+inf,5)};
        \addplot [fill=black!05] coordinates {(25.4,1)(61.5,2)(291.1,3)(+inf,4)(+inf,5)};
        \legend{PLR c 2n,PLR 2n,PLR c,PLR}
      \end{axis}
    \end{tikzpicture}
  \end{center}
  \caption{Effectiveness of optimizations on large policies and small ontologies}
  \label{fig:hist5}
\end{figure}
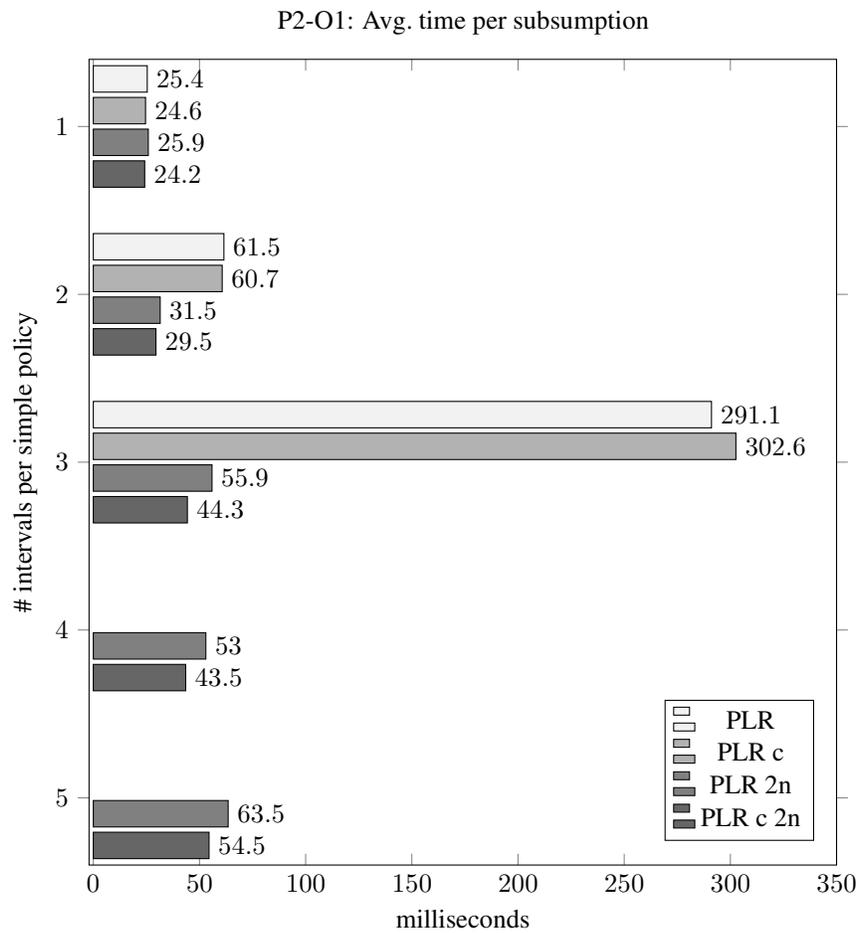

Over large policies (P2), the normalization of consent policies (2n)
is essential to mitigate the combinatorial explosion of $\splt{C}{D}$,
as shown in Figure~\ref{fig:hist5}.  The versions of \PLR that do not
normalize $D$ become impractical already for \ni=3, while the
computation time of \PLR 2n and \PLR c 2n moderately increases.  This
behavior can be explained by observing the effects of normalization
on this test set: after the application of the rewrite rules, the average number of intervals is about 10 times smaller, which reduces the probability of an exponential
growth of $\splt{C}{D}$.


\begin{figure}
  \begin{center}
    \begin{tikzpicture}
      \begin{axis}[
          title=P2-O2: Avg.\ time per subsumption,
          xbar,
          y dir=reverse,
          width=.95\textwidth,
          height=360pt,
          xmin=-2,
          xmax=310,
          xlabel={milliseconds},
          ylabel={\# intervals per simple policy},
          reverse legend,
          legend pos =north east,
          ytick=data,
          nodes near coords, nodes near coords align={horizontal},
        ]
        \addplot [fill=black!60] coordinates
        {(31.7,1)(33.0,2)(38.3,3)(54.8,4)(70.2,5)(96.4,6)(104.6,7)(272.4,8)(268.2,9)};
        \addplot [fill=black!05] coordinates
        {(149.3,1)(150.2,2)(148.6,3)(150.0,4)(158.7,5)(149.4,6)(146.1,7)(152.9,8)(147.8,9)};
        \legend{PLR c 2n,Hermit}
      \end{axis}
    \end{tikzpicture}
  \end{center}
  \caption{Hermit vs PLR with caches and double normalization}
  \label{fig:hist5b}
\end{figure}
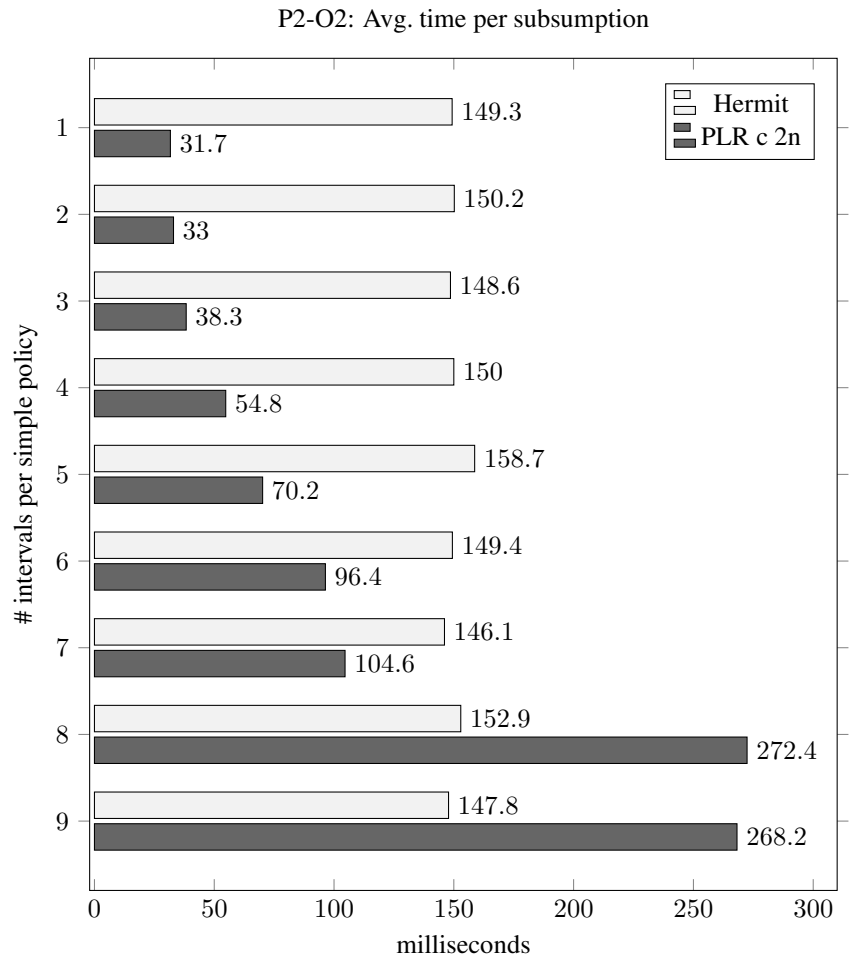

Next, in Figure~\ref{fig:hist5b}, we compare the best version of the
engine over medium/large policies (i.e.\ \PLR c 2n) with Hermit.  The optimizations delay the
effects of combinatorial explosions until \ni=7.  After this
threshold, Hermit becomes faster.


\begin{figure}
  \begin{center}
    \begin{tikzpicture}
      \begin{axis}[
          title=Avg.\ time per subsumption,
          xbar,
          y dir=reverse,
          width=.95\textwidth,
          height=350pt,
          xmin=-0.4,
          xmax=65,
          xlabel={milliseconds},
          symbolic y coords={PXS,TR,P1-O1,P1-O2,P1-O3},
          reverse legend,
          ytick=data,
          nodes near coords, nodes near coords align={horizontal},
        ]
        \addplot [fill=black!60] coordinates
                 {(0.33,PXS) (0.47,TR) (2.28,P1-O1) (1.70,P1-O2) (4.15,P1-O3)};
        \addplot [fill=black!50] coordinates
                 {(0.44,PXS) (0.60,TR) (2.97,P1-O1) (2.56,P1-O2) (5.56,P1-O3)};
        \addplot [fill=black!30] coordinates
                 {(0.41,PXS) (0.57,TR) (4.90,P1-O1) (3.67,P1-O2) (12.15,P1-O3)};
        \addplot [fill=black!05] coordinates
                 {(3.7,PXS) (4,TR) (11.3,P1-O1) (16.1,P1-O2) (58.2,P1-O3)};
        \legend{PLR pre,PLR c 2n,PLR c,Hermit}
      \end{axis}
    \end{tikzpicture}
  \end{center}
  \caption{Effectiveness of business policy pre-normalization on small/medium policies}
  \label{fig:pre1}
\end{figure}
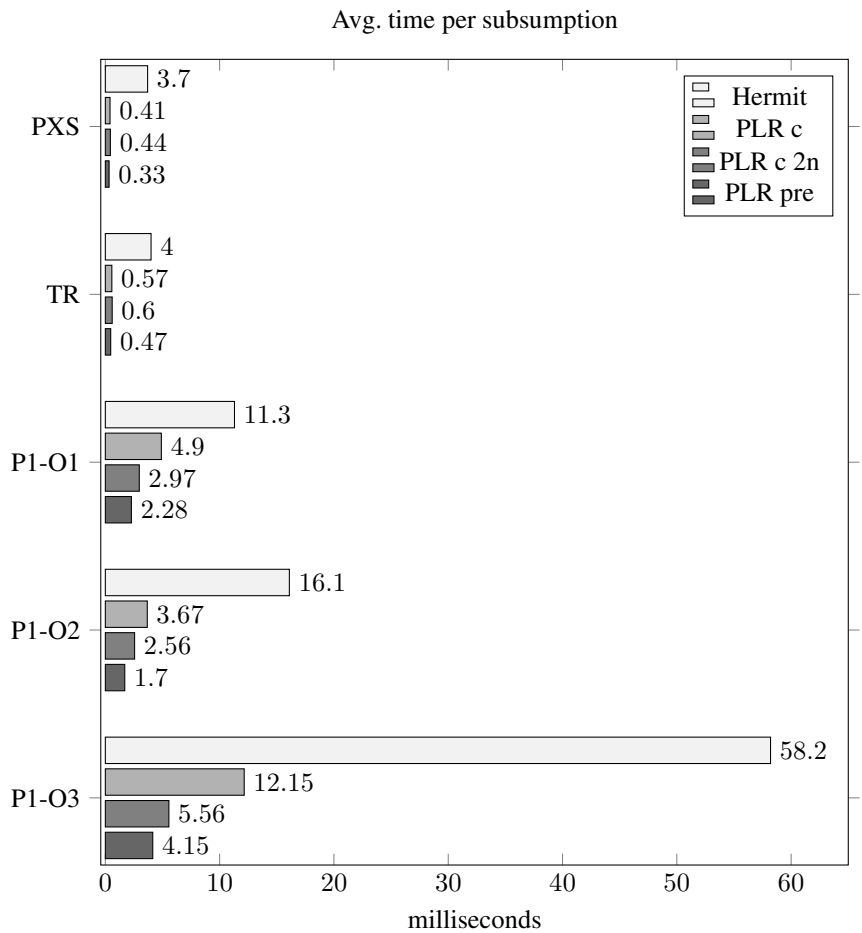

Finally, we analyzed the effectiveness of business policy
pre-normalization (pre). Recall that this approach is feasible in practice only if both
the business policies and the intervals that may occur in consent
policies are known in advance, and do not change frequently.
The effects of pre-normalization on small and medium policies is
remarkable: \PLR pre is approximately one order of magnitude faster
than Hermit, as shown in Figure~\ref{fig:pre1}. Over pilot-inspired
tests, pre-normalization brings the average time per subsumption query well below 500
$\mu$-seconds.

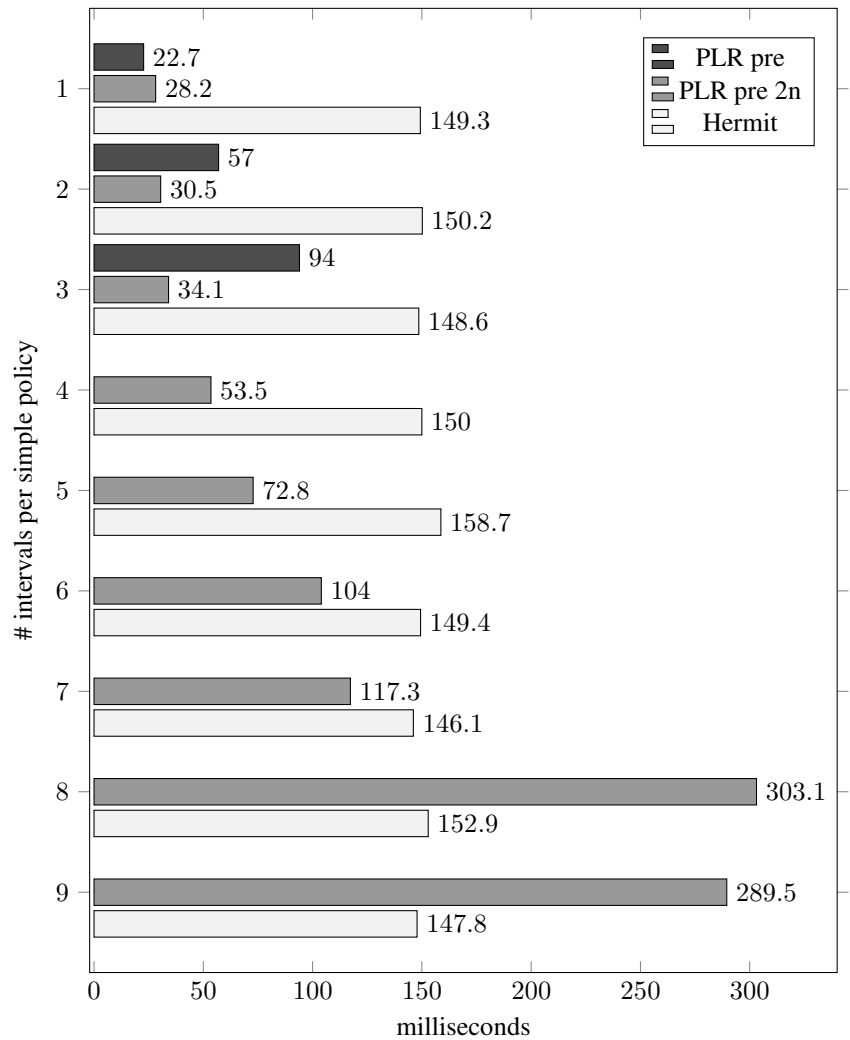
\begin{figure}
  \begin{center}
    \begin{tikzpicture}
      \begin{axis}[
          title=P2-O2: Avg.\ time per subsumption,
          xbar,
          y dir=reverse,
          width=.95\textwidth,
          height=410pt,
          xmin=-2,
          xmax=340,
          xlabel={milliseconds},
          ylabel={\# intervals per simple policy},
          reverse legend,
          legend pos =north east,
          ytick=data,
          nodes near coords, nodes near coords align={horizontal},
        ]
        \addplot [fill=black!05] coordinates
        {(149.3,1)(150.2,2)(148.6,3)(150.0,4)(158.7,5)(149.4,6)(146.1,7)(152.9,8)(147.8,9)};
        \addplot [fill=black!40] coordinates
        {(28.2,1)(30.5,2)(34.1,3)(53.5,4)(72.8,5)(104.0,6)(117.3,7)(303.1,8)(289.5,9)};
        \addplot [fill=black!70] coordinates
        {(22.7,1)(57.0,2)(94.0,3)};
        \legend{Hermit,PLR pre 2n,PLR pre}
      \end{axis}
    \end{tikzpicture}
  \end{center}
  \caption{Effectiveness of business policy pre-normalization on large policies. For $\ni\geq 4$ the response time of PLR pre exceeds Hermit's.}
  \label{fig:pre2}
\end{figure}

The effects of pre-normalization quickly disappear over large
policies. Figure~\ref{fig:pre2} shows that the explosion of
$\splt{C}{D}$ makes it necessary to apply also the normalization of
consent policies to delay combinatorial effects (cf.\ \PLR pre
2n). However, for \ni=8, \PLR~pre~2n is slower than Hermit, so
pre-normalization does not deal with the combinatorial explosion better
than \PLR~c~2n.

In order to assess the quality of \PLR's engineering, we compared it
with the specialized engine ELK on the test case based on Proximus'
policies.  ELK supports neither functionality axioms nor interval
constraints, so it is generally unable to handle \PL subsumption problems; however Proximus' policies make no use of intervals, and
never contain more than one expression $\exists R.C$ with the same
role $R$.  As a consequence, ELK computes the correct answer on this
test set.  The average time per subsumption query is 3.11
milliseconds; therefore all versions of \PLR are significantly
faster. In particular, \PLR pre is approximately one order of
magnitude faster.

We have also considered Konclude, a reasoner that is very competitive
on standard benchmarks \cite{DBLP:journals/ws/SteigmillerLG14}.
Konclude integrates a tableau algorithm with completion-based
saturation -- for pay-as-you-go behavior -- and adopts a wide range of optimizations.
The current version, however, is focussed on classification tasks; streams
of \PL subsumptions can be processed only at the cost of repeating
classification for each query. This prevents a fair comparison with Hermit and \PLR.

\section{Conclusions}
\label{sec:conclusions}

We have introduced the
description logic \PL in order to formalize the data usage policies
adopted by controllers as well as the consent to data processing
granted by data subjects. Checking whether the controllers' policies
comply with the available consent boils down to subsumption checking
between \PL concepts.  \PL can also formalize parts of the GDPR; then,
by means of subsumption checking, one can automatically check several
constraints on usage policies such as, for example:
\begin{itemize}
\item Are all the required policy properties specified?
\item Are all the required obligations specified?
\item Is the policy compatible with GDPR's constraints on cross-border data transfers?
\end{itemize}

\noindent
\PL supports interval constraints of the form $\exists f.[\ell,u]$ in
order to model limitations on data storage duration.  This feature
affects convexity, and places \PL outside the space of Horn DLs -- including the
tractable profiles of OWL2.

The frequency of compliance checks can be high, so \PL has been
designed to address scalability requirements by making the language as
simple as possible.  Despite this, general subsumption checking in \PL
is co\NP complete, due to the interplay of interval constraints and
concept union. However, reasoning becomes tractable if the number of
interval constraints in each simple policy on the left of subsumption
is bounded by a constant, as it happens in SPECIAL's usage policies, consent
policies, and in the formalization of the GDPR.  Under this
assumption, subsumption checking can be split into a polynomial-time
normalization phase and a subsequent subsumption check that can be
carried out by a fast, structural subsumption algorithm (\SSA).

The scalability of the complete algorithm (\PLR) has been
experimentally assessed.  Some of the test sets consist of realistic
policies and ontologies, derived from SPECIAL's pilots. Such policies
and ontologies are small, so we generated synthetic stress tests,
where policies and ontologies are significantly larger than what we
expect in real GDPR compliance scenarios.
Our tests show that \PLR is significantly faster than Hermit on small
and medium policies. Moreover, \PLR's performance can be improved by
caching normalized policies (\mbox{\PLR\!c}).  With this solution,
\PLR takes around 500\,$\mu$ seconds per subsumption check, over the
test sets inspired by SPECIAL's pilots (PXS and TR).  By
pre-normalizing business policies (\PLR\!pre), the average cost per
subsumption check can be further reduced to 333\,$\mu$sec (PXS) and
487\,$\mu$sec (TR).

Over large policies (P2), the probability of observing a combinatorial
explosion during interval splitting grows, and the performance of \PLR
exhibits an exponential decrease as \ni (the average number of
intervals per simple policy measured after applying the seven rewrite
rules) increases.  This phenomenon is unavoidable, unless
$\mathsf{P}=\NP$, because \PL subsumption checking is co\NP-hard if
\ni is unrestricted. However, by normalizing also consent policies,
combinatorial effects are mitigated, and \PLR\!c\,2n turns out to be faster than
Hermit for $\ni < 8$.

In perspective, the expressiveness needed to encode the vocabularies
of data categories, purposes, recipients, etc.\ is going to exceed the
capabilities of \PL.  For this reason, we have shown how to integrate
the compliance checking method based on \PLR with reasoners for
logics more expressive than \PL.  The integration is based on the
\emph{import by query} approach.  If the ``external'' ontology \O that
defines vocabulary terms is in Horn-\SRIQ, and if the main knowledge
base \K and the given subsumption query share only concept names with
\O, then algorithm \PLRO\ -- an adaptation of \PLR that calls a
reasoner for \O\ -- is sound and complete.  If \O additionally belongs
to a tractable DL, then subsumption checking is tractable in the IBQ
framework, too. The restriction on roles can be partly lifted by
allowing queries to mention \O's roles, provided that if
$R\in\sig{\O}$, then the existential restrictions $\exists R.C$ may
contain only roles in $\sig{\O}$.

We have also illustrated a different implementation strategy, based on
a pre-com\-pilation of \K and \O into a single \PL knowledge base
$\comp(\K,\O)$, whose size is polynomial in the size of $\K\cup\O$ and
in the number of business policies.
Compliance checks are computed in polynomial time, after compilation,
even if \O belongs to an intractable logic. Moreover, pre-compilation
allows to exploit the implementation of \PLR, whose scalability has
been assessed in Section~\ref{sec:experiments}.
This approach works well in SPECIAL's use cases because the number of
business policies is usually small, and \K, \O, and the business
policies are relatively stable and persistent.  Unfortunately, the
above assumptions cannot be made in general, for all potential
applications of \PL.

Such applications include also the representation of licenses, which
are a fundamental element of data markets.  The application context is
in some respect analogous to SPECIAL's: \PL concepts should encode the
usage restrictions that apply to datasets, multimedia content, and so
on.  In this case, however, the policies that can be reasonably
assumed to belong to a limited set are those associated to sellers,
that occur on the right-hand side of subsumptions, while the left-hand
side can hardly be restricted.  This hinders the compilation-based
approach, and may require a direct implementation of \PLRO, that is, the
general IBQ reasoner for \PL.  Such implementation and its
experimental assessment are interesting topics for further research.

\PL can also naturally encode electronic health records (EHRs). In
this case, the top-level properties of \PL queries encode the sections
of EHRs -- according, say, to the HL7 standard -- while some of the
sections' contents can be specified with SNOMED terms.  The IBQ
framework allows to process \PL queries with \PLRO, and reduce the
cost of SNOMED to oracle calls, consisting of linear time visits to
its classification graph. The efficiency of the structural subsumption
reasoner is very promising in this context, that is challenging for
all engines due to the remarkable size of SNOMED.  We plan to try
\PLRO to increase the performance of the secure view construction
reported in \cite{DBLP:conf/dlog/BonattiPS15}.

The simplicity of \PLR makes it possible to embed \PL reasoning in
objects with limited scripting capabilities. For example,
one of SPECIAL's partners has programmed \PL compliance checking as a
smart contract in an Ethereum blockchain.  In this way, the creation
of new entries in the blockchain is subject to compliance with a
specified policy.

SPECIAL's deliverables comprise dashboards for controllers, data
subjects, and data protection officers. We are going to support these
interfaces by developing explanation algorithms for helping users in
understanding policies and their decisions. The idea is leveraging the
simple structure of \PL concepts and axioms to generate high-level,
user-friendly explanations.

On the theoretical side, \PL and its combination with
\ELp and \DLLh constitute new tractable fragments of OWL2.  The
negative result on oracles with nominals
(Theorem~\ref{thm:IBQ-incompleteness}) extends a result of
\cite{DBLP:journals/jair/GrauM12} to logics that (like \PL) enjoy the
finite model property, and to IBQ mechanism where the axioms of the main
knowledge base \K may be shifted to the imported ontology \O.

Further interesting topics for future work include: an analysis of the
effects of dropping the requirement that $(\sig{\K} \cup \sig{q})
\setminus \sig{\O} \subseteq \NC$, and a complexity analysis of the
extensions of \PL obtained by adding CLASSIC's constructs, such as
number restrictions and role-value maps.


\subsection*{Acknowledgments}

This research is funded by the European Union’s Horizon 2020 research and innovation programme under grant agreement N.~\grant.
The GDPR compliance use case -- here sketched with (\ref{gdpr1}), (\ref{gdpr2}), and Example~\ref{ex:GDPR-checking} -- is due to Benedict Whittam Smith (Thomson Reuters).


\bibliography{biblio,oracle}

\begin{thebibliography}{10}

\bibitem{DBLP:journals/jair/ArtaleCKZ09}
A.~Artale, D.~Calvanese, R.~Kontchakov, and M.~Zakharyaschev.
\newblock The {DL-Lite} family and relations.
\newblock {\em J. Artif. Intell. Res.}, 36:1--69, 2009.

\bibitem{DBLP:conf/ijcai/BaaderBL05}
F.~Baader, S.~Brandt, and C.~Lutz.
\newblock Pushing the {EL} envelope.
\newblock In {\em IJCAI-05}, pages 364--369. Professional Book Center, 2005.

\bibitem{DBLP:conf/dlog/2003handbook}
F.~Baader, D.~Calvanese, D.~L. McGuinness, D.~Nardi, and P.~F. Patel-Schneider,
  editors.
\newblock {\em The Description Logic Handbook: Theory, Implementation, and
  Applications}. Cambridge University Press, 2003.

\bibitem{DBLP:conf/datalog/Bonatti10}
P.~A. Bonatti.
\newblock Datalog for security, privacy and trust.
\newblock In O.~de~Moor, G.~Gottlob, T.~Furche, and A.~J. Sellers, editors,
  {\em Datalog Reloaded - First International Workshop, Datalog 2010, Oxford,
  UK, March 16-19, 2010. Revised Selected Papers}, volume 6702 of {\em Lecture
  Notes in Computer Science}, pages 21--36. Springer, 2010.

\bibitem{DBLP:conf/ijcai/Bonatti18}
P.~A. Bonatti.
\newblock Fast compliance checking in an {OWL2} fragment.
\newblock In J.~Lang, editor, {\em Proceedings of the Twenty-Seventh
  International Joint Conference on Artificial Intelligence, {IJCAI} 2018, July
  13-19, 2018, Stockholm, Sweden.}, pages 1746--1752. ijcai.org, 2018.

\bibitem{DBLP:conf/semweb/BonattiBDFKPPW18}
P.~A. Bonatti, B.~Bos, S.~Decker, J.~D. Fern{\'{a}}ndez, S.~Kirrane,
  V.~Peristeras, A.~Polleres, and R.~Wenning.
\newblock Data privacy vocabularies and controls: Semantic web for transparency
  and privacy.
\newblock In K.~K. Waterman, editor, {\em Proceedings of the Workshop on
  Semantic Web for Social Good co-located with 17th International Semantic Web
  Conference, SW4SG@ISWC 2018, Monterey, California, USA, October 9, 2018.,
  Monterey, California, USA, October 9, 2018.}, volume 2182 of {\em {CEUR}
  Workshop Proceedings}. CEUR-WS.org, 2018.

\bibitem{DBLP:journals/tkde/BonattiCOS10}
P.~A. Bonatti, J.~L.~D. Coi, D.~Olmedilla, and L.~Sauro.
\newblock A rule-based trust negotiation system.
\newblock {\em {IEEE} Trans. Knowl. Data Eng.}, 22(11):1507--1520, 2010.

\bibitem{DBLP:journals/tissec/BonattiVS02}
P.~A. Bonatti, S.~D.~C. di~Vimercati, and P.~Samarati.
\newblock An algebra for composing access control policies.
\newblock {\em {ACM} Trans. Inf. Syst. Secur.}, 5(1):1--35, 2002.

\bibitem{DBLP:conf/safecomp/BonattiKPW17}
P.~A. Bonatti, S.~Kirrane, A.~Polleres, and R.~Wenning.
\newblock Transparent personal data processing: The road ahead.
\newblock In S.~Tonetta, E.~Schoitsch, and F.~Bitsch, editors, {\em Computer
  Safety, Reliability, and Security - {SAFECOMP} 2017 Workshops, ASSURE,
  DECSoS, SASSUR, TELERISE, and TIPS, Trento, Italy, September 12, 2017,
  Proceedings}, volume 10489 of {\em Lecture Notes in Computer Science}, pages
  337--349. Springer, 2017.

\bibitem{DBLP:conf/dlog/BonattiPS15}
P.~A. Bonatti, I.~M. Petrova, and L.~Sauro.
\newblock Optimized construction of secure knowledge-base views.
\newblock In D.~Calvanese and B.~Konev, editors, {\em Proceedings of the 28th
  International Workshop on Description Logics, Athens,Greece, June 7-10,
  2015.}, volume 1350 of {\em {CEUR} Workshop Proceedings}. CEUR-WS.org, 2015.

\bibitem{DBLP:journals/jair/BorgidaP94}
A.~Borgida and P.~F. Patel{-}Schneider.
\newblock A semantics and complete algorithm for subsumption in the {CLASSIC}
  description logic.
\newblock {\em J. Artif. Intell. Res.}, 1:277--308, 1994.

\bibitem{DBLP:journals/jar/GlimmHMSW14}
B.~Glimm, I.~Horrocks, B.~Motik, G.~Stoilos, and Z.~Wang.
\newblock Hermit: An {OWL} 2 reasoner.
\newblock {\em J. Autom. Reasoning}, 53(3):245--269, 2014.

\bibitem{DBLP:journals/jair/GrauM12}
B.~C. Grau and B.~Motik.
\newblock Reasoning over ontologies with hidden content: The import-by-query
  approach.
\newblock {\em J. Artif. Intell. Res.}, 45:197--255, 2012.

\bibitem{DBLP:conf/ijcai/GrauMK09}
B.~C. Grau, B.~Motik, and Y.~Kazakov.
\newblock Import-by-query: Ontology reasoning under access limitations.
\newblock In C.~Boutilier, editor, {\em {IJCAI} 2009, Proceedings of the 21st
  International Joint Conference on Artificial Intelligence, Pasadena,
  California, USA, July 11-17, 2009}, pages 727--732, 2009.

\bibitem{DBLP:conf/ecai/HaaseL08}
C.~Haase and C.~Lutz.
\newblock Complexity of subsumption in the {\EL} family of description logics:
  Acyclic and cyclic tboxes.
\newblock In M.~Ghallab, C.~D. Spyropoulos, N.~Fakotakis, and N.~M. Avouris,
  editors, {\em {ECAI} 2008 - 18th European Conference on Artificial
  Intelligence, Patras, Greece, July 21-25, 2008, Proceedings}, volume 178 of
  {\em Frontiers in Artificial Intelligence and Applications}, pages 25--29.
  {IOS} Press, 2008.

\bibitem{DBLP:conf/kr/HorrocksKS06}
I.~Horrocks, O.~Kutz, and U.~Sattler.
\newblock The even more irresistible {SROIQ}.
\newblock In P.~Doherty, J.~Mylopoulos, and C.~A. Welty, editors, {\em
  Proceedings, Tenth International Conference on Principles of Knowledge
  Representation and Reasoning, Lake District of the United Kingdom, June 2-5,
  2006}, pages 57--67. {AAAI} Press, 2006.

\bibitem{DBLP:journals/tods/JajodiaSSS01}
S.~Jajodia, P.~Samarati, M.~L. Sapino, and V.~S. Subrahmanian.
\newblock Flexible support for multiple access control policies.
\newblock {\em {ACM} Trans. Database Syst.}, 26(2):214--260, 2001.

\bibitem{rei}
L.~Kagal, T.~W. Finin, and A.~Joshi.
\newblock A policy language for a pervasive computing environment.
\newblock In {\em 4th IEEE International Workshop on Policies for Distributed
  Systems and Networks (POLICY)}, pages 63--, Lake Como, Italy, June 2003. IEEE
  Computer Society.

\bibitem{DBLP:journals/jar/KazakovKS14}
Y.~Kazakov, M.~Kr{\"{o}}tzsch, and F.~Simancik.
\newblock The incredible {ELK} - from polynomial procedures to efficient
  reasoning with {EL} ontologies.
\newblock {\em J. Autom. Reasoning}, 53(1):1--61, 2014.

\bibitem{DBLP:conf/esws/KirraneFDMPBWDR18}
S.~Kirrane, J.~D. Fern{\'{a}}ndez, W.~Dullaert, U.~Milosevic, A.~Polleres,
  P.~A. Bonatti, R.~Wenning, O.~Drozd, and P.~Raschke.
\newblock A scalable consent, transparency and compliance architecture.
\newblock In A.~Gangemi, A.~L. Gentile, A.~G. Nuzzolese, S.~Rudolph,
  M.~Maleshkova, H.~Paulheim, J.~Z. Pan, and M.~Alam, editors, {\em The
  Semantic Web: {ESWC} 2018 Satellite Events - {ESWC} 2018 Satellite Events,
  Heraklion, Crete, Greece, June 3-7, 2018, Revised Selected Papers}, volume
  11155 of {\em Lecture Notes in Computer Science}, pages 131--136. Springer,
  2018.

\bibitem{DBLP:conf/dlog/MatentzogluBP13}
N.~Matentzoglu, S.~Bail, and B.~Parsia.
\newblock A corpus of {OWL} {DL} ontologies.
\newblock In T.~Eiter, B.~Glimm, Y.~Kazakov, and M.~Kr{\"{o}}tzsch, editors,
  {\em Informal Proceedings of the 26th International Workshop on Description
  Logics, Ulm, Germany, July 23 - 26, 2013}, volume 1014 of {\em {CEUR}
  Workshop Proceedings}, pages 829--841. CEUR-WS.org, 2013.

\bibitem{DBLP:journals/jair/MotikSH09}
B.~Motik, R.~Shearer, and I.~Horrocks.
\newblock Hypertableau reasoning for description logics.
\newblock {\em J. Artif. Intell. Res.}, 36:165--228, 2009.

\bibitem{DBLP:conf/kr/OrtizRS10}
M.~Ortiz, S.~Rudolph, and M.~Simkus.
\newblock Worst-case optimal reasoning for the horn-dl fragments of {OWL} 1 and
  2.
\newblock In F.~Lin, U.~Sattler, and M.~Truszczynski, editors, {\em Principles
  of Knowledge Representation and Reasoning: Proceedings of the Twelfth
  International Conference, {KR} 2010, Toronto, Ontario, Canada, May 9-13,
  2010}. {AAAI} Press, 2010.

\bibitem{DBLP:conf/ijcai/OrtizRS11}
M.~Ortiz, S.~Rudolph, and M.~Simkus.
\newblock Query answering in the horn fragments of the description logics
  {SHOIQ} and {SROIQ}.
\newblock In T.~Walsh, editor, {\em {IJCAI} 2011, Proceedings of the 22nd
  International Joint Conference on Artificial Intelligence, Barcelona,
  Catalonia, Spain, July 16-22, 2011}, pages 1039--1044. {IJCAI/AAAI}, 2011.

\bibitem{DBLP:journals/ws/SteigmillerLG14}
A.~Steigmiller, T.~Liebig, and B.~Glimm.
\newblock Konclude: System description.
\newblock {\em J. Web Semant.}, 27-28:78--85, 2014.

\bibitem{DBLP:conf/policy/UszokBJSHBBJKL03}
A.~Uszok, J.~M. Bradshaw, R.~Jeffers, N.~Suri, P.~J. Hayes, M.~R. Breedy,
  L.~Bunch, M.~Johnson, S.~Kulkarni, and J.~Lott.
\newblock {KAoS} policy and domain services: Towards a description-logic
  approach to policy representation, deconfliction, and enforcement.
\newblock In {\em 4th IEEE International Workshop on Policies for Distributed
  Systems and Networks (POLICY)}, pages 93--96, Lake Como, Italy, June 2003.
  IEEE Computer Society.

\bibitem{DBLP:journals/csec/WooL93}
T.~Y.~C. Woo and S.~S. Lam.
\newblock Authorizations in distributed systems: {A} new approach.
\newblock {\em Journal of Computer Security}, 2(2-3):107--136, 1993.

\end{thebibliography}
\bibliographystyle{abbrv}

\end{document}